\documentclass[12pt, nosubfloats]{l4dc2022}

\title[Physical Derivatives]{Physical Derivatives: Computing policy gradients by physical forward-propagation}
\usepackage{times}

\usepackage{subcaption}
\usepackage{hyperref}
\usepackage{url}
\usepackage{cleveref}
\usepackage{caption}
\usepackage{booktabs}  
\usepackage{subcaption}
\usepackage{wrapfig}
\usepackage{dsfont}
\input{notation}

\author{%
 \Name{Arash Mehrjou} \Email{arash.x.mehrjou@gsk.com}\\
 \addr GlaxoSmithKline, Artifical Intelligence \& Machine Learning\\
 \addr Max Planck Institue for Intelligent Systems, T\"ubingen, Germany\\
 \addr ETH Z\"urich, Z\"urich, Switzerland
 \AND
 \Name{Ashkan Soleymani} \Email{ashkanso@mit.edu}\\
 \addr Massachusetts Institute of Technology, Cambridge, US \\
 \addr Max Planck Institue for Intelligent Systems, T\"ubingen, Germany
  \AND
 \Name{Stefan Bauer} \Email{baue@kth.se}\\
 \addr KTH, Stockholm, Sweden\\
 \addr GlaxoSmithKline, Artifical Intelligence \& Machine Learning
 \AND
  \Name{Bernhard Sch\"olkopf} \Email{bs@tuebingen.mpg.de}\\
  \addr Max Planck Institue for Intelligent Systems, T\"ubingen, Germany\\
  \addr ETH Z\"urich, Z\"urich, Switzerland
}

\begin{document}

\maketitle

\begin{abstract}%
     Model-free and model-based reinforcement learning are two ends of a spectrum. Learning a good policy without a dynamic model can be prohibitively expensive. Learning the dynamic model of a system can reduce the cost of learning the policy, but it can also introduce bias if it is not accurate. We propose a middle ground where instead of the transition model, the sensitivity of the trajectories with respect to the perturbation of the parameters is learned. This allows us to predict the local behavior of the physical system around a set of nominal policies without knowing the actual model. We assay our method on a custom-built physical robot in extensive experiments and show the feasibility of the approach in practice. We investigate potential challenges when applying our method to physical systems and propose solutions to each of them.%
\end{abstract}

\begin{keywords}%
  Task-agnostic Reinforcement Learning, Reinforcement Learning, Control Theory%
\end{keywords}

\begin{figure}[h!]
    \centering
    \begin{subfigure}[t]{1in}
        \centering
        \includegraphics[width=0.97in]{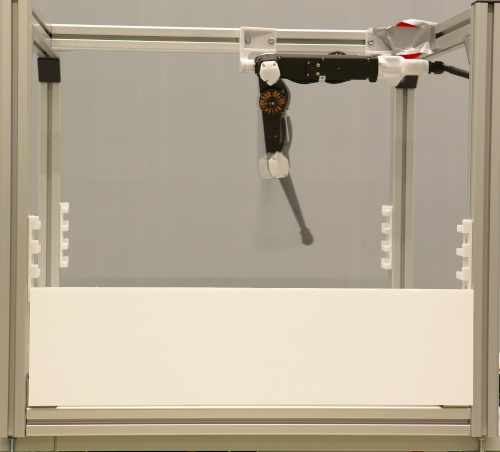}
        \caption{}\label{fig:1a}        
    \end{subfigure}
    \begin{subfigure}[t]{1in}
        \centering
        \includegraphics[width=1in]{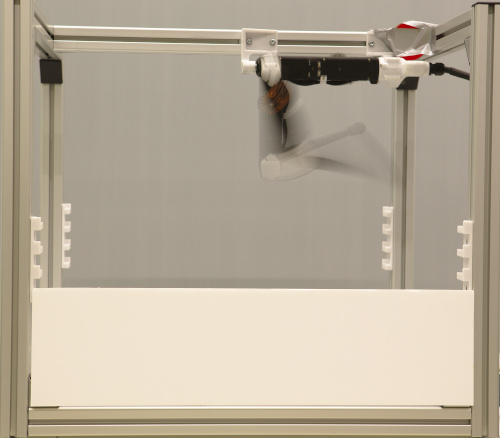}
        \caption{}\label{fig:1b}    
   \end{subfigure}
   \begin{subfigure}[t]{1in}
        \centering
        \includegraphics[width=1in]{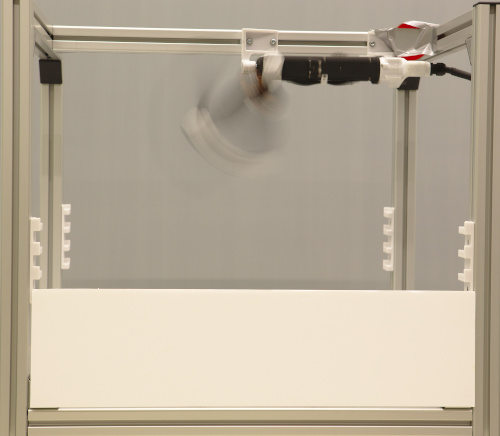}
        \caption{}\label{fig:1c}    
   \end{subfigure}
   \begin{subfigure}[t]{1in}
        \centering
        \includegraphics[width=1in]{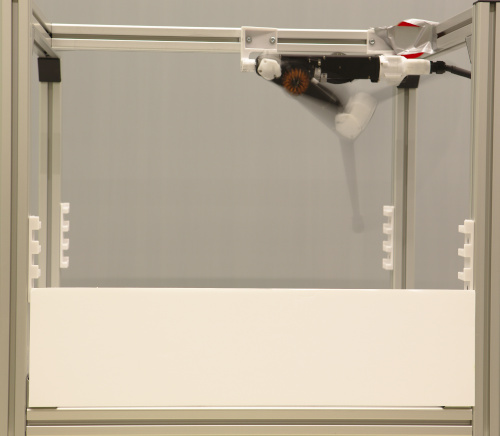}
        \caption{}\label{fig:1d}    
    \end{subfigure}
    \caption{Physical finger platform in action with different policies.}\label{fig:platform_in_motion}
\end{figure}

\section{Introduction}

Traditional reinforcement learning crucially relies on \emph{reward}~\citep{sutton2018reinforcement}. However, reward binds the agent to a certain task for which the reward represents success. Aligned with the recent surge of interest in unsupervised methods in reinforcement learning~\citep{baranes2013active, bellemare2016unifying, gregor2016variational, houthooft2016variational, gupta2018unsupervised, hausman2018learning, pong2019skew, laskin2020curl, laskin2021urlb, he2021wasserstein} and previously proposed ideas~\citep{schmidhuber1991curious, schmidhuber2010formal}, we argue that there exist properties of a dynamical system which are not tied to any particular task, yet highly useful, leveraging them can help solve other tasks more efficiently. This work focuses on the sensitivity of the produced trajectories of the system with respect to the policy so-called~\emph{Physical Derivatives}. The term~\emph{physical} comes from the fact that it uses the physics of the system rather than any idealized model. We learn a map from the directions in which policy parameters change to the directions in which every state of the trajectory changes. In general, our algorithm learns the Jacobian matrix of the system at every time step through the trajectory. The training phase consists of physically calculating directional derivatives by the finite difference after applying perturbed versions of a nominal policy (a.k.a. controller). After training, the learned directional derivatives are used to guide the controller to achieve the desired behaviour. Due to the difficulty of computing the Jacobian matrix by the finite difference in higher dimensions, we use random controllers joint with probabilistic learning methods to obtain a robust estimate of the Jacobian matrix at each instant of time along a trajectory. The generalization to unseen perturbations is possible because the trajectories produced by physical systems live on an intrinsically low-dimensional manifold and change slowly with respect to perturbations in the system~\citep{koopman1931hamiltonian}. This assumption holds as long as the system is not chaotic or close to a bifurcation condition~\citep{khalil2002nonlinear}.

\subsection{Related works}
\label{sec:related_works}

A truly intelligent agent must develop some sort of~\emph{general competence} that allows it to combine primitive skills to master a range of tasks, not only a single task associated with a specified reward function. The major part of such competence comes from unsupervised experiences. Animals use a similar competence to quickly adapt to new environments~\citep{weng2001autonomous}. and function efficiently soon after birth before being exposed to massive supervised experience~\citep{zador2019critique}. Due to its generality, such basic skills can be inherited over generations rather than being learned from scratch~\citep{weightagnostic2019}. Despite traditional RL that the learning is driven by an~\emph{extrinsic} reward signal,~\emph{intrinsically} motivated RL concerns task-agnostic learning~\citep{sontakke2021causal,sontakke2021galilai}. Similar to animals' babies~\citep{touwen1992development}, the agent may undergo a developmental period in which it acquires reusable modular skills~\citep{kaplan2003motivational,weng2001autonomous, tian2021unsupervised}, such as curiosity and confidence~\citep{schmidhuber1991curious,kompella2017continual, riedmiller2018learning, burda2018large, mirza2020physically, groth2021curiosity, huang2022unified}. Another aspect of such general competence is the ability of the agent to remain~\emph{safe} during its learning and deployment period~\citep{garcia2015comprehensive}. In physical systems, especially continuous control,~\emph{stability} is a major aspect of safety that implies states of the system converge to some invariant sets or remain within a certain bound~\citep{lyapunov1992general} which could be enforced as a prior in system identification~\citep{mehrjou2020learning} or realized as a target in control tasks~\citep{mehrjou2020neural}. Control theory often assumes the model of the system known to guarantee stability~\citep{khalil2002nonlinear}. In the absence of the model, model-based RL learns the model along with the policy. Hence, learning the transition model to predict the states in the future can be another intrinsic reward (See~\Cref{sec:detailed_literature_review} for a detailed literature review).

\subsection{Preliminaries}
We consider a closed-loop dynamical system represented by the state vector $\xb\in\Xcal\subseteq\RR^d$ and the policy function $\ub=\pi(\xb;\thetab)$ that emits the $q-$dimensional control signal $\ub\in\Ucal\subseteq\RR^q$. Let $r:\Xcal \times \Ucal\to\RR$ denote the reward function and $R:\Pi (\Theta)\to\RR$ be the cumulative reward (\emph{return}). For parametric policies, the space of feasible parameters $\Theta$ has a one-to-one correspondence to the policy space $\Pi$. The agent who takes on the policy $\pi$ from state $\xb_0$ produces the trajectory $\Tcal\in\TT$ where $\TT$ is the space of possible trajectories. The expected return becomes a function of the policy as $J(\pi_\thetab) = \EE_\Tcal\{R(\Tcal)\}$ where the expectation is taken with respect to the probability distribution $P(\Tcal | \pi_\thetab)$. Traditionally in reinforcement learning, the goal is to perturb the policy to improve the expected return: 
\begin{equation}
  \label{eq:policy_parameter_update_by_return}
  \thetab_{t+1} = \thetab_t + \alpha \left .\frac{\partial J(\pi_\thetab)}{\partial \thetab}\right |_{\thetab=\thetab_t}.
\end{equation}
The gradient $\partial J(\pi_\thetab) / \partial \thetab$ can be written as an integral
\begin{equation}
  \label{eq:expected_return_derivative}
  \frac{\partial J(\pi_\thetab)}{\partial \thetab} = \int_\TT \frac{\partial p(\Tcal|\pi_\thetab)}{\partial \thetab}R(\Tcal)\diff\Tcal
\end{equation}
which is hard compute in practice. In model-free RL, the policy is updated so that the mode of $p(\Tcal|\pi_\thetab)$ aligns with the modes of $R(\Tcal)$. In model-based RL, the dynamical model that generates the trajectory $\Tcal$ is learned and then used to estimate $\frac{\partial J(\pi_\thetab)}{\partial \thetab}$. In this work, we take a middle-ground approach to estimate an in-between quantity and show how the estimated map can guide the policy towards desired behaviour.

\subsection{What is Physical Derivative}
In this paper, we investigate the feasibility of learning a less explored unsupervised quantity, the so-called~\emph{Physical Derivative} which is computed directly from the physical system. In abstract terms, we perturb the policy and learn the effect of its perturbation on the resulting trajectory. The difference from traditional RL, whose algorithms are based on~\Cref{eq:policy_parameter_update_by_return}, is the absence of a specified reward function. Instead, we generate samples from $\partial p(\Tcal|\pi_\thetab)/\partial \thetab$ of~\Cref{eq:expected_return_derivative} that makes it possible to compute $\partial J(\pi_\thetab)/\partial \thetab$ for an arbitrary return function $R$. If the exact model of the system is known, control theory has a full set of tools to intervene on the system with stability and performance guarantees. When the system is unknown, one could identify the system as a preliminary step followed by a normal control synthesis process from control theory~\citep{ljung2001system}. Otherwise, the model and the policy can be learned together in a model-based RL~\citep{sutton1996model} or in some cases adaptive control~\citep{sastry2011adaptive}. We argue that learning physical derivatives is a middle ground. It is not model-based in the sense that it does not assume knowing the exact model of the system. For example, assume we are interested in going from the current trajectory $\Tcal(\thetab)$ to the target trajectory $\Tcal^*$. The distance  between these trajectories is reduced by perturbing the policy parameters in the direction $\minus\partial \lVert \Tcal(\thetab) - \Tcal^*\rVert / \partial \thetab$. This direction is already available since we have direct access to $\partial \Tcal(\thetab)/\partial \thetab$ as a physical derivative.

{\it \textbf{Our contributions}--- }In summary, the key contributions of the current paper are as follows:
\begin{itemize}
    \item A method to generate training pairs to learn the map from the policy perturbations to the resulting changes in the trajectories.
    \item Learning the above map as a probabilistic function and showing that it generalizes to unseen perturbations in the policy.
    \item Use the inverse of the above map to perturb the policy in the desired direction to achieve certain goals without conventional RL methods.
    \item Use a physical custom-built robotic platform to test the method and propose solutions to deal with the inherent issues of the physical system to ensure the practicality of the method (see~\Cref{fig:platform_in_motion} for images of the platform and ~\Cref{sec:physical_platform} for technical details).
    \item The supplementary materials for the paper, including full-version technical report (referred to as the Appendix in the paper), code, and the videos of the robot in action, can be found in~\url{https://sites.google.com/view/physicalderivatives/}.
\end{itemize}

\section{Applications of physical derivatives}
\label{app:applications}

Supposing we know how the states of a trajectory change as a result of a change in the policy parameters, the policy can be easily updated to push the trajectory towards a desired one. For example, assume we are interested in going from the current trajectory $\Tcal(\thetab)$ to the target trajectory $\Tcal^*$. The distance  between these trajectories can get minimized by perturbing the policy parameters in the direction $\minus\partial \lVert \Tcal(\thetab) - \Tcal^*\rVert / \partial \thetab$. This direction is already available since we have estimated $\partial \Tcal(\thetab)/\partial \thetab$ as a physical derivative. As an exemplary case, we show this application of our method in practice in~\Cref{sec:experiments}. Other applications of physical derivatives are in~\emph{robust control} and~\emph{safety}. In both cases, the physical derivative allows us to predict the behaviour of the system if the policy changes in a neighbourhood around a nominal policy. Then, it is possible to make sure that some performance or safety criteria will not be violated for the local perturbation in the policy. As a concrete example, for an autonomous driving system, there can be a calibration phase during which physical derivatives of the car are estimated by perturbing the controller parameters around different nominal policies which are likely to occur on real roads. The calibration must be done in a safe condition and before deploying the system. When deployed, the estimated physical derivatives can be used to predict the effect of a change of the policy on the behaviour of the system and neutralize the change if it would move the car towards unsafe regions of its state space. The command that changes the policy can be issued by a high-level controller (e.g., guidance system), and the safety is confirmed by a low-level mechanism through physical derivatives. This work focuses on the concept and the introduction of physical derivatives and direct applications would go significantly beyond the scope of this work. In the~\Cref{app:applications}, we provide a more detailed description of the use of physical derivatives in different applications (\emph{robust control}, \emph{safety}, and \emph{adversarial system identification}).

\section{Estimating Physical Derivatives}
\label{sec:method}
In this section, we describe our pipeline to estimate the physical derivatives and our proposed solutions to the inevitable challenges that are likely to occur while working with a real physical robot. We are interested in $\partial \Tcal / \partial \thetab$ which denotes how a small change in the parameters $\thetab$ of the controller results in a different trajectory produced by the system. We normally consider a finite period of time $[0, T]$ and the trajectory is an ordered list of states $\Tcal=[\xb_0, \xb_1, \ldots, \xb_T]$ where the subscript shows the time step. Therefore, having $\partial \Tcal / \partial \thetab$ is equivalent to having $\partial \xb_t / \partial \thetab$ for every $t\in\{1, \ldots, T\}$. Notice that the initial state $\xb_0$ is chosen by us. Hence we can see it either as a constant or as a changeable parameter in $\thetab$. We kept it fixed in our experiments.

Assume $\xb_t\in \RR^d$ and $\thetab\in \RR^m$. Hence, $\nabla_{\thetab}\xb_t = \partial \xb_t / \partial \thetab \in \RR^{d\times m}$ 
where the $t^\mathrm{th}$ row of this matrix is $\nabla_{\thetab}x_{it} = (\partial x_{it} / \partial \thetab)\tran\in \RR^m$ showing how the $i^\mathrm{th}$ dimension of the state vector changes in response to a perturbation in $\thetab$. The directional derivative of $x_{it}$ in the direction $\delta\thetab$ is defined as
\begin{equation}\label{eq:directional_derivative}
    \nabla_{\thetab}^{\delta\thetab} x_{it} = \langle\nabla_{\thetab} x_{it}, \frac{\delta \thetab}{|\delta \thetab|}\rangle.
\end{equation}
If~\eqref{eq:directional_derivative} is available for $m$ linearly independent and orthonormal directions, $\{\delta\thetab^{(1)}, \delta\thetab^{(2)}, \ldots, \delta\thetab^{(m)}\}$, the directional derivative along an arbitrary $\delta \thetab$ can be approximated by
\begin{equation}\label{eq:decomposed_directional_derivative}
    \nabla_{\thetab}^{\delta\thetab} x_{it} = \sum_{j=1}^m c_j \langle \nabla_{\thetab} x_{it}, \delta \thetab^{(j)}\rangle
\end{equation}
where $c_j = \langle\delta\thetab, \delta\thetab^{(j)}\rangle$ is the coordinates of the desired direction in the coordinate system formed by the orthonormal bases. 

  \begin{wrapfigure}{l}{6cm}
     \begin{tabular}{@{}cc@{}}
       \includegraphics[width=0.2\textwidth]{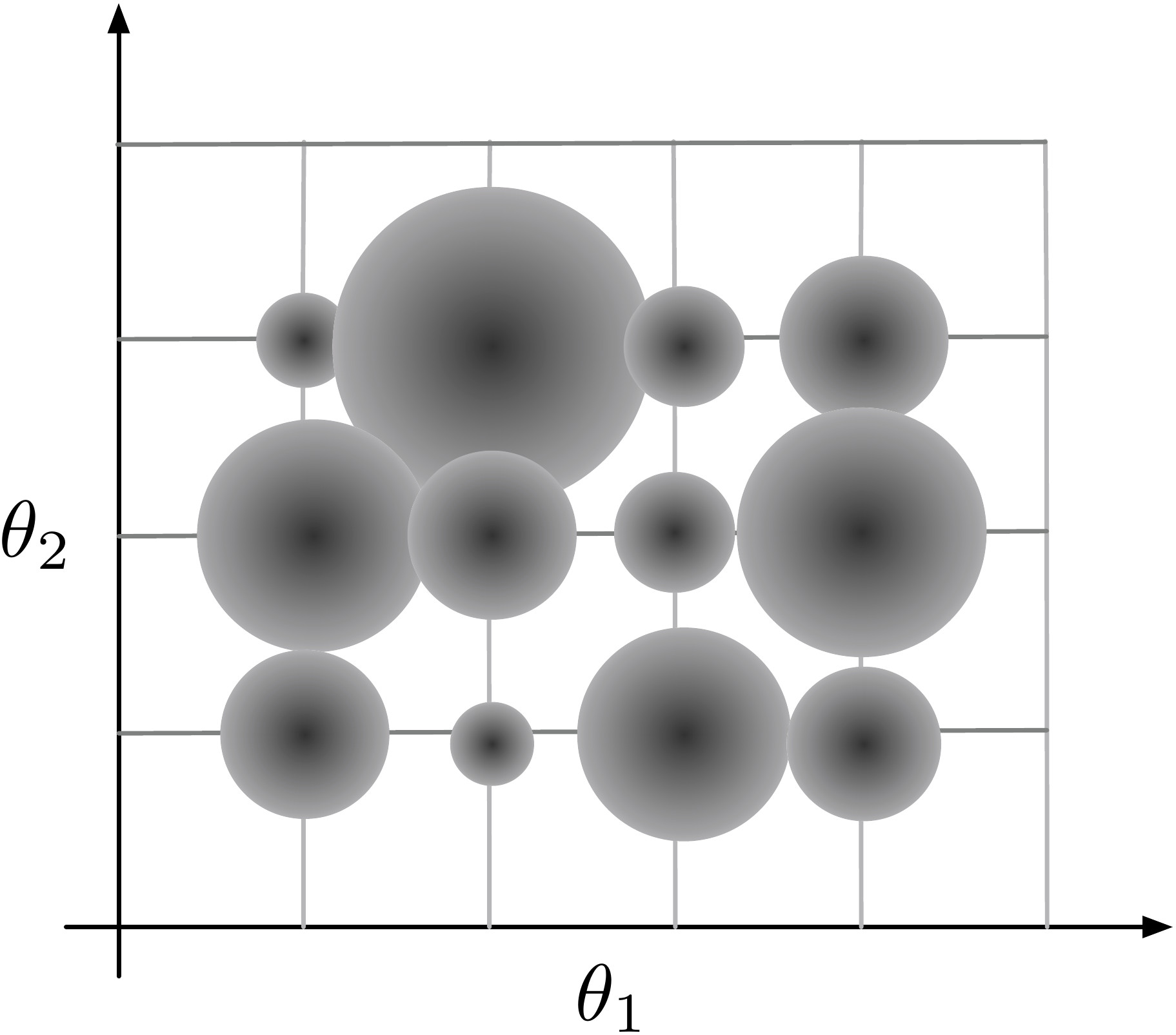}&
  \includegraphics[width=0.2\textwidth]{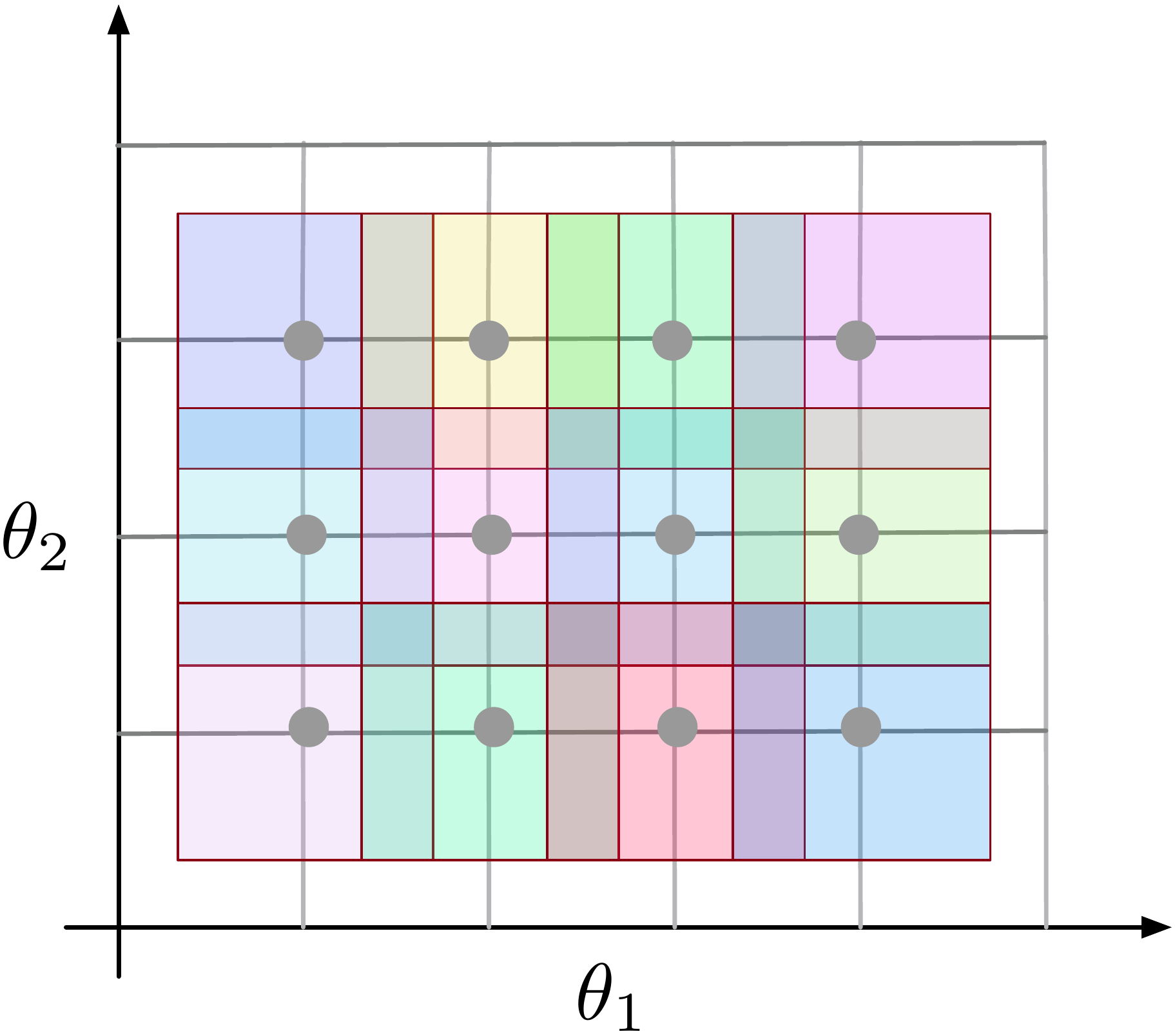}\\
     \end{tabular}
       \caption{Gaussian (left) and uniform (right) perturbation examples.}
       \label{fig:shaking}
 \end{wrapfigure}

\if(0)
\begin{figure}[t!]
    \centering
    \begin{subfigure}[t]{0.2\textwidth}
        \centering
        \includegraphics[width=\textwidth]{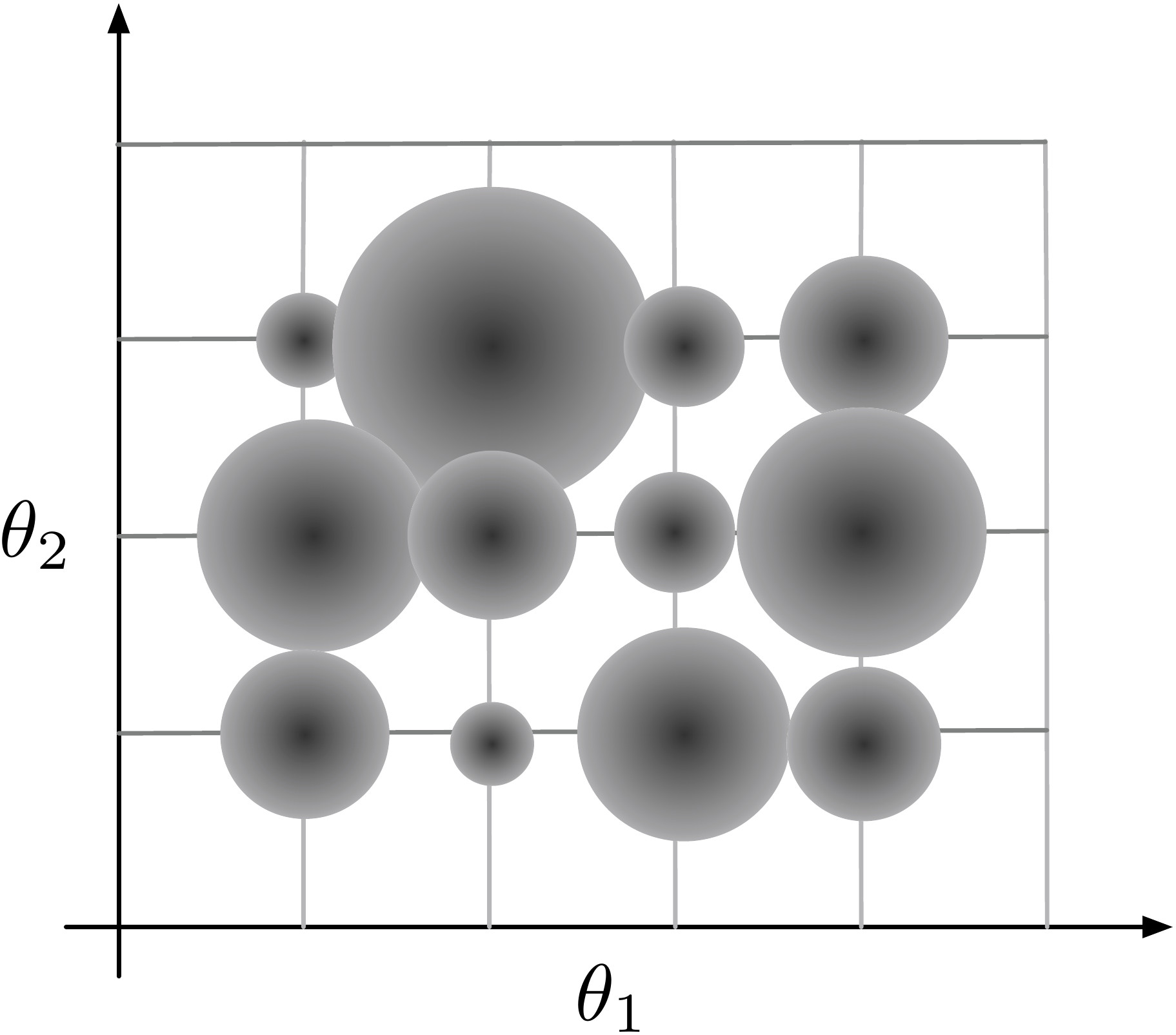}
        \caption{Gaussian perturbation\ar{emphasize the size}}\label{fig:normal_shaking}        
    \end{subfigure}
   \begin{subfigure}[t]{0.2\textwidth}
        \centering
        \includegraphics[width=\textwidth]{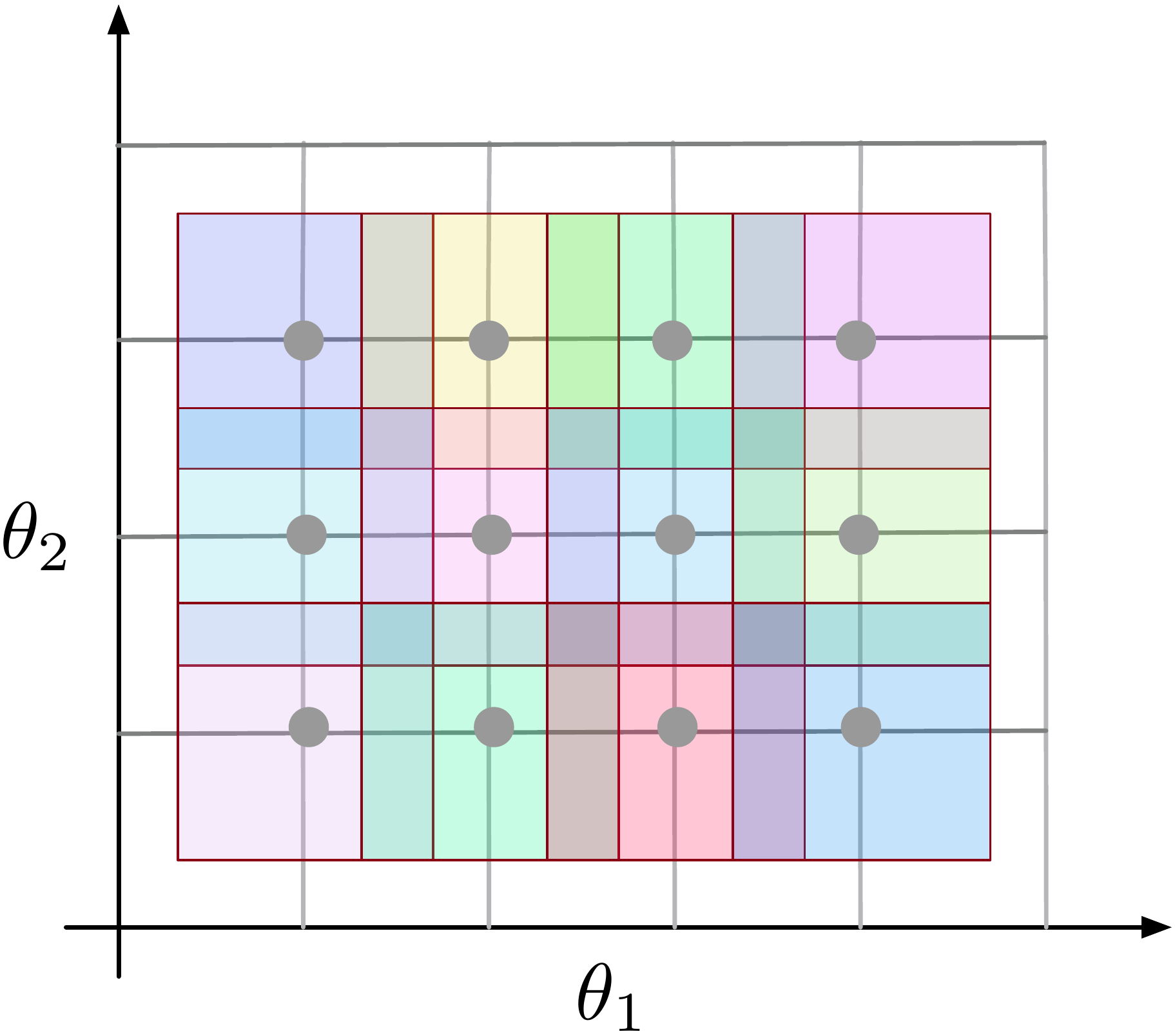}\caption{Uniform perturbation}
        \label{fig:uniform_shaking}    
    \end{subfigure}
    \caption{ perturbation examples}\label{fig:shakings}
\end{figure}
\fi

In practice, $m$ directions $\delta \thetab^{(j)}$ can be randomly chosen or can be along some pre-defined axes of the coordinate system. To compute $\langle\nabla_{\thetab}x_{it}, \delta \thetab^{(j)}\rangle$, the nominal policy parameters $\thetab$ are perturbed by $\delta \thetab^{(j)}$ as $\thetab^{(j)}\leftarrow\thetab+\delta \thetab^{(j)}$ and the derivative is computed as
\begin{equation}
  \label{eq:finite_difference_directional_derivative}
  \langle\nabla_{\thetab}x_{it}, \delta \thetab^{(j)}\rangle = \lim_{h\to 0}\frac{x_{it}(\thetab+h \delta\thetab^{(j)}) - x_{it}(\thetab)}{h}.
\end{equation}
This quantity is often approximated by finite difference where $h$ takes a small nonzero value. By perturbing the parameters $\thetab$ along $m$ orthonormal directions $\delta\thetab^{(j)}$ and computing the approximate directional derivative by~\eqref{eq:finite_difference_directional_derivative}, $\nabla_{\thetab}^{\delta\thetab} x_{it}$ can be computed along every arbitrary direction $\delta \thetab$, meaning that we can compute $\nabla_\thetab x_{it}$ by evaluating it along any direction, which is the aim of this paper.

In the matrix form for $\xb\in\RR^d$, we can compute $\nabla^{\delta \thetab^{(j)}}_\thetab \xb=[\nabla^{\delta \thetab^{(j)}}_\thetab x_1, \nabla^{\delta \thetab^{(j)}}_\thetab x_1, \ldots, \nabla^{\delta \thetab^{(j)}}_\thetab x_d]\tran$
in a single run by computing~\eqref{eq:finite_difference_directional_derivative} for all $d$ dimensions of the states. Let's define
\begin{equation}
    \label{eq:gradients_in_all_directions}
    \Delta_\thetab \xb \triangleq [\nabla^{\delta \thetab^{(1)}}_\thetab \xb, \nabla^{\delta \thetab^{(2)}}_\thetab \xb, \ldots, \nabla^{\delta \thetab^{(m)}}_\thetab \xb ]
\end{equation}
where $\Delta_\thetab \xb\in \RR^{d \times m}$ and let $\Lambda = [\delta \thetab^{(1)}, \delta \thetab^{(2)}, \ldots, \delta \thetab^{(m)}]$. Therefore, if $\Delta_\thetab ^{\delta \thetab} \xb$ shows the directional derivative of $\xb$ along $\delta \thetab$, we can write it as:
\begin{equation}
    \label{eq:directional_derivative_matrixform}
    \nabla_\thetab ^{\delta \thetab} \xb = \Delta_\thetab \xb (\Lambda\tran \delta \thetab)
\end{equation}
which is a vectoral representation of~\Cref{eq:directional_derivative}. Even though the linear formula of~\Cref{eq:directional_derivative_matrixform} requires only $m$ directional derivatives, it has two major downsides. First, it does not give a clear way to incorporate more than $m$ training directional physical derivatives. Second, the linear approximation remains valid only for very small $\delta\thetab$. We propose to use Gaussian Process (GP) as a nonlinear probabilistic function approximator~\citep{rasmussen2003gaussian} to capture the maps $\hat{g}_t$ defined as
\begin{align}
    &\hat{g}_t:\Theta\to \Xcal \\
    &\hat{g}_t(\delta \thetab) = \delta \xb
\end{align}
where subscript $t$ shows the function that maps $\delta\thetab$ to the change of the states $\delta\xb_t$ at time step $t$. We considered distinct functions for every time step. Taking into account the commonality among the function approximators corresponding to different time steps is deferred to future research. 

The required training data to learn these maps is produced by perturbing the parameters of the controller and recording the produced trajectories. We propose two perturbation methods called \emph{Gaussian} and~\emph{Uniform} as illustrated in~\Cref{fig:shaking}. The effect of each of these sampling strategies is presented in~\Cref{sec:experiments}. 

A major blockage against estimating physical derivatives is the inherent noise of the physical systems that may be mixed with the intentional perturbations which are applied to collect data to estimate physical derivatives. In the following section, we describe this issue concretely and elaborate on our solution to deal with it.

\section{Tackling inherent noise}
\label{sec:real_challenges}
Inherent noise refers to a change in the system trajectory that is not caused by the intentional intervention on the controller. In our finger platform, we observed two different major sources of noise that are likely to occur in other physical systems too. We call them \emph{temporal} and \emph{spatial} noise for the reasons that come in the following.

\paragraph{Temporal noise.} The temporal noise represented by $\nbb$ affects trajectories by shifting them in time
\begin{equation}
  \xb_t \leftarrow \xb_{t + \nbb}\;\mathrm{for}\; t=0,1,\ldots,T.
\end{equation}
Notice that the absence of subscript $t$ in $\nbb$ shows that this noise is not time-dependent, i.e., the time shift does not change along the trajectory as time proceeds.

\paragraph{Spatial noise.} The trajectories affected by spatial noise cannot be aligned with each other by shifting forward or backward in time. We can model this noise as a state-dependent influence on the state of the system at every time step.
\begin{equation}
  \xb_t \leftarrow \xb_t + \nbb_{\xb_t}
\end{equation}

The following definition makes the distinction more concrete.
\begin{definition}
  Consider two trajectories $\Tcal^{(1)}(t)$ and $\Tcal^{(2)}(t)$ as two temporal signals. Assume $S_{t_\circ}$ is the shift-in-time operator defined as
\begin{equation}
  S_{t_\circ} \Tcal(t)= \Tcal(t+t_\circ)
\end{equation}
for an arbitrary function of time $\Tcal(t)$. We say $\Tcal^{(2)}(t)$ is temporally noisy version of $\Tcal^{(1)}(t)$ if
\begin{equation}
  \exists t_\circ\in\RR\;s.t.\;  \lVert \Tcal^{(2)} - S_{t_\circ}\Tcal^{(1)} \rVert_1\leq \epsilon
\end{equation}
where $\epsilon$ is a hyper-parameter threshold that reflects our prior confidence about the accuracy of the motors, joints, physical and electrical elements (in general construction process) of the robot. On the other hand, $\Tcal^{(2)}$ is called a spatially noisy version of $\Tcal^{(1)}$ if
\begin{equation}
    \nexists t_\circ\in\RR\;s.t.\;  \lVert \Tcal^{(2)} - S_{t_\circ}\Tcal^{(1)} \rVert_1\leq \epsilon
\end{equation}
\end{definition}

The temporal noise is coped with correlation-based delay estimation (see~\Cref{sec:temporal_noise}) and the spatial noise is dealt with voxelization (see~\Cref{sec:spatial_noise}). 

After dealing with the challenge of inherent noise, we pursue the main goal of this paper that is estimating $\partial \Tcal / \partial \thetab$ directly from the physical system. In the following, we investigate the use of the different types of controllers to emphasize the extent of applicability of the proposed method.

\section{Experiments}
\label{sec:experiments}
In this section, we show how physical derivatives can be estimated in practice through several experiments. Notice that our work is different from computing gradients around the working point of a system by finite-difference. We aim to collect samples from such gradients by perturbing a grid of nominal values of the policy parameters and then generalize to unseen perturbations by Gaussian process as a probabilistic regression method. 
The experiments are designed to show each challenge separately and the efficacy of our proposed solution to it. Due to space constraints, details of the physical platform can be found in section \ref{sec:physical_platform} in the Appendix. See\footnote{\url{https://sites.google.com/view/physicalderivatives/}} for videos of the robot while collecting data for different experiments and more backup materials.
\subsection{Linear open-loop controller}
\label{sec:linear_openloop_controller}
As a simple yet general policy, in this section, we consider an open-loop controller , which is a linear function of time. The policy $\ub_t = [u_{1t}, u_{2t}, u_{3t}]$ constitutes the applied torques to the three motors $\{m_1, m_2, m_3\}$ of the system and is assigned as
\begin{equation}
    u_{it} = w_i t + b_i\;\;\;\mathrm{for}\;\;\; i=1,2,3
\end{equation}

Notice that the torque consists of two terms. The first term $w_i t$ grows with time and the second term remains constant. The controller has $6$ parameters in total denoted by $\thetab$.
The task is to predict $\nabla_\thetab \xb_t$ for every $t$ along the trajectory. In the training phase, the training data is obtained via perturbation as described in~\Cref{sec:method}. 

\Cref{fig:quivers} shows examples of nominal trajectories + trajectories produced by the perturbed controller and the computed derivatives. The arrows are plotted as they originate from the perturbed trajectories only for easier distinction. Each arrow corresponds to the change of the states at a certain time step on the source trajectory as a result of perturbing the policy. Each figure corresponds to a pair of nominal values of $\{w, b\}$ for the linear open-loop controller. See~\Cref{fig:quivers} for examples.

\begin{figure}[t!]
    \centering
    \includegraphics[width=0.95\textwidth]{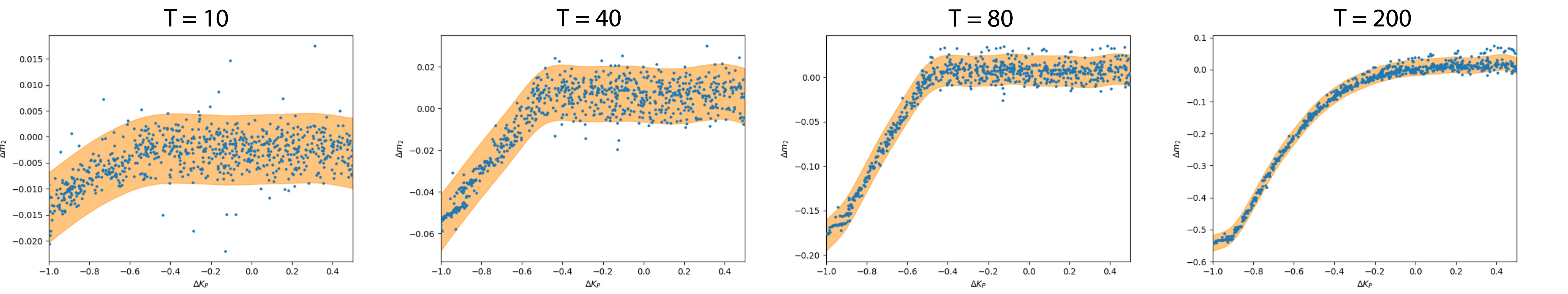}
    \caption{The time evolution of the GP approximated $\hat{g}_t$ for a PD feedback controller at some exemplary time instances}
    \label{overview_pd_normal_dim2_voxel=0}
\end{figure}

\begin{figure}[t!]
    \centering
    \includegraphics[width=0.95\textwidth]{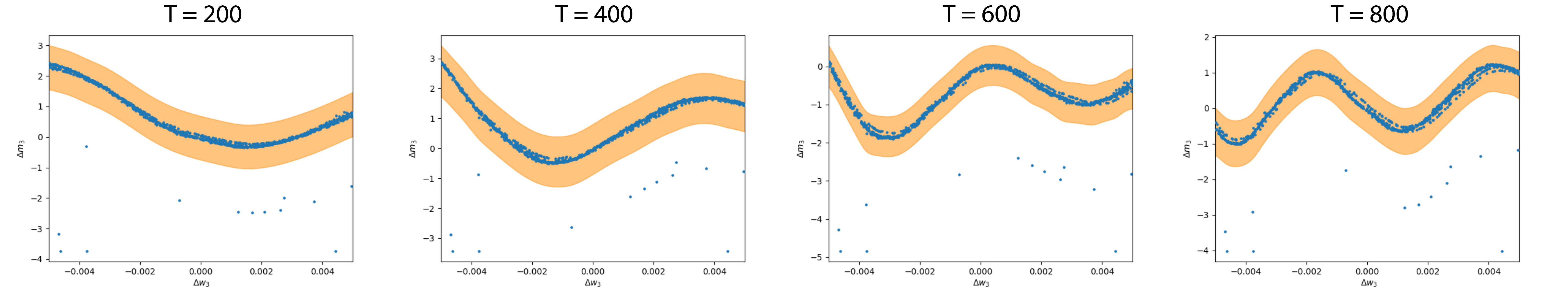}
    \caption{The time evolution of the GP approximated $\hat{g}_t$ for a nonlinear sinusoidal open-loop controller at some exemplary time instances.}
    \label{overview_sine_joint3_uniform_dim3_voxel=0}
\end{figure}


\begin{table}[htb!]
    \centering
    \caption{The aggregate performance of our method to predict physical derivatives in unseen directions of perturbations to the parameters. $\langle\cdot\rangle$ shows the time average. The first column is the normalized time-averaged MSE. The second column is the time-averaged GP score (closer to 1 is better. See~\Cref{sec:gp_score} for definition). The third column is the time-averaged misalignment between derivatives. Every experiment is repeated for 10 voxel sizes and the values are chosen for the best voxel size. N: Gaussian sampling, U: uniform sampling.}
    \begin{tabular}{lccc}
    \toprule
    Task          & $\langle\mathrm{MSE}\rangle$ & $\langle\mathrm{Score}\rangle$ & $\langle\mathrm{\cos \alpha}\rangle$ \\
\midrule
    PD controller (N) &   0.00871  &    0.8991   &   0.9492 \\
    PD controller (U) &  0.0018   &   0.9841    &  0.9723\\
Sine 2 joints (U) &  0.0368   &   0.7992    &  0.9513\\
Sine 2 joints (N) &  0.0796   &   0,6696    &   0,9553 \\
    \bottomrule
    \end{tabular}
\vspace{0.3cm}

\label{tab:overall}
\end{table}

\begin{figure}[h!]
    \centering
    \begin{subfigure}[t]{0.23\textwidth}
        \centering
        \includegraphics[width=0.97\textwidth]{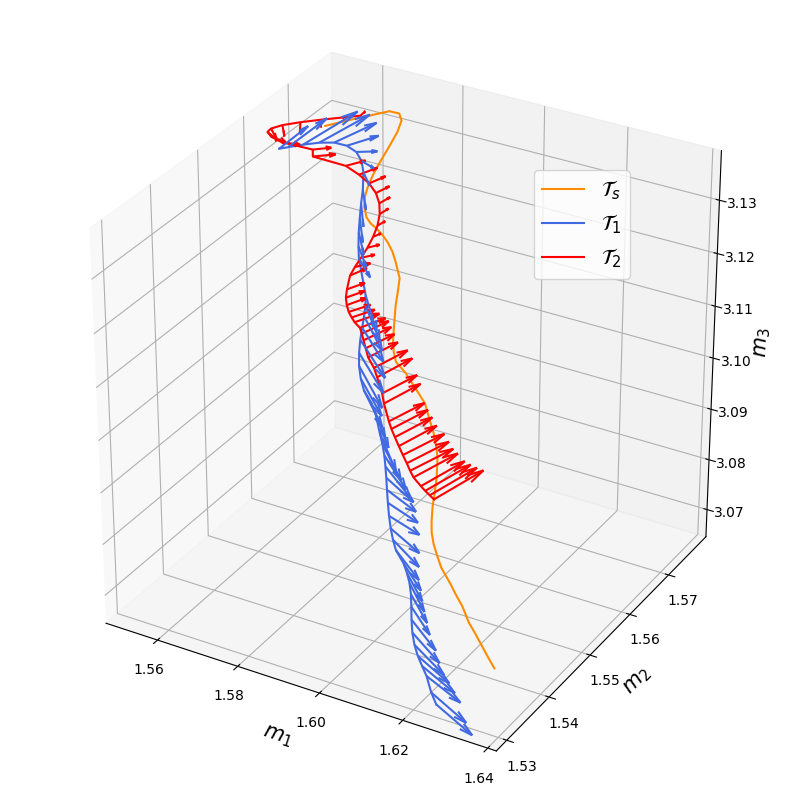}
        \caption{}        
    \end{subfigure}
    \begin{subfigure}[t]{0.23\textwidth}
        \centering
        \includegraphics[width=0.95\textwidth]{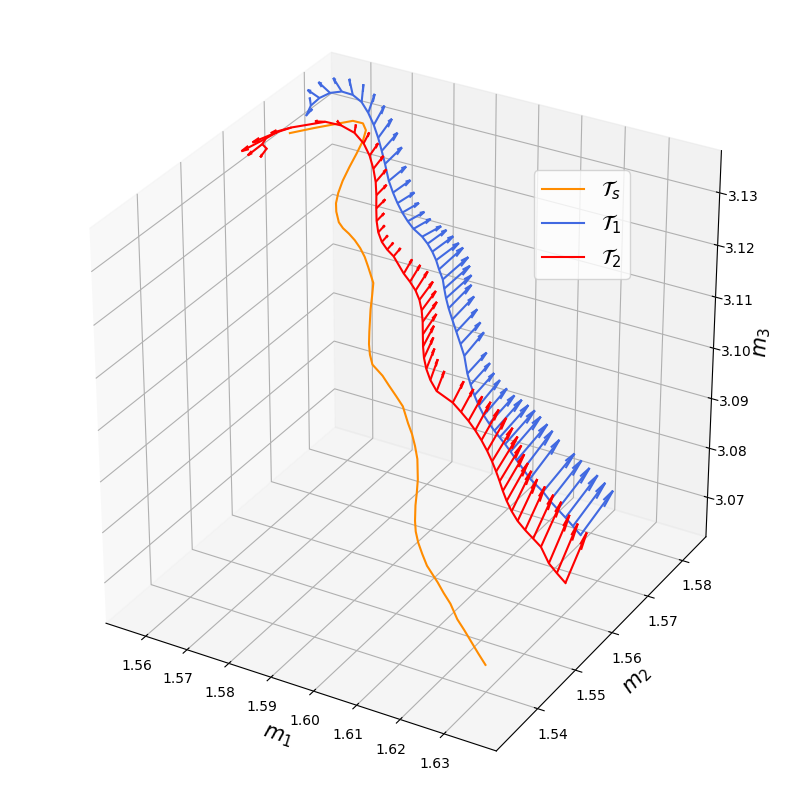}
        \caption{}    
   \end{subfigure}
   \begin{subfigure}[t]{0.23\textwidth}
        \centering
        \includegraphics[width=0.95\textwidth]{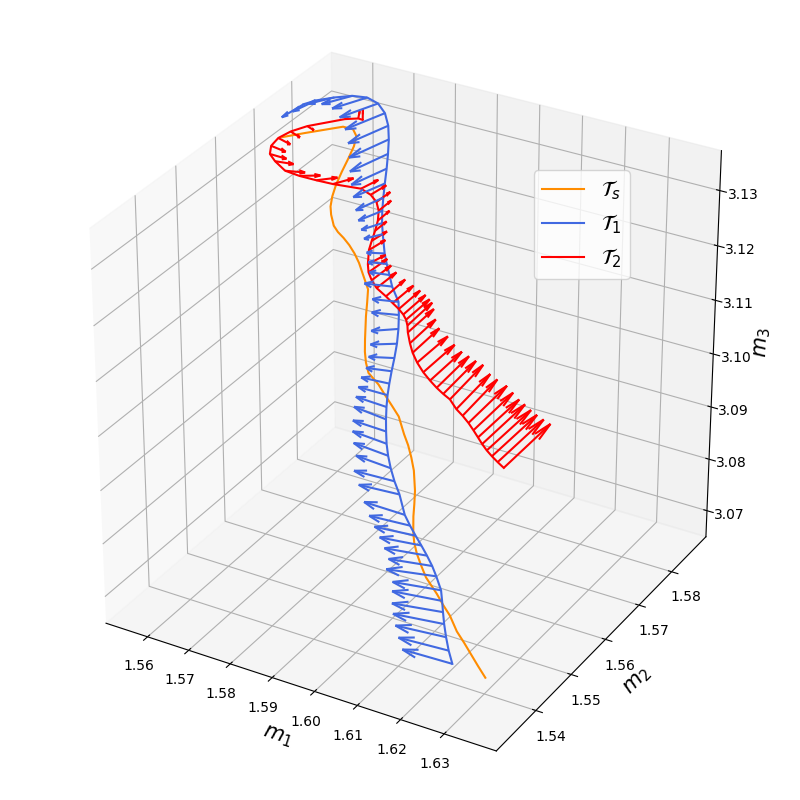}
        \caption{}    
   \end{subfigure}
   \begin{subfigure}[t]{0.23\textwidth}
        \centering
        \includegraphics[width=0.95\textwidth]{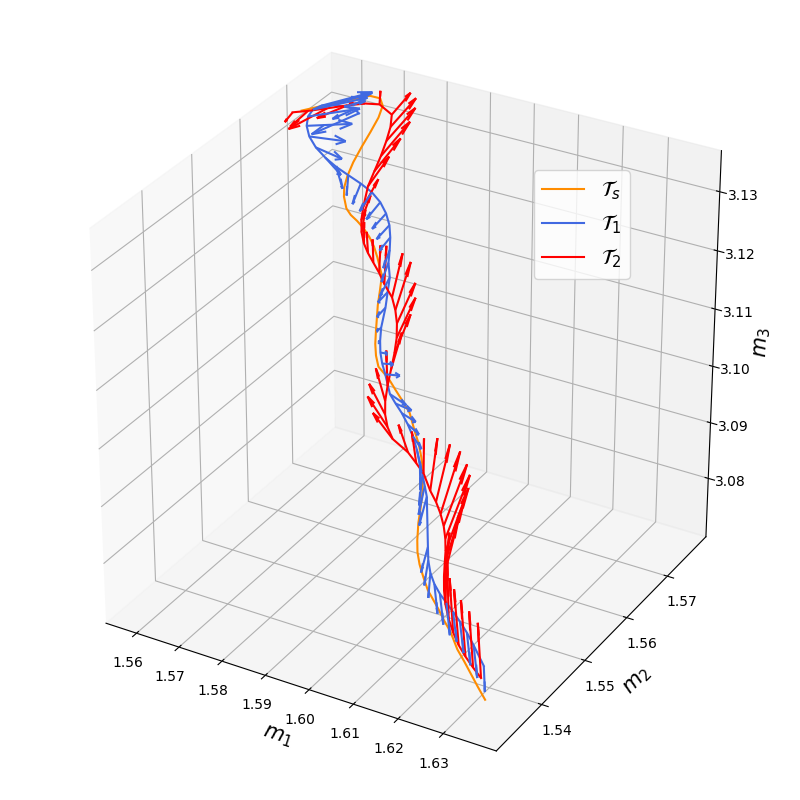}
        \caption{}    
    \end{subfigure}
    \caption{Physical gradients computed for various time steps along a source trajectory using two perturbed trajectories of linear open-loop (The orange trajectory is the source trajectory $\Tcal_s$ and others are the perturbed ones $\Tcal_1, \Tcal_2$. The quivers on the perturbed trajectories represent calculated physical derivatives). See further examples in~\Cref{fig:quivers_more}.}
    \label{fig:quivers}
\end{figure}

\subsection{Nonlinear open-loop controller}
\label{sec:nonlinear_openloop_controller}
Physical derivatives can naturally be computed for either linear or nonlinear controllers, which makes it different from taking the gradient of models through time. In model-based methods, if the model's transition dynamics is not differentiable, taking the derivative is theoretically challenging. However, our method takes advantage of the real physics of the system to compute the gradients regardless of whether the approximating model is differentiable or not. To elaborate more on this, we test our method for a simple but nonlinear policy, i.e., $u_t = \Acal\sin(\omega t)$. The sinusoidal torque is applied to either one or two motors of the system to investigate the performance of our method. We tested Gaussian and uniform perturbation for $\thetab=\{\Acal, \omega\}$ as parameters of this controller. The GP interpolation for the partial derivatives at some time instances along the trajectory can be seen in~\Cref{overview_sine_joint3_uniform_dim3_voxel=0} and more extensively in~\Cref{fig:sine_joint23_uniform_dim2_voxel=0,fig:sine_joint3_normal_dim3_voxel=0,fig:sine_joint3_uniform_dim3_voxel=0} in the Appendix. One might be interested in the direction of the predicted derivative instead of its exact size. To this end, we take several test perturbations for every time step and use $\cos(\alpha)$ as a measure of alignment between the predicted and ground-truth derivative vectors. The time evolution of the histogram of this measure along the trajectory shows a better alignment as time proceeds. This effect can be seen in~\Cref{fig:hist_sine_joint23_normal_voxel=0.01,fig:hist_sine_joint23_uniform_voxel=0.01}. This confirms our observation of initial transient noise in the system that dies out gradually by the progression of time. The overall performance of our method in predicting physical derivatives in unseen directions for two different perturbation methods is shown in~\Cref{tab:overall}.

\subsection{Feedback controller}
\label{sec:feedback_controller}
Often in practice, the policy incorporates some function of the states of the system. Some well-known examples which have been extensively used in control applications are P, PD, PI and PID controllers. Here, we consider two members of this family, i.e., P and PD controllers. The policy becomes $\ub = K_p \eb$ for P controllers and $\ub = K_p \eb + K_d \dot{\eb}$ for PD controllers. The error $\eb$ shows the difference between the current state $\xb$ and the desired state $\xb^*$. The parameters of the controller $\{K_p, K_d\}$ are scalar values that are multiplied by the error vector elementwise. This implies that the controller parameters are the same for three motors, leaving the whole platform's controller with two parameters that weigh the value and the rate of the error. We applied the uniform and Gaussian perturbation for the set of parameters $\thetab = \{K_p, K_d\}$ with different scenarios.~\Cref{fig:PD_perturbed_controller} is an illustration of the resultant trajectories for different levels of noise in the policy. The GP interpolation for the physical derivatives at some time instances along the trajectory can be seen in~\Cref{overview_sine_joint3_uniform_dim3_voxel=0} and more extensively in~\Cref{fig:pd_normal_dim1_voxel=0,fig:pd_normal_dim2_voxel=0,fig:pd_normal_dim3_voxel=0,fig:pd_uniform_dim1_voxel=0,fig:pd_uniform_dim2_voxel=0,fig:pd_uniform_dim3_voxel=0} in the Appendix. The time evolution of the histogram of misalignment between predicted and ground-truth directional derivatives (see~\Cref{fig:hist_pd_normal_voxel=0.01,fig:hist_sine_joint23_uniform_voxel=0.01} in the Appendix) once again confirms the existence of the initial transient noise, as was also observed in the~\Cref{sec:nonlinear_openloop_controller}. Similar to the sinusoidal experiment, the overall performance of our method is presented in~\Cref{tab:overall}.

\subsection{Zero-shot planning task}
\label{sec:zeroshot_planning}
Our previous experiments in sections \ref{sec:linear_openloop_controller},\ref{sec:nonlinear_openloop_controller} and \ref{sec:feedback_controller} showed that learning the physical derivative map is feasible for various types of controllers. In this section, we demonstrate an example of a constraint satisfaction task by means of the physical derivative map. In this experiment, the superscript $(s)$ corresponds to the nominal trajectory, which is called~\emph{source}. Assume the system is controlled by a PD controller to reach a target state $\xb^*$, i.e., the control torques are designed as $\ub = k_p^{(s)}(\xb - \xb^*)+k_d^{(s)} \dot{\xb}$. The controller does a decent job of reaching the target state given reasonable values for $k_p$ and $k_d$. However, such a controller does not give us a clear way to shape the trajectory that starts from $\xb_\circ$ and ends at $\xb^*$. 
Assume it is desired that the nominally controlled trajectory $\Tcal^{(s)}$ passes through an intermediate state $\xb_t^*$ at time $t$ on its way towards the target state $\xb^*$ (we can equally assume that the system must avoid some regions of the state space because of safety reasons). 
The solution with physical derivatives is as follows . Assume $k_d^{(s)}$ is fixed and only $k_p^{(s)}$ is changeable. If the physical derivatives map is available, we have access to $\hat{g}_t(k_p^* - k_p^{(s)}) = (\xb_t^*- \xb_t^{(s)}) / (k_p^* - k_p^{(s)})$. By simple algebraic rearrangement, we have
\begin{equation}
    \label{eq:constrain_satisfaction_update_equation}
    k_p^* = \frac{\xb^* - \xb_t^{(s)}}{\hat{g}_t(k_p^* - k_p^{(s)})} + k_p^{(s)}.
\end{equation}
The new parameter of the policy is supposed to push the source trajectory $\Tcal^{(s)}$ towards a target trajectory $\Tcal^*$ that passes through the desired state $\xb^*_t$ at time $t$. The result of this experiment on our physical finger platform is available in~\Cref{fig:constraint_satisfaction}.

\begin{figure}[t!]
    \centering
    \begin{subfigure}[b]{0.32\textwidth}
        \centering
        \includegraphics[width=0.9\textwidth]{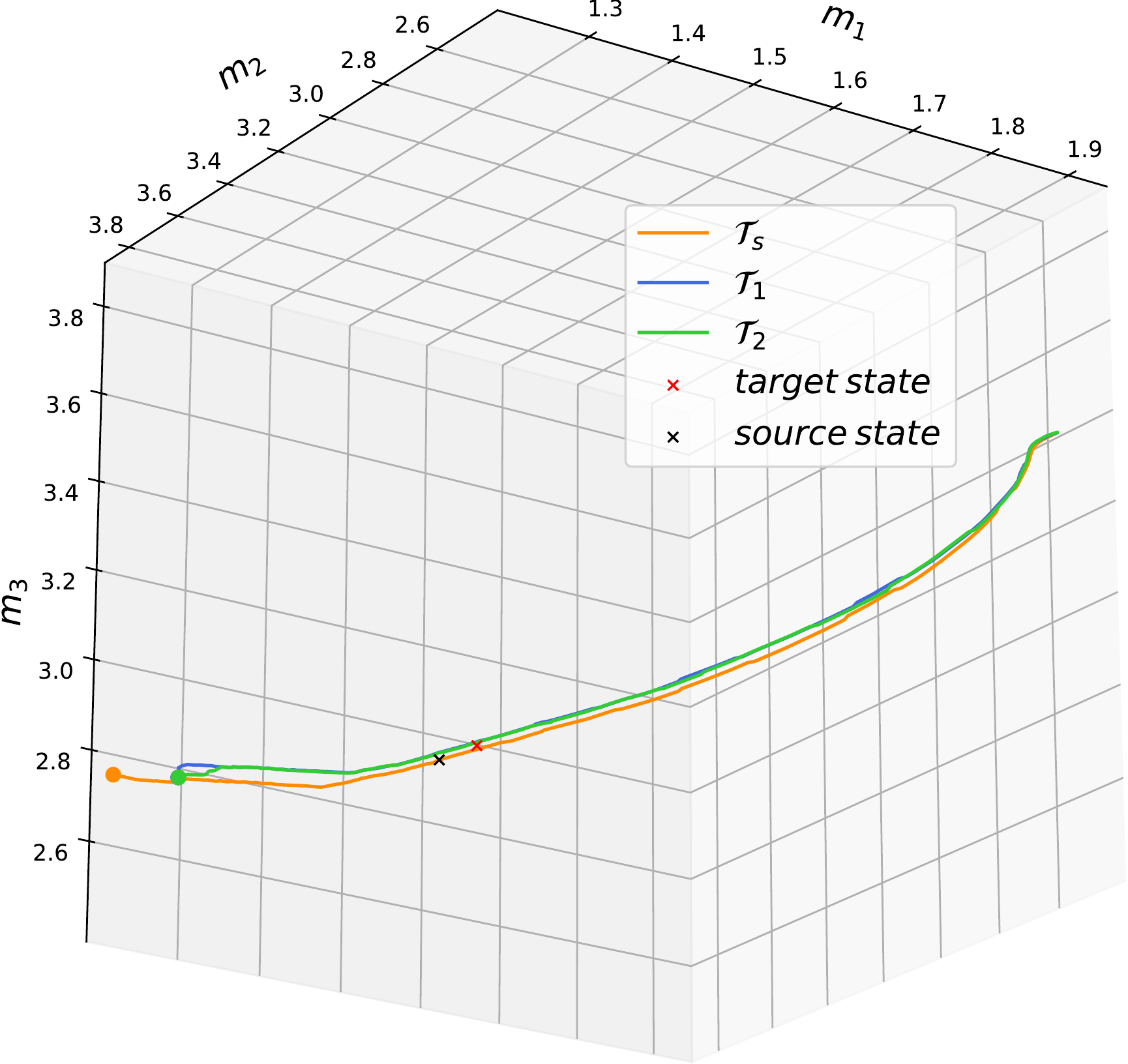}
        \caption{Short distance target}        
    \end{subfigure}
    \begin{subfigure}[b]{0.32\textwidth}
        \centering
        \includegraphics[width=0.9\textwidth]{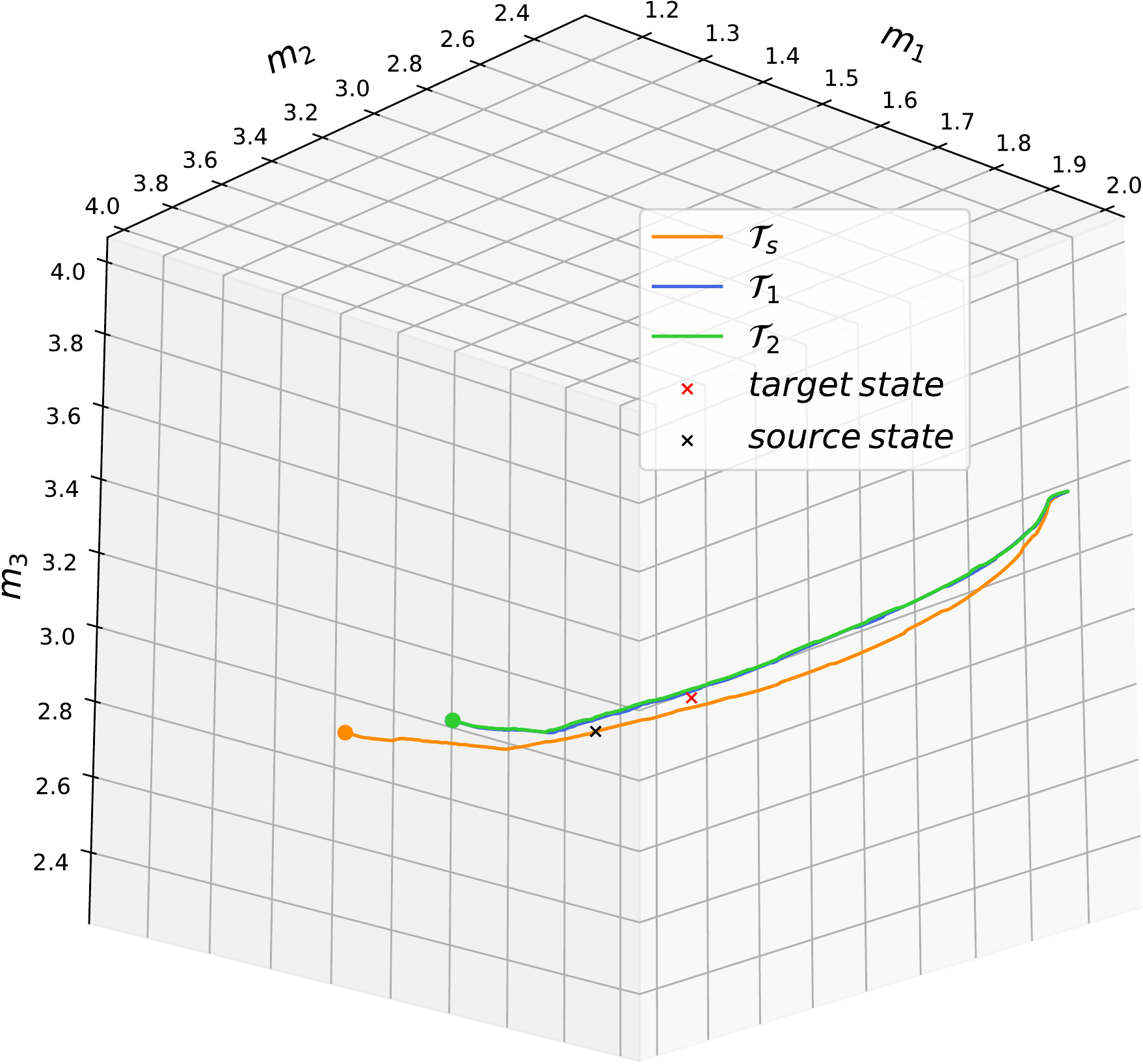}
        \caption{Medium distance target}        
    \end{subfigure}
    \begin{subfigure}[b]{0.32\textwidth}
        \centering
        \includegraphics[width=0.9\textwidth]{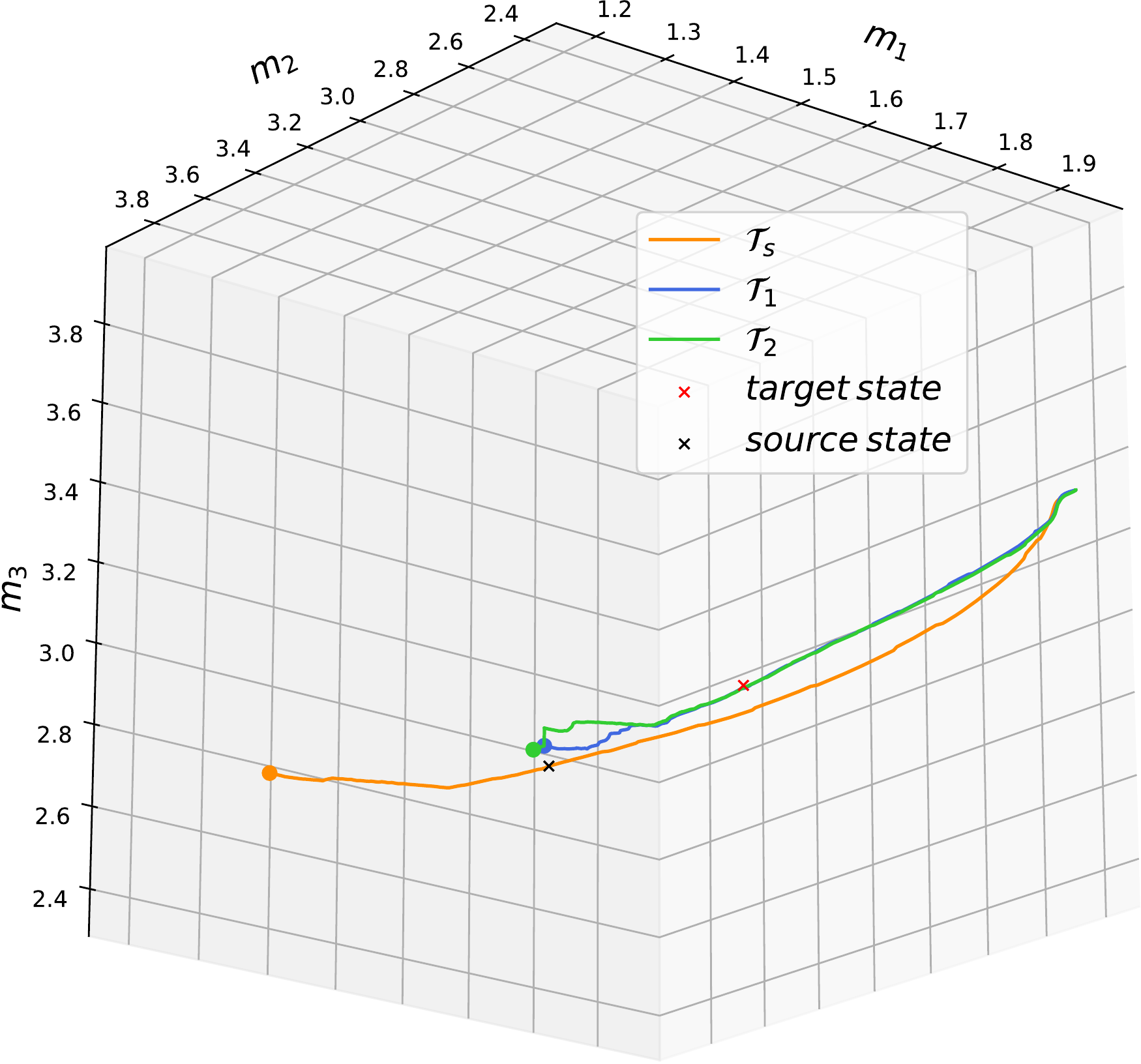}
        \caption{Long distance target}        
    \end{subfigure}
    \caption{Zero-shot planning with constraint satisfaction. The orange trajectory is the source produced by the nominal controller. The green and blue are two sampled trajectories produced by perturbing $k_p$ to $k_p^*$ by~\Cref{eq:constrain_satisfaction_update_equation}.}\label{fig:constraint_satisfaction}
\end{figure}

\section{Conclusion}
In this paper, we present a method to learn how the trajectories of a physical real-world dynamical system change with respect to a change in the policy parameters. We tested our method on a custom-built platform called finger robot that allows testing a couple of controllers with various settings to show the applicability of our method for linear, nonlinear, open-loop, and feedback controllers. By estimating the physical derivative function, we showed that our method is able to push a controlled trajectory towards a target intermediate state. We investigate the real-world challenges when doing a fine sensitive task such as estimating physical derivatives on a real robot and proposed solutions to make our algorithm robust to inherent imperfection and noise in physical systems. We focused mainly on low-level issues of physical derivatives and showed the feasibility of estimating them robustly. We expect that physical derivatives will contribute to the research areas such as safety, control with constraint satisfaction and trajectory planning, robust, or safe control. 


\newpage
\bibliography{l4dc2022-manuscript}

\newpage
\appendix

\section{Physical Platform}
\label{sec:physical_platform}
In this section, we introduce the physical robot on which we tested our method. The robot is called \emph{finger platform} or simply \emph{finger} throughout this paper. The range of movement for the motors are $[0, \pi]$, $[0, \pi]$, $[0, 2\pi]$ respectively. The axes of the plots throughout the paper are in radian. It consists of three articulated arms with three degrees of freedom in total (see~\Cref{fig:whole_finger}). The motors $\{m_1, m_2, m_3\}$ are depicted in the figure. This naming remains consistent throughout this paper. Each arm is moved by a separate brushless DC motor and has one degree of freedom to swing in its own plane (see~\Cref{fig:motor_1}). Each arm is equipped with an encoder that measures its angle (see~\Cref{fig:encoder}). The brushless motors are controlled by an electronic driver that receives torque values applied to each motor from a computer terminal via a CAN bus and applies the torques to the motors (see~\Cref{fig:electronics_board}). Due to the imperfections of the arms, motors, and drivers, we did not use any model for the system, including the inertial matrix of the robot or the current-torque characteristic function of the motors. The low-cost and safe nature of this robot makes it a suitable platform to test the idea of physical derivatives that requires applying many different controllers in the training phase.

\begin{figure}[h!]
    \centering
    \begin{subfigure}[t]{0.23\textwidth}
        \centering
        \includegraphics[height=0.87in]{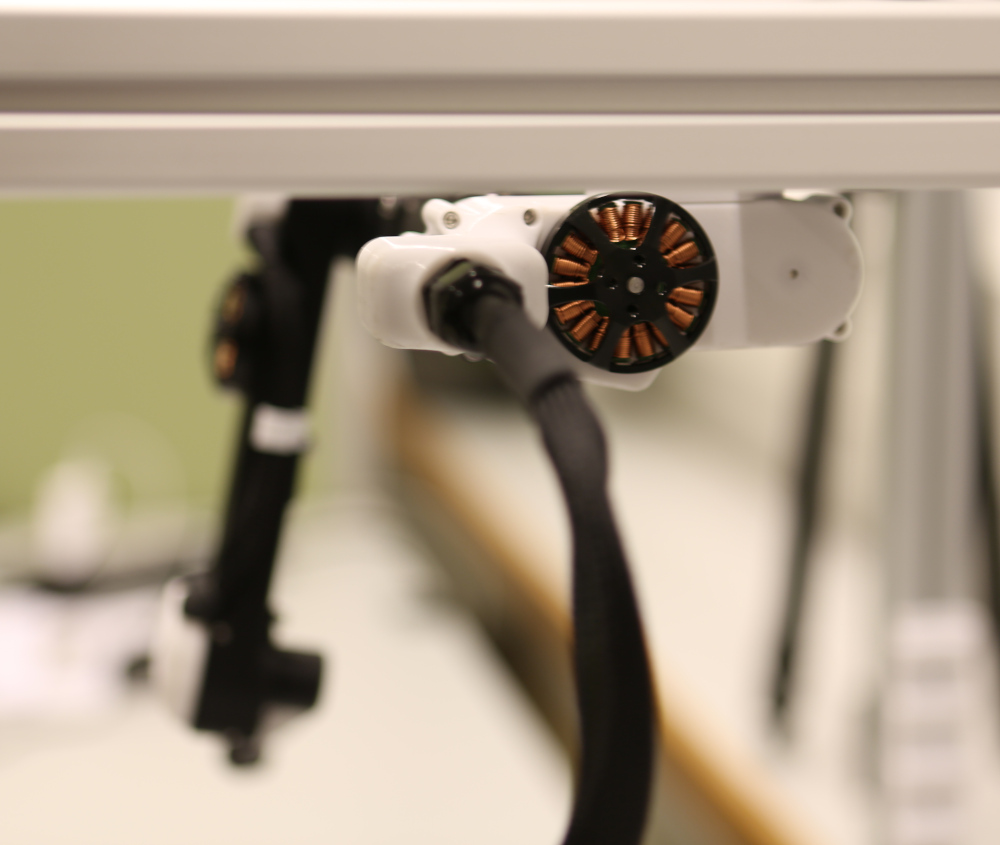}
        \caption{Motor $1$}\label{fig:motor_1}        
    \end{subfigure}
    \begin{subfigure}[t]{0.23\textwidth}
        \centering
        \includegraphics[height=0.86in]{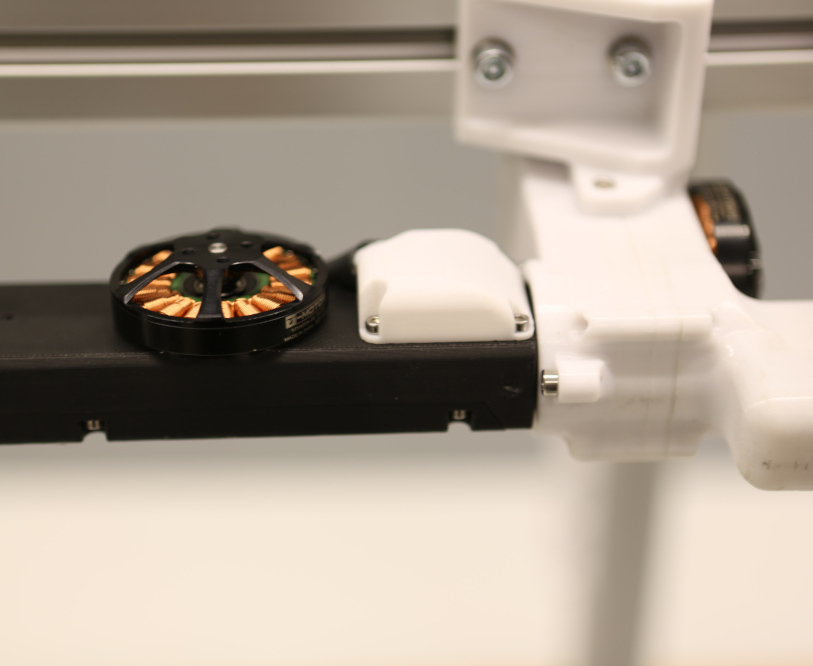}
        \caption{Encoder}\label{fig:encoder}    
   \end{subfigure}
   \begin{subfigure}[t]{0.23\textwidth}
        \centering
        \includegraphics[height=0.86in]{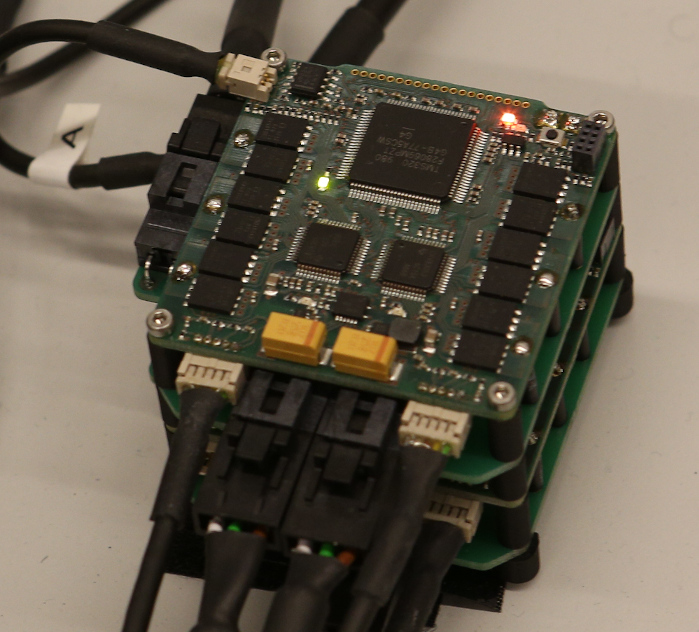}
        \caption{Driver}\label{fig:electronics_board}    
   \end{subfigure}
   \begin{subfigure}[t]{0.23\textwidth}
        \centering
        \includegraphics[height=0.87in]{figs/motor1.jpg}
        \caption{Finger}\label{fig:whole_finger}    
    \end{subfigure}
    \caption{Components of the physical finger platform}\label{fig:platform_components}
\end{figure}

\section{Additional Plots illustrating Real World Challenges (section \ref{sec:real_challenges}) }
\label{app:challenges}

\begin{figure}[h!]
    \centering
    \begin{subfigure}[t]{1in}
        \centering
        \includegraphics[width=1in]{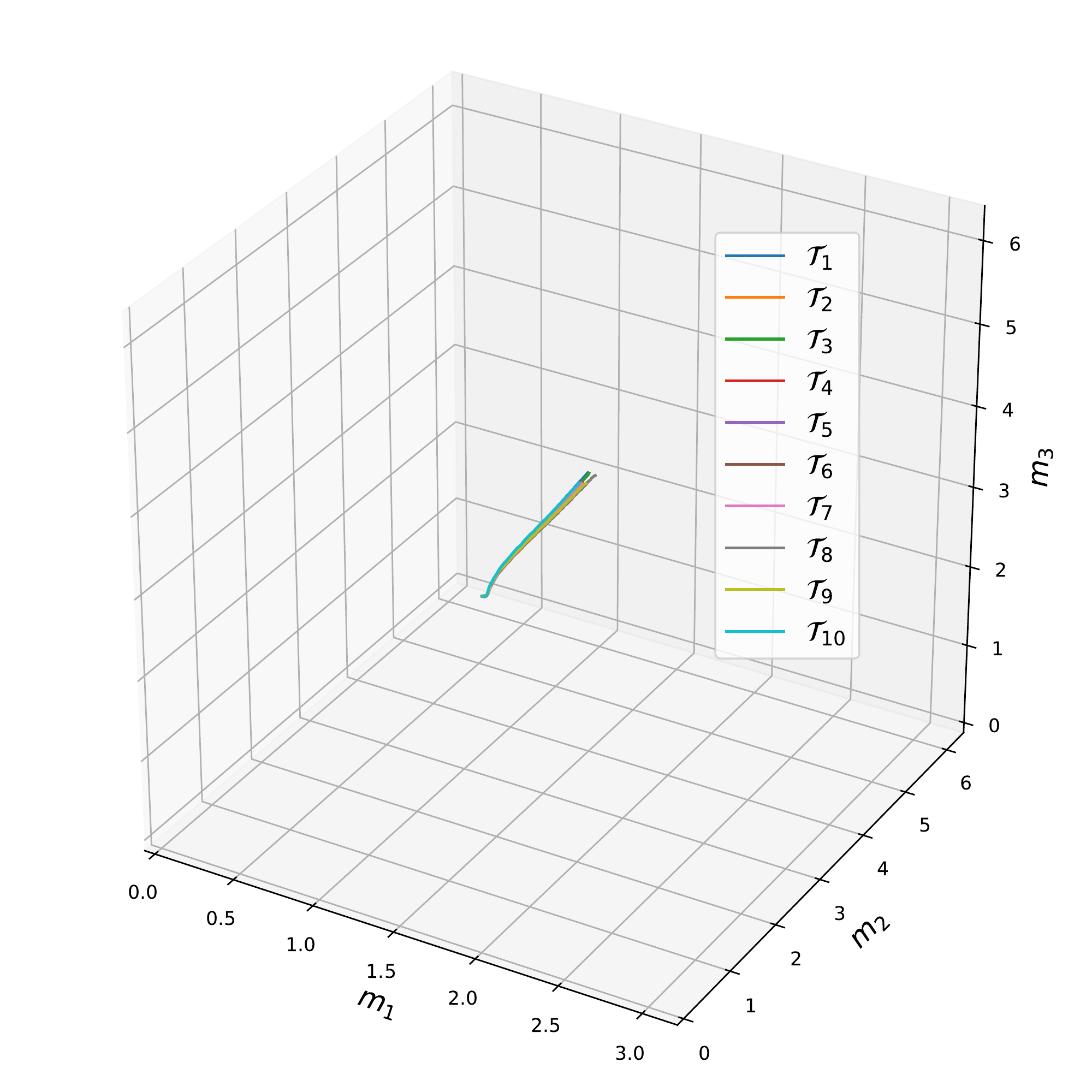}
        \caption{$t=200$}        
    \end{subfigure}
    \begin{subfigure}[t]{1in}
        \centering
        \includegraphics[width=1in]{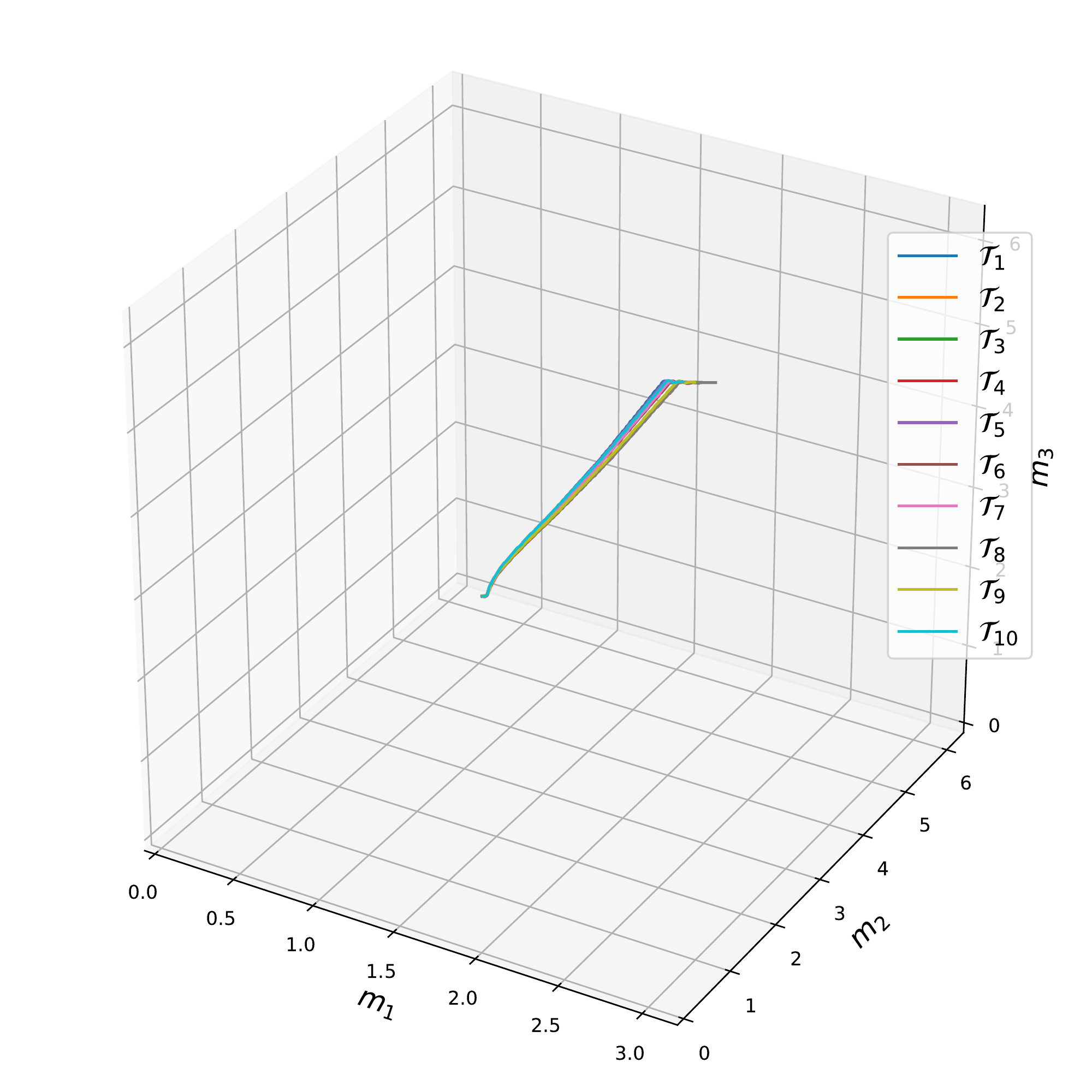}\caption{$t=400$}    
   \end{subfigure}
   \begin{subfigure}[t]{1in}
        \centering
        \includegraphics[width=1in]{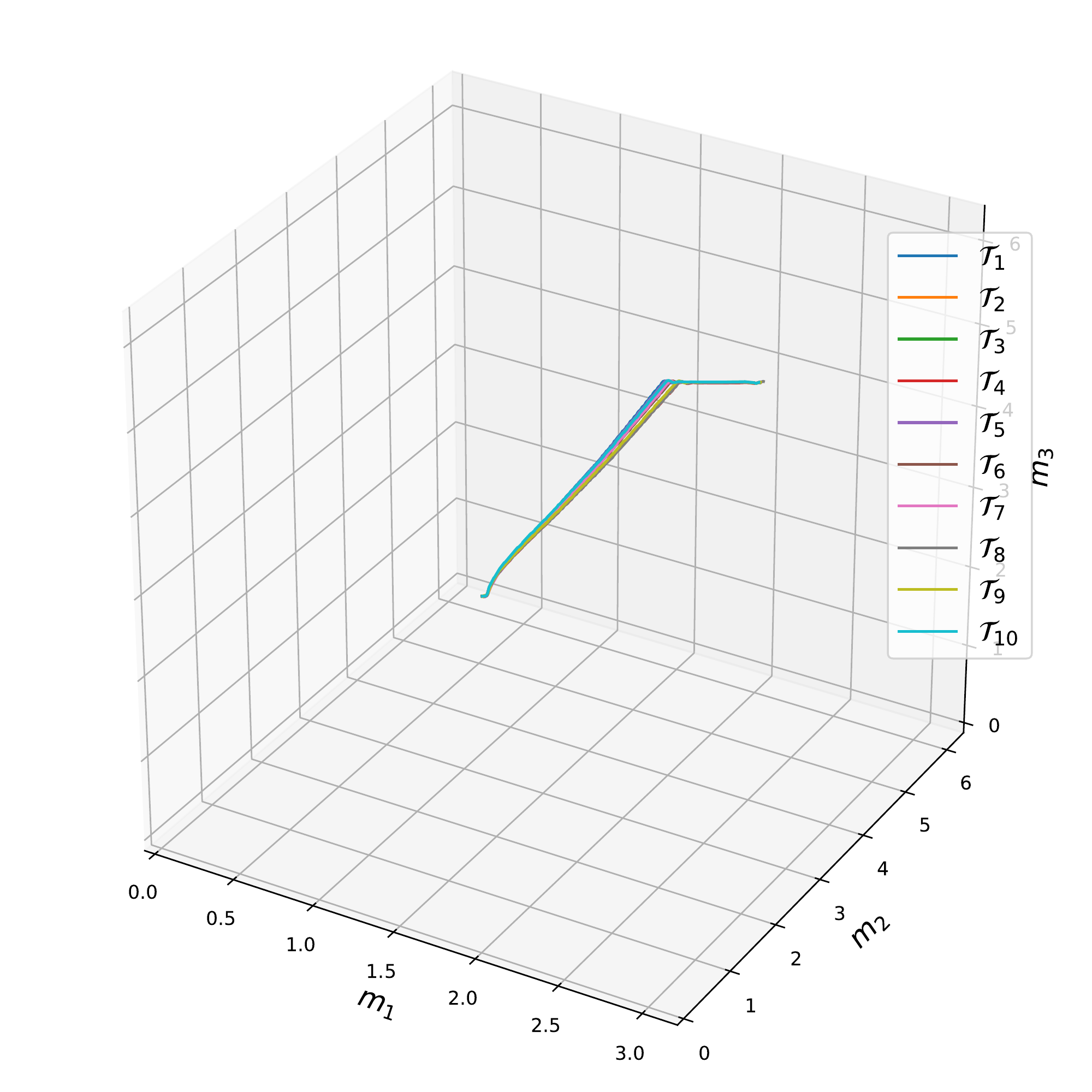}\caption{$t=600$}    
   \end{subfigure}
   \begin{subfigure}[t]{1in}
        \centering
        \includegraphics[width=1in]{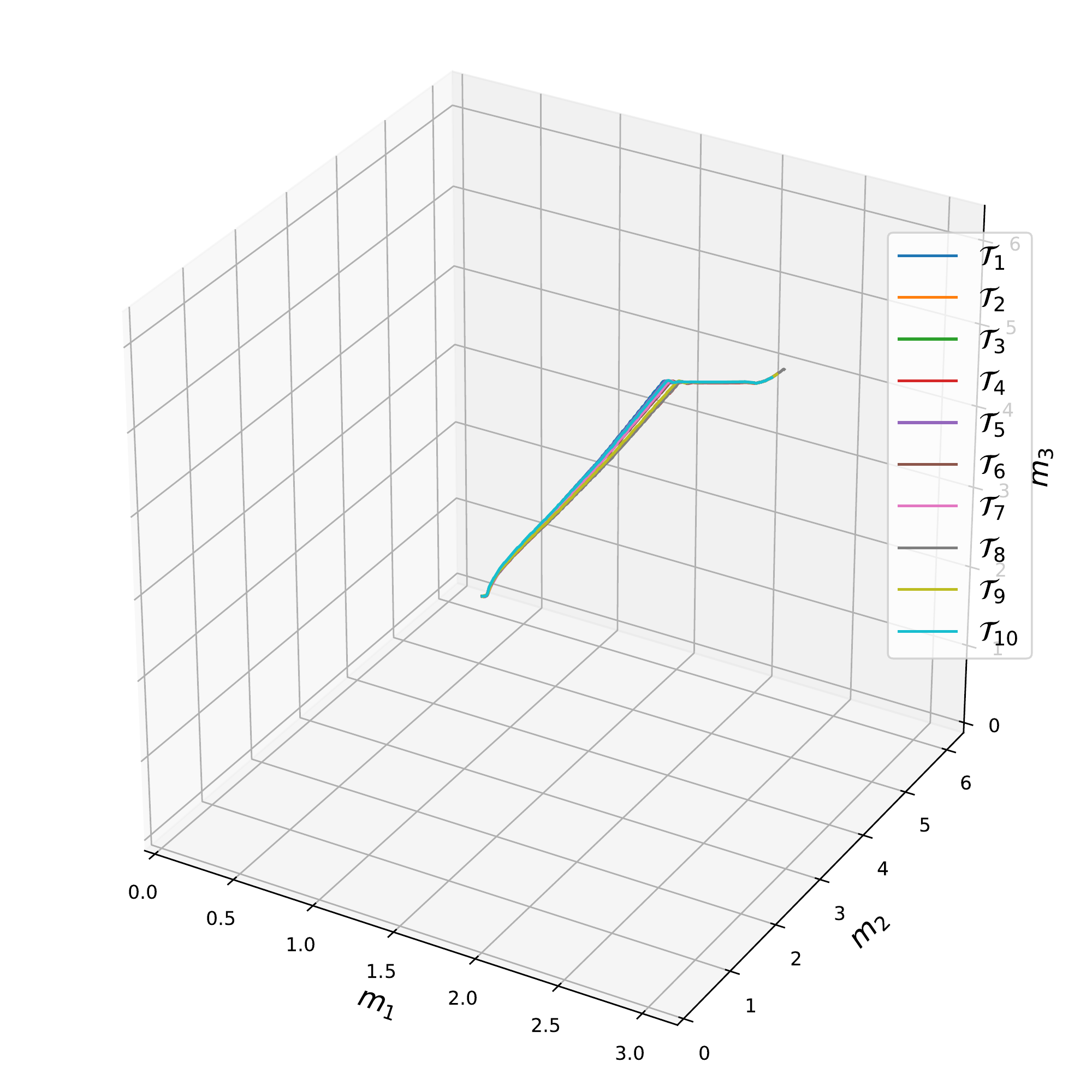}\caption{$t=800$}    
    \end{subfigure}
    \begin{subfigure}[t]{1in}
        \centering
        \includegraphics[width=1in]{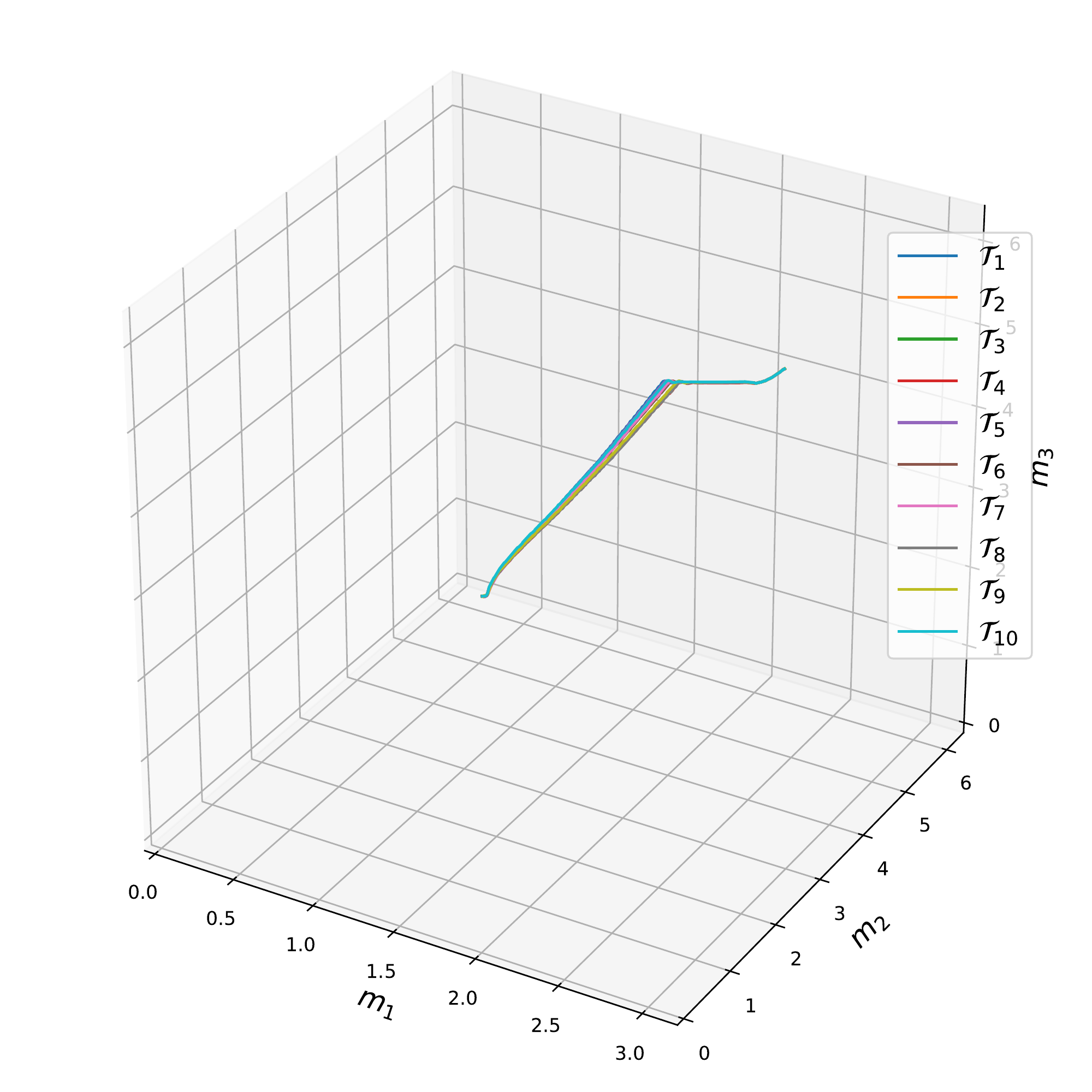}\caption{$t=1000$}    
    \end{subfigure}
    \caption{Same controller applied for multiple runs. The trajectories are produced by the linear open-loop controller similar to those used in~\Cref{sec:linear_openloop_controller}. See the plot for a different set of nominal parameters of the controller in~\Cref{fig:linear_noisy_openloop_controller_setting2} in the Appendix (Zooming is recommended).}
    \label{fig:same_controller_multiple_runs_setting1}
\end{figure}

\begin{figure}[h!]
    \centering
    \begin{subfigure}[t]{1in}
        \centering
        \includegraphics[width=1in]{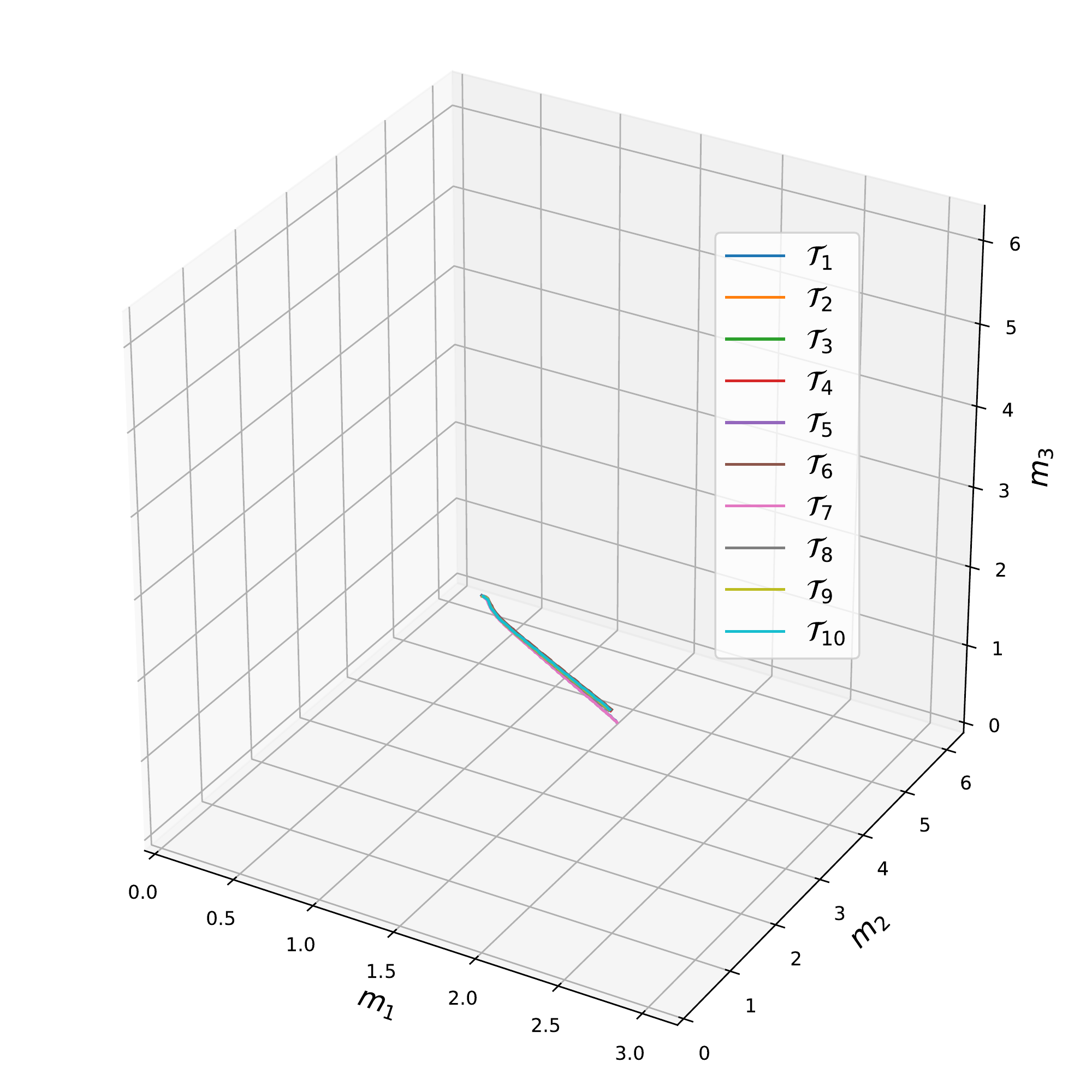}
        \caption{$t=200$}        
    \end{subfigure}
    \begin{subfigure}[t]{1in}
        \centering
        \includegraphics[width=1in]{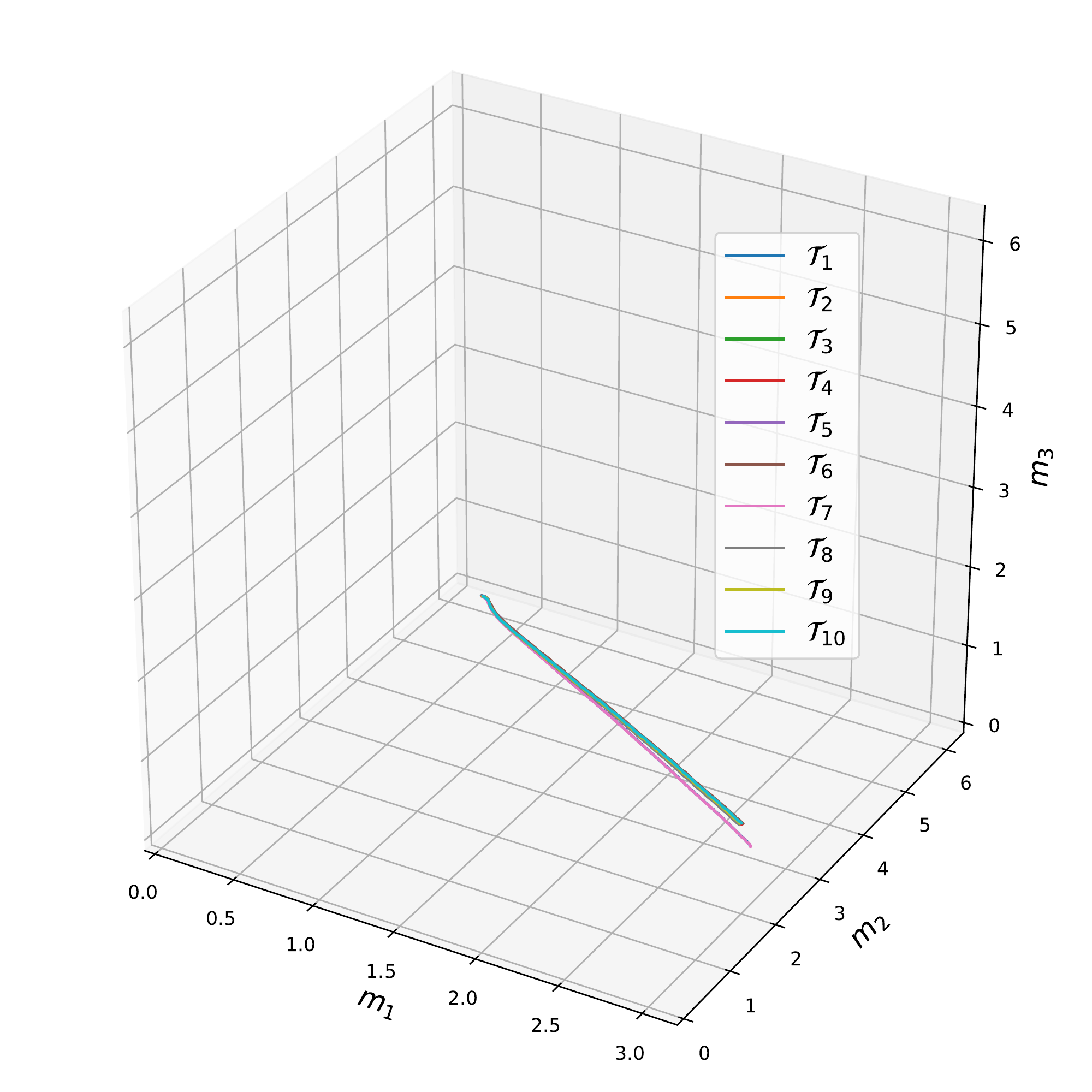}\caption{$t=400$}    
   \end{subfigure}
   \begin{subfigure}[t]{1in}
        \centering
        \includegraphics[width=1in]{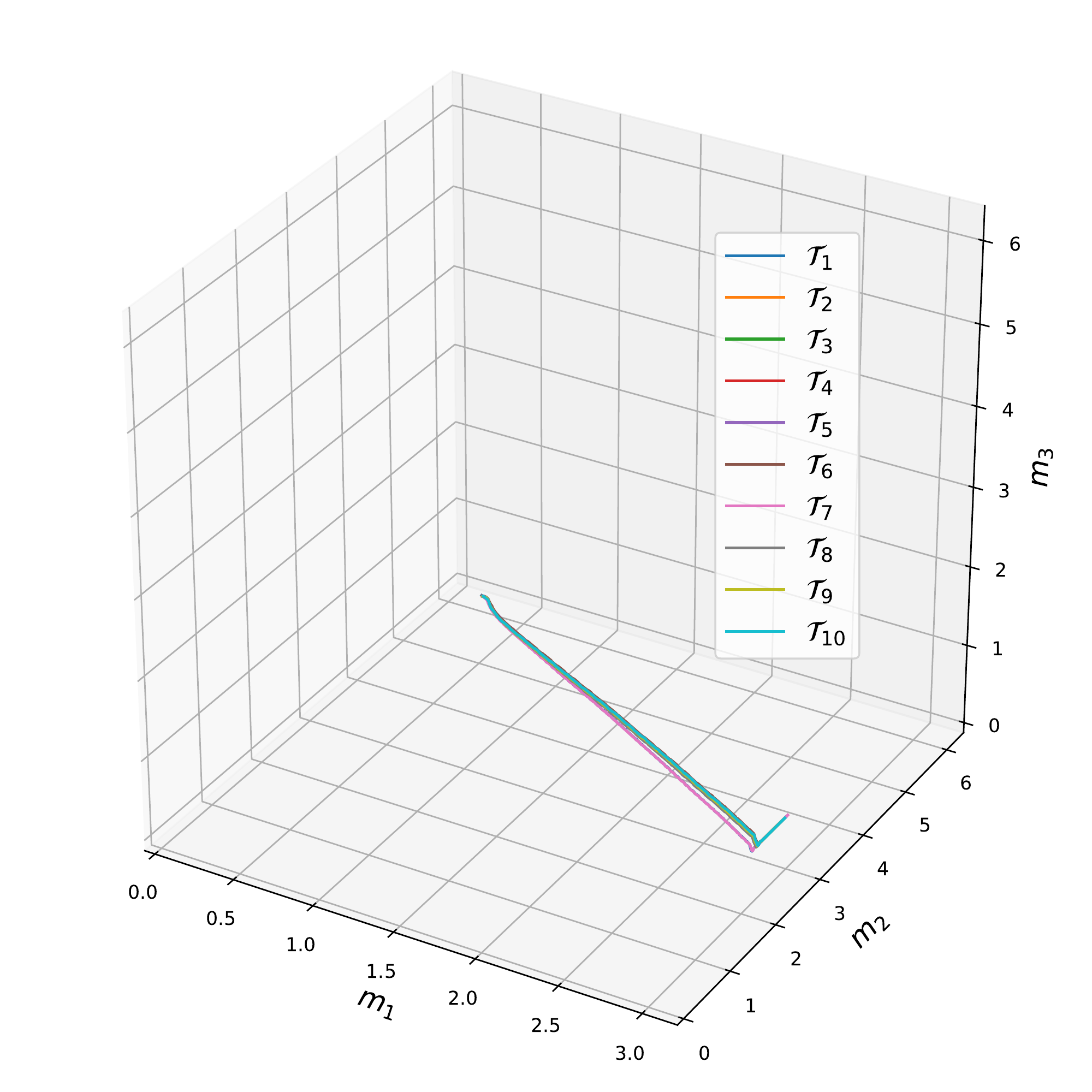}\caption{$t=600$}    
   \end{subfigure}
   \begin{subfigure}[t]{1in}
        \centering
        \includegraphics[width=1in]{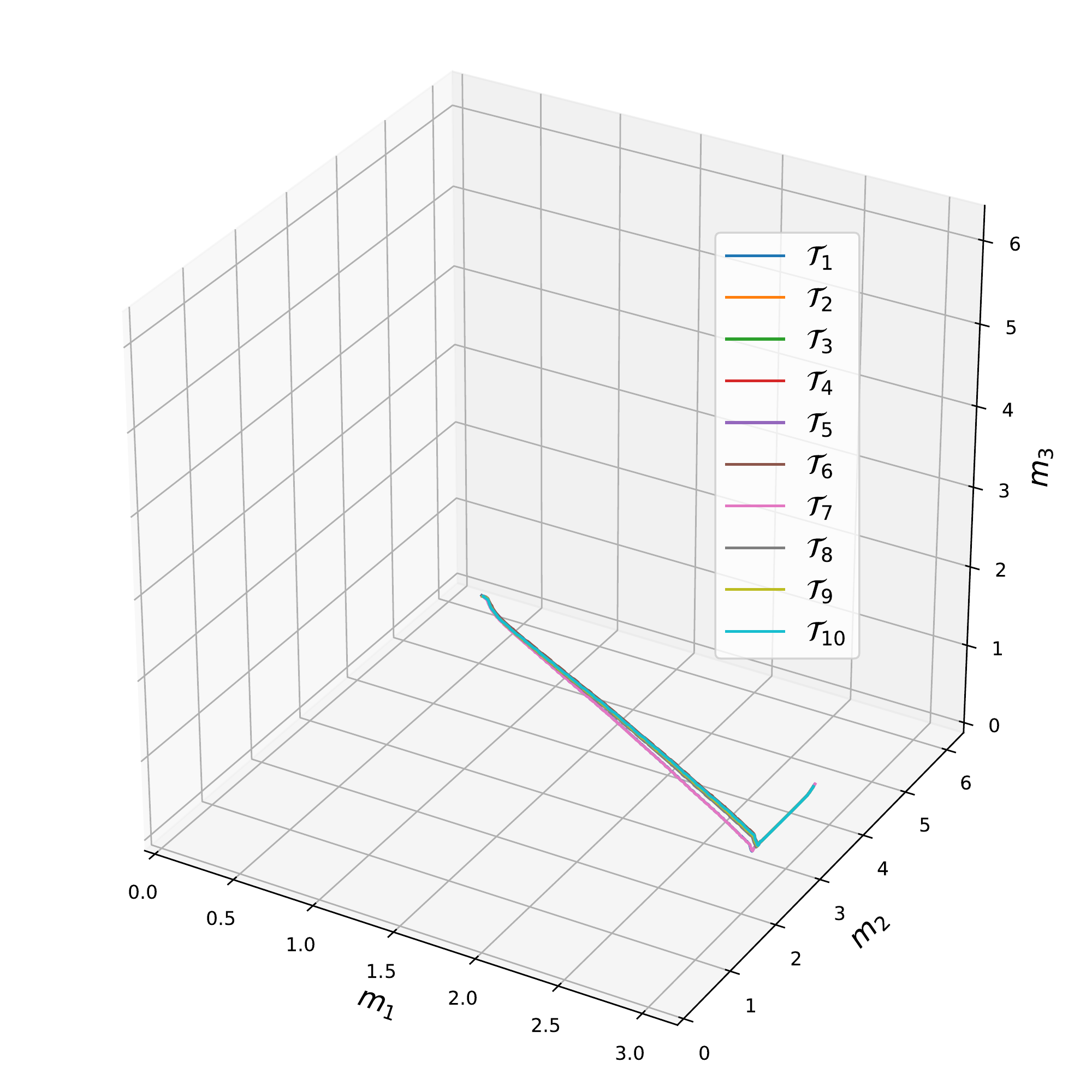}\caption{$t=800$}    
    \end{subfigure}
    \begin{subfigure}[t]{1in}
        \centering
        \includegraphics[width=1in]{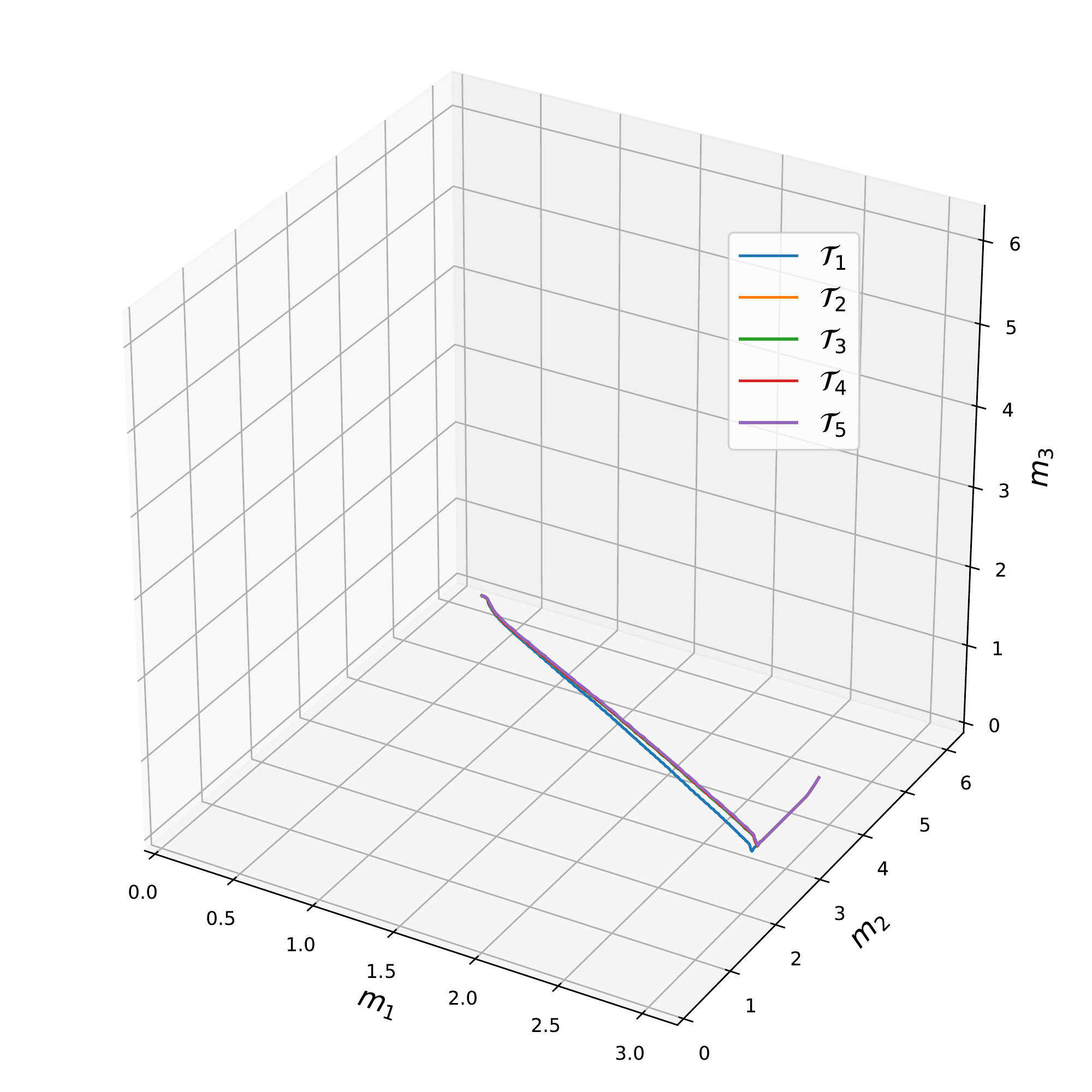}\caption{$t=1000$}    
    \end{subfigure}
    \caption{The same as~\Cref{fig:same_controller_multiple_runs_setting1} but for a different setting.}\label{fig:same_controller_multiple_runs_setting2}
\end{figure}

\begin{figure}[t!]
    \centering
    \begin{subfigure}[t]{1in}
        \centering
        \includegraphics[width=1in]{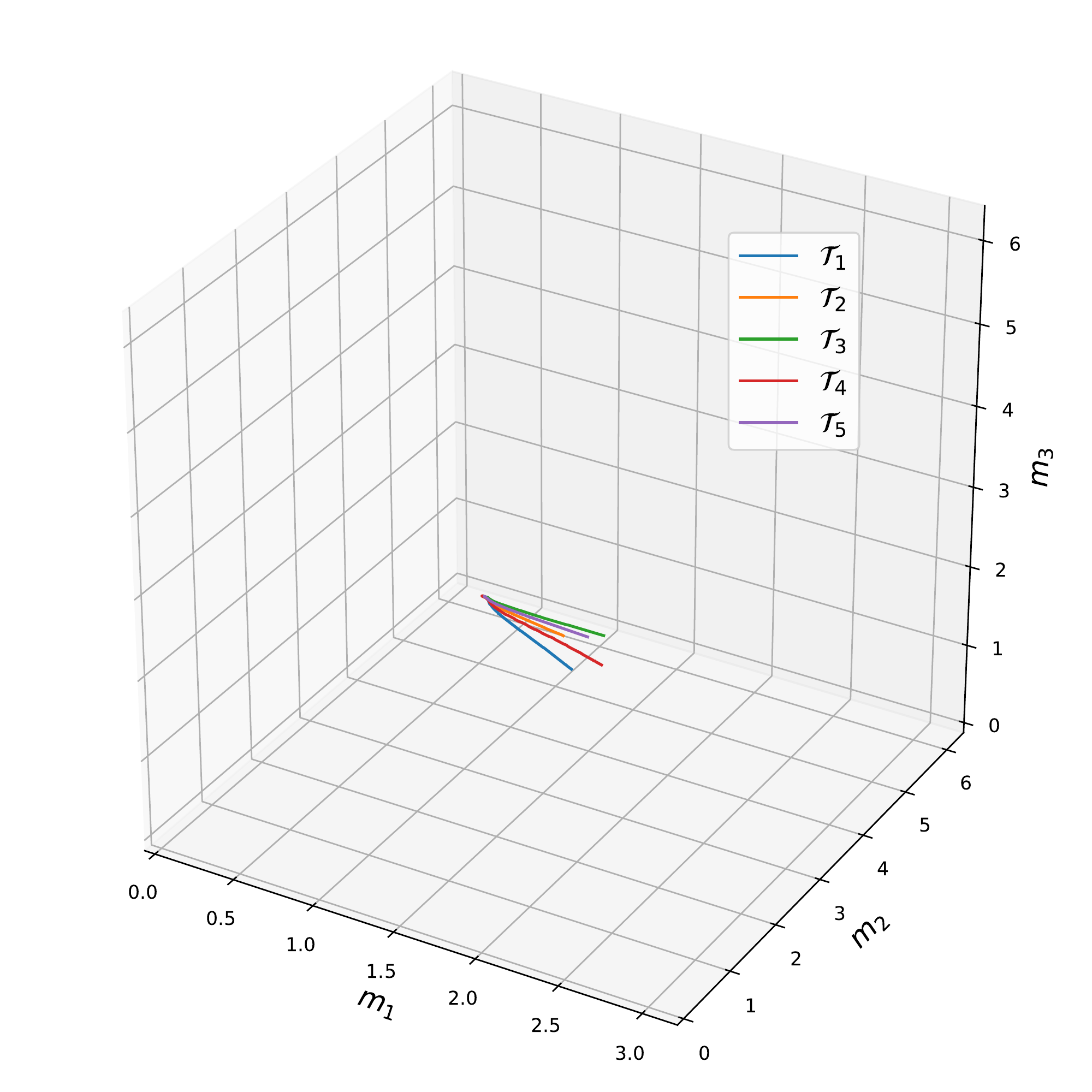}
        \caption{$t=200$}        
    \end{subfigure}
    \begin{subfigure}[t]{1in}
        \centering
        \includegraphics[width=1in]{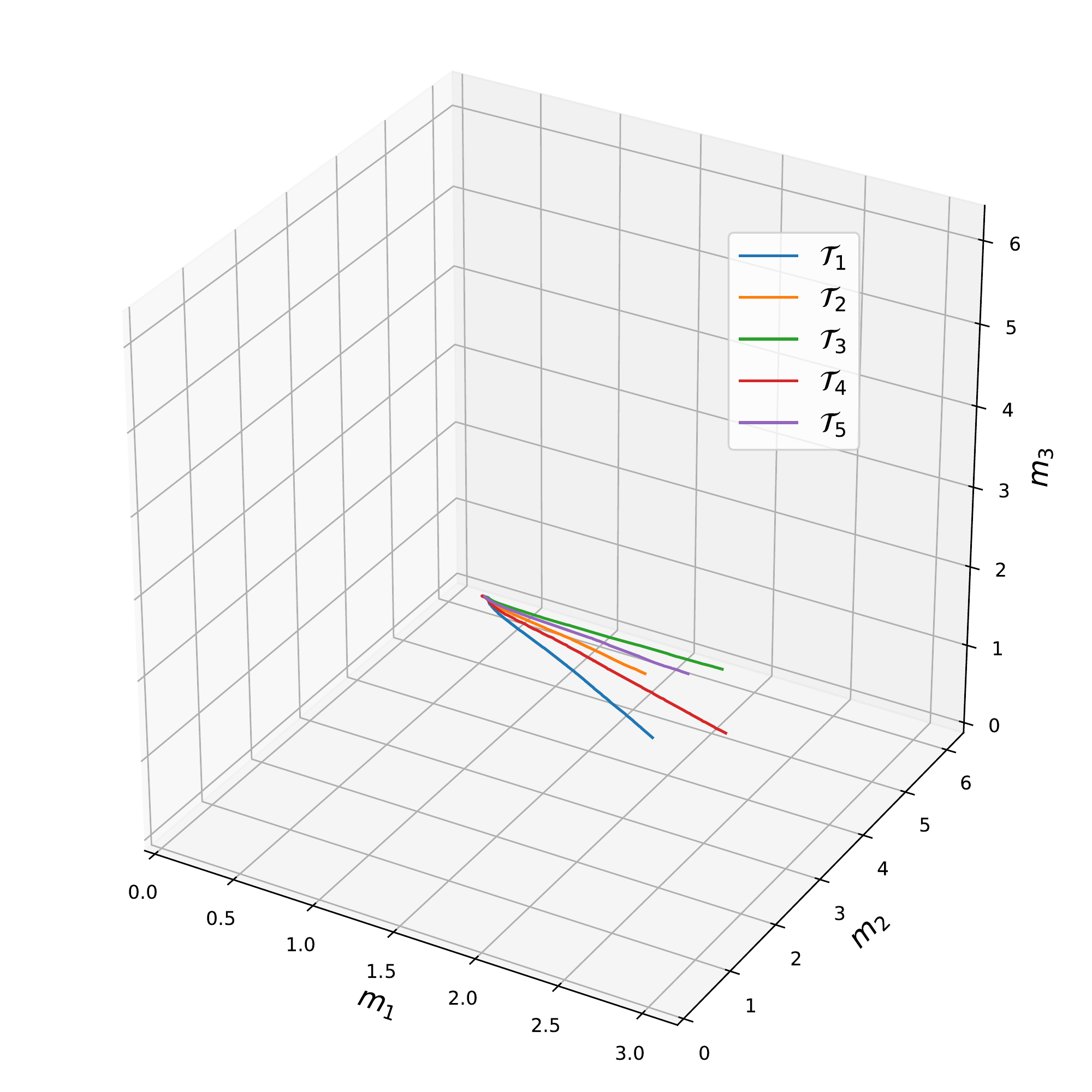}\caption{$t=400$}    
   \end{subfigure}
   \begin{subfigure}[t]{1in}
        \centering
        \includegraphics[width=1in]{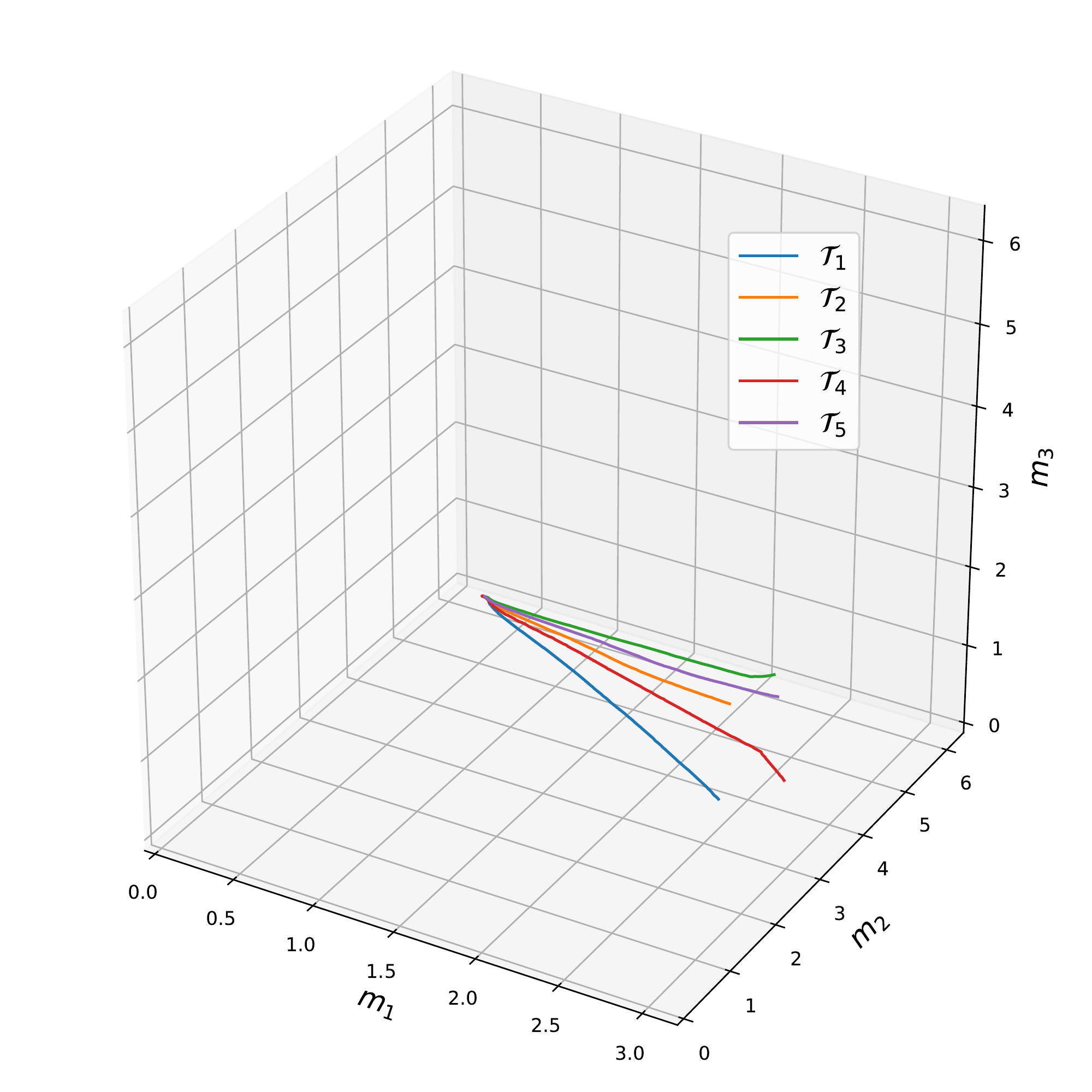}\caption{$t=600$}    
   \end{subfigure}
   \begin{subfigure}[t]{1in}
        \centering
        \includegraphics[width=1in]{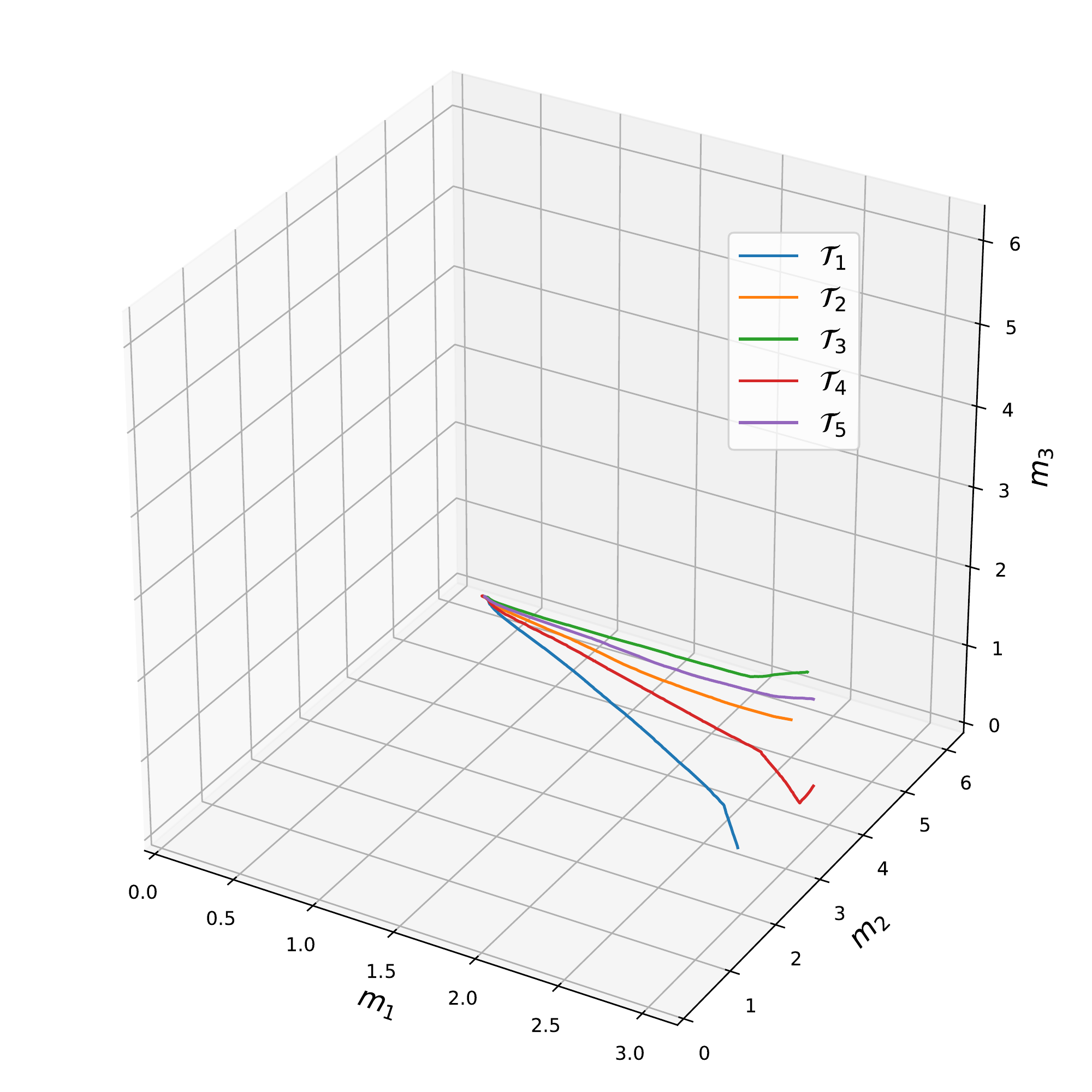}\caption{$t=800$}    
    \end{subfigure}
    \begin{subfigure}[t]{1in}
        \centering
        \includegraphics[width=1in]{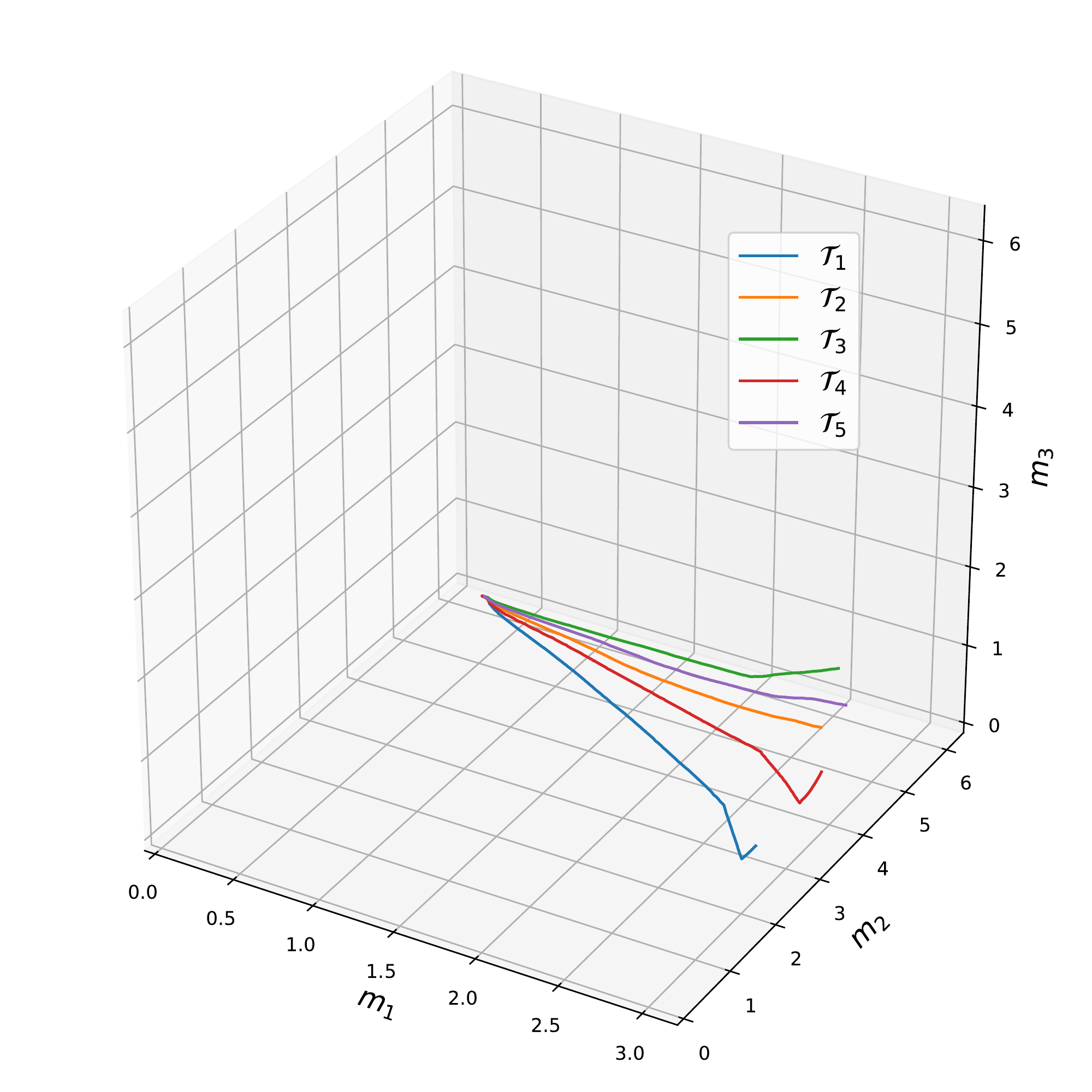}\caption{$t=1000$}    
    \end{subfigure}
    \caption{Noisy linear open-loop controller.}\label{fig:nonlinear_noisy_openloop_controller_setting1}
\end{figure}

\begin{figure}[h!]
    \centering
    \begin{subfigure}[t]{1in}
        \centering
        \includegraphics[width=1in]{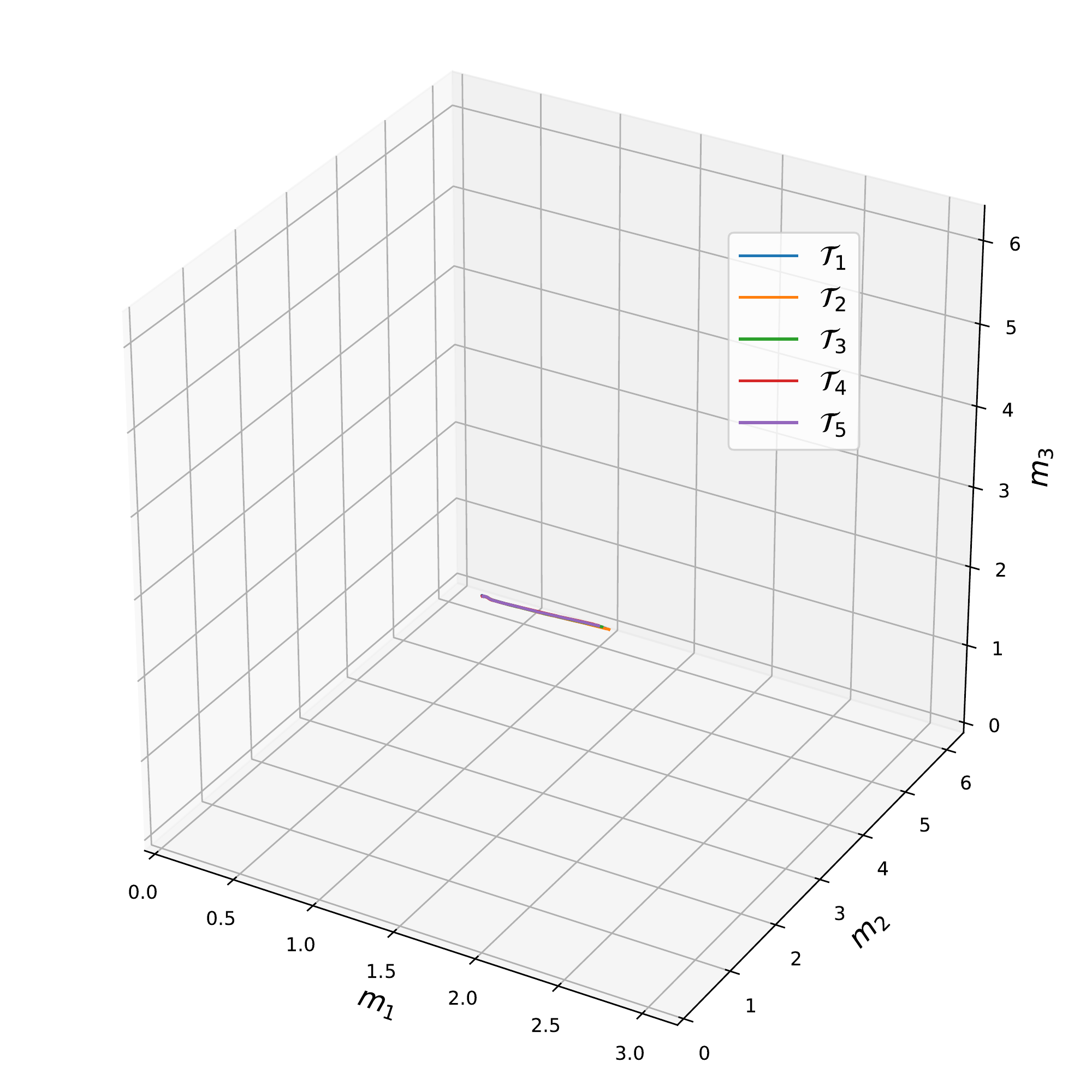}
        \caption{$t=200$}        
    \end{subfigure}
    \begin{subfigure}[t]{1in}
        \centering
        \includegraphics[width=1in]{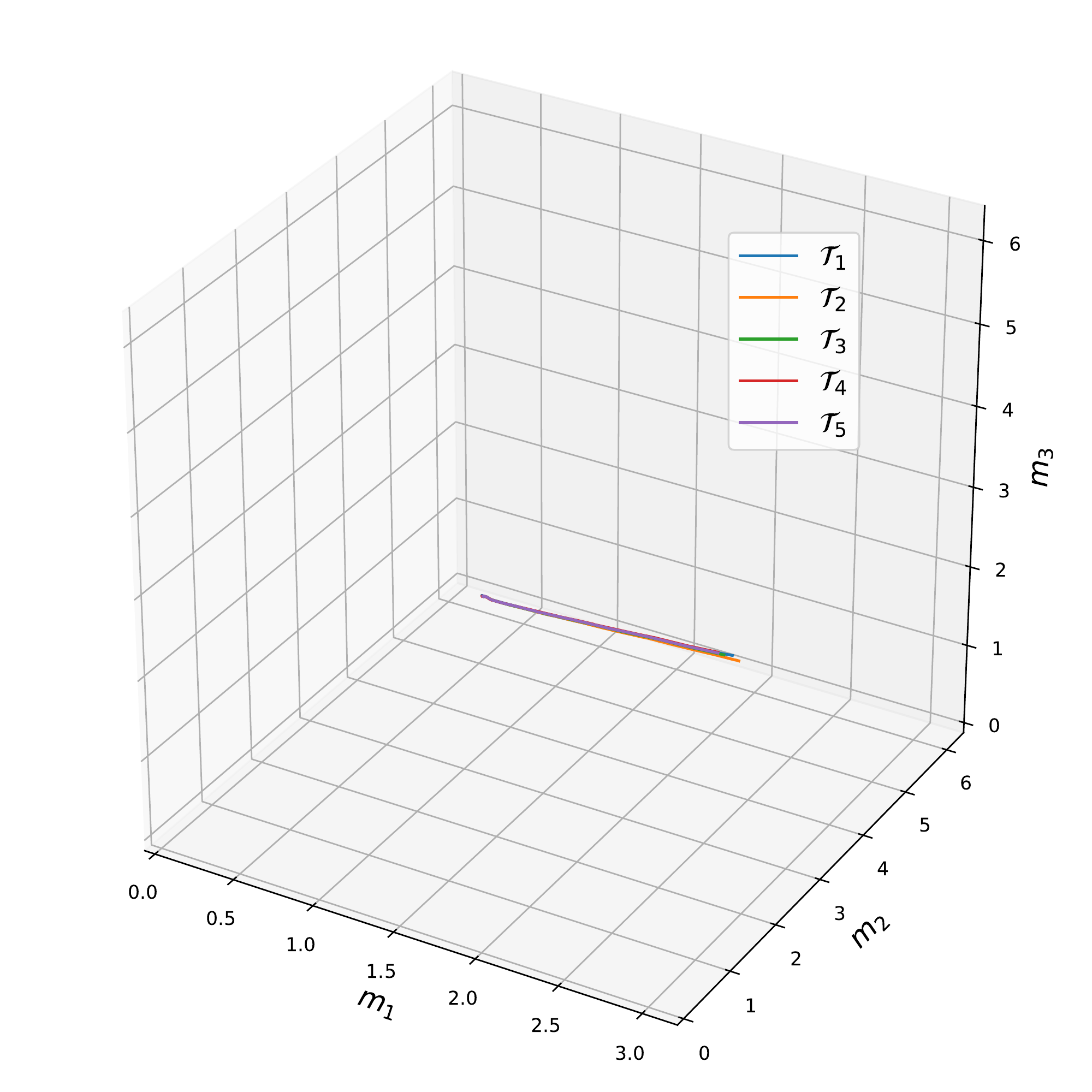}\caption{$t=400$}    
   \end{subfigure}
   \begin{subfigure}[t]{1in}
        \centering
        \includegraphics[width=1in]{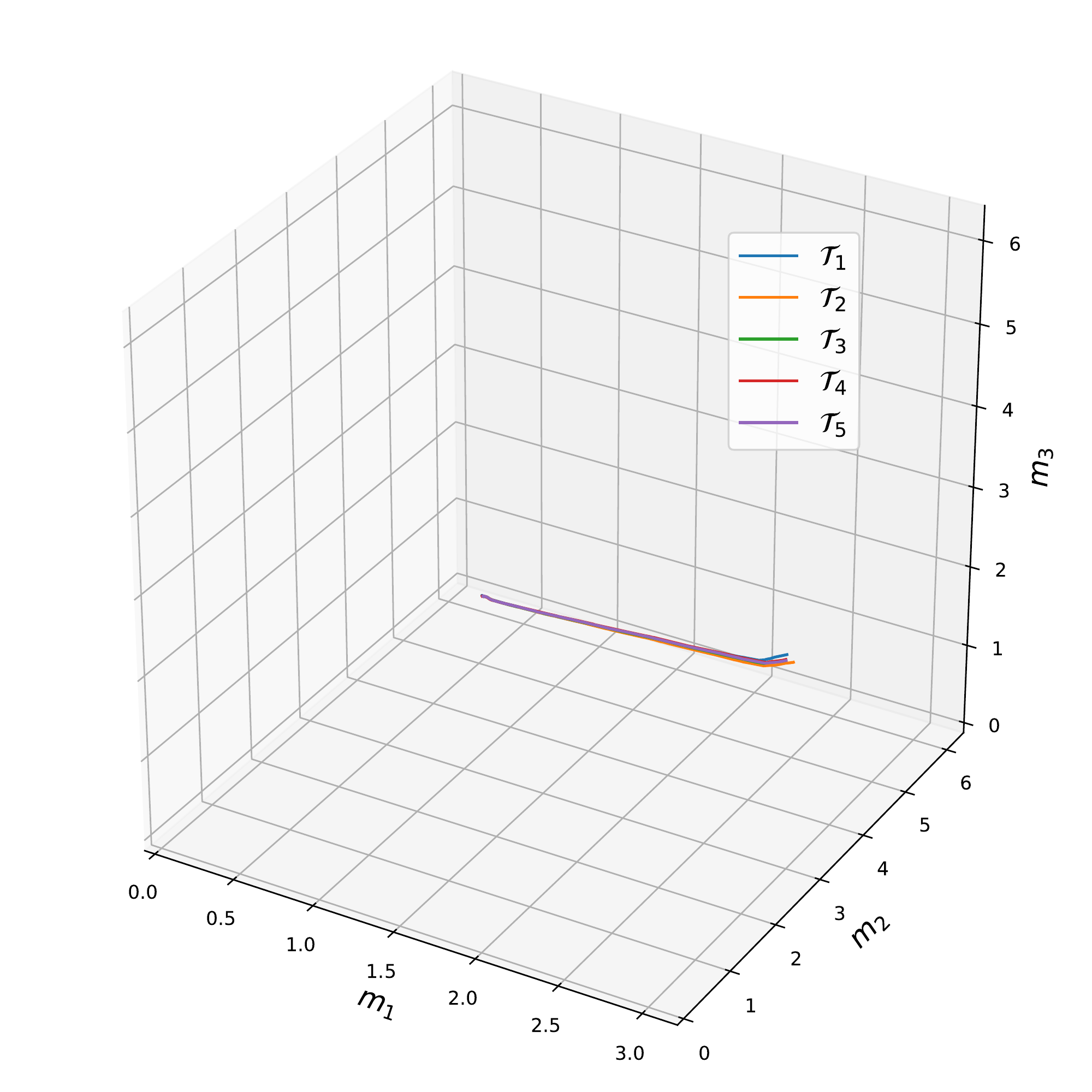}\caption{$t=600$}    
   \end{subfigure}
   \begin{subfigure}[t]{1in}
        \centering
        \includegraphics[width=1in]{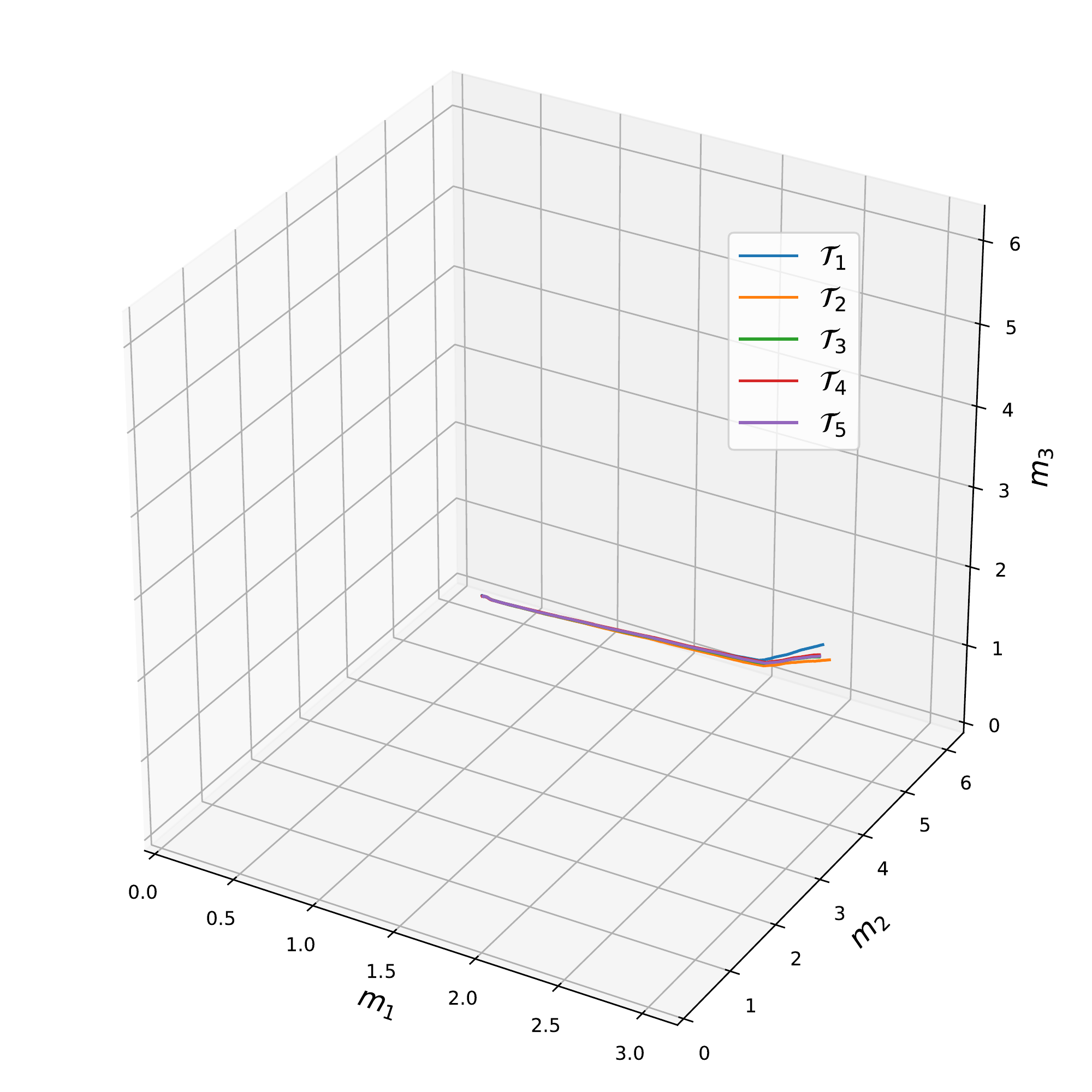}\caption{$t=800$}    
    \end{subfigure}
    \begin{subfigure}[t]{1in}
        \centering
        \includegraphics[width=1in]{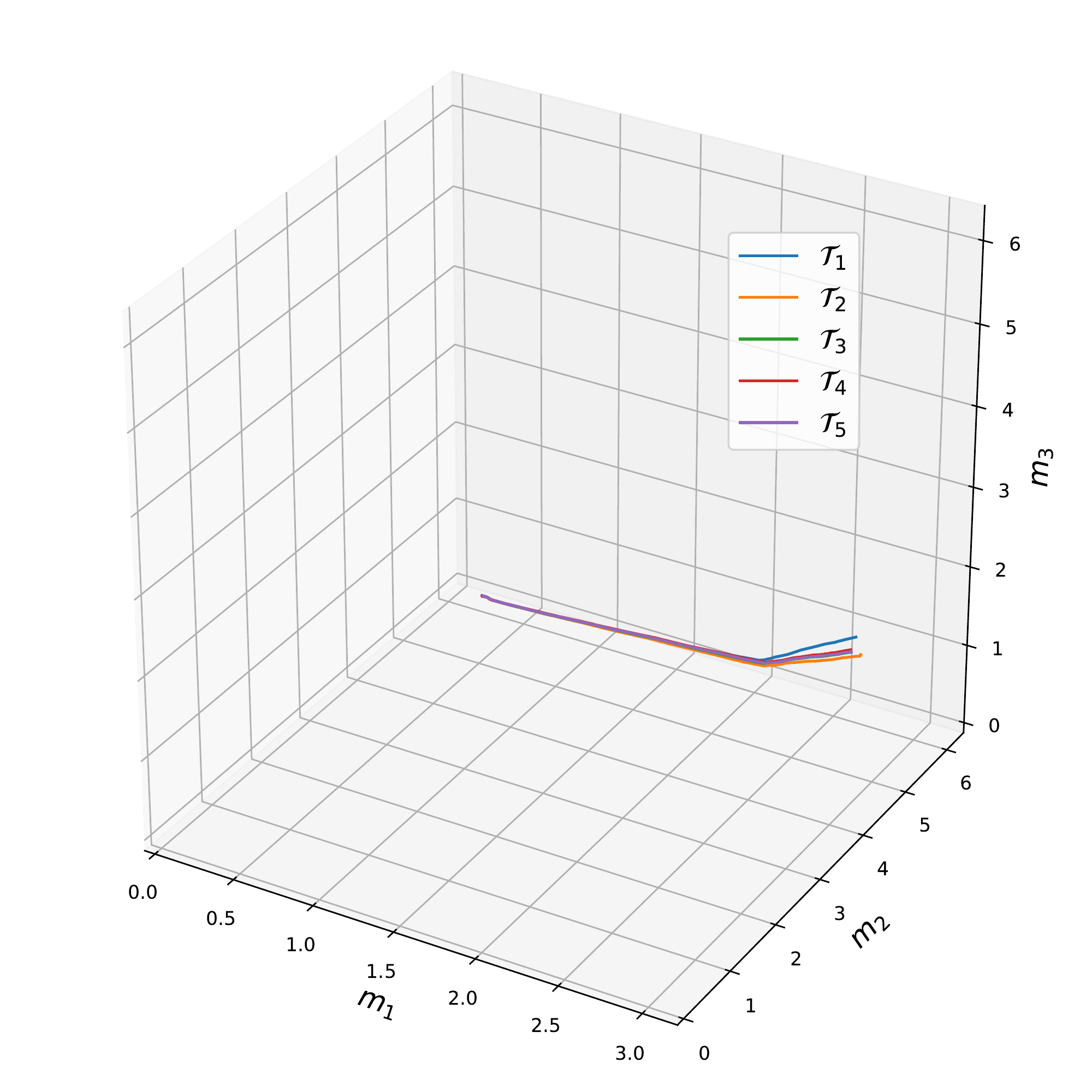}\caption{$t=1000$}    
    \end{subfigure}
    \caption{The same as~\Cref{fig:nonlinear_noisy_openloop_controller_setting1} but for a different nominal parameters of the policy.}
    \label{fig:linear_noisy_openloop_controller_setting2}
\end{figure}

\begin{figure}[h!]
    \centering
    \begin{subfigure}[b]{0.33\textwidth}
        \centering
        \includegraphics[width=0.8\textwidth]{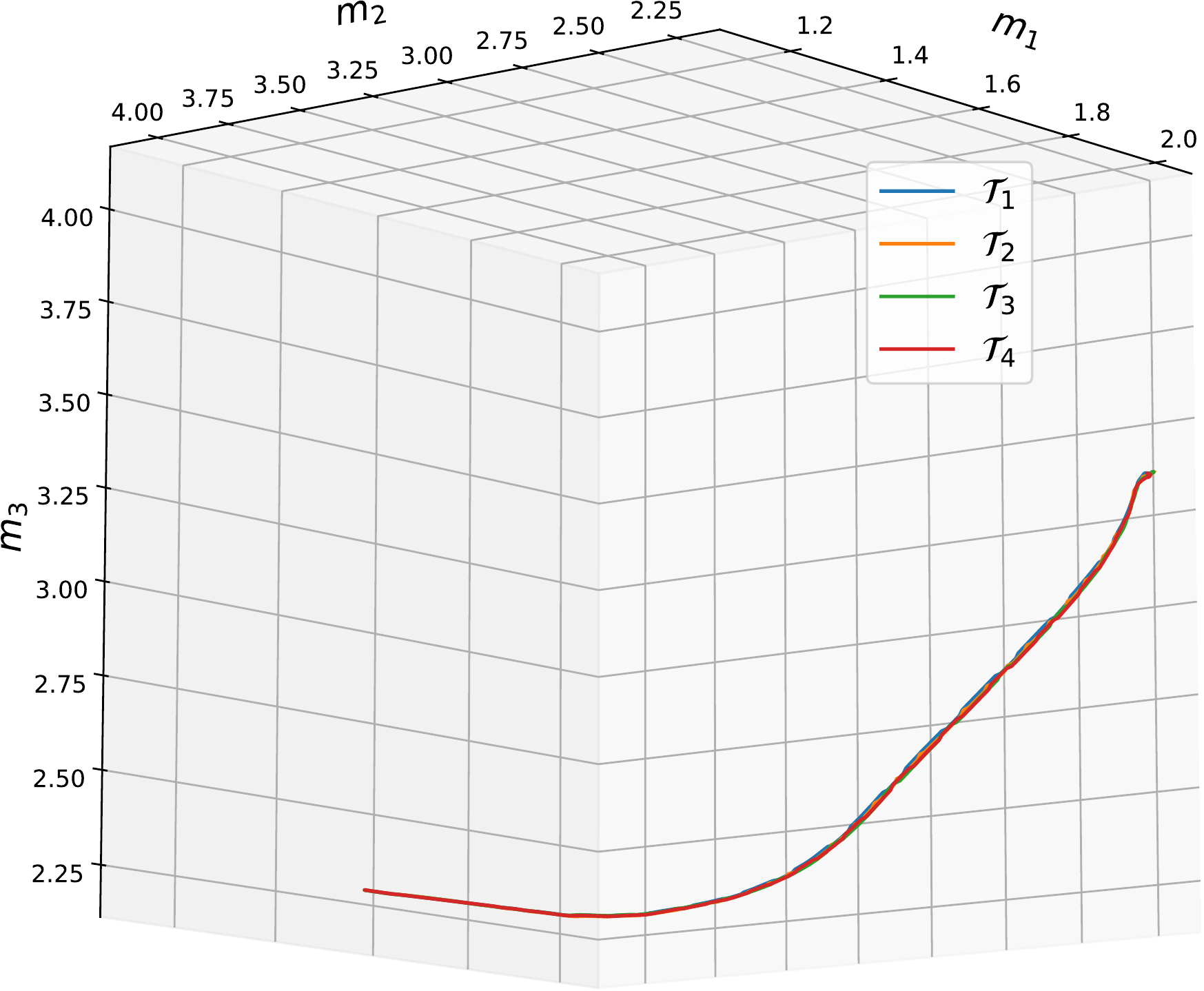}
        \caption{Low noise}        
    \end{subfigure}
    \begin{subfigure}[b]{0.3\textwidth}
        \centering
        \includegraphics[width=0.85\textwidth]{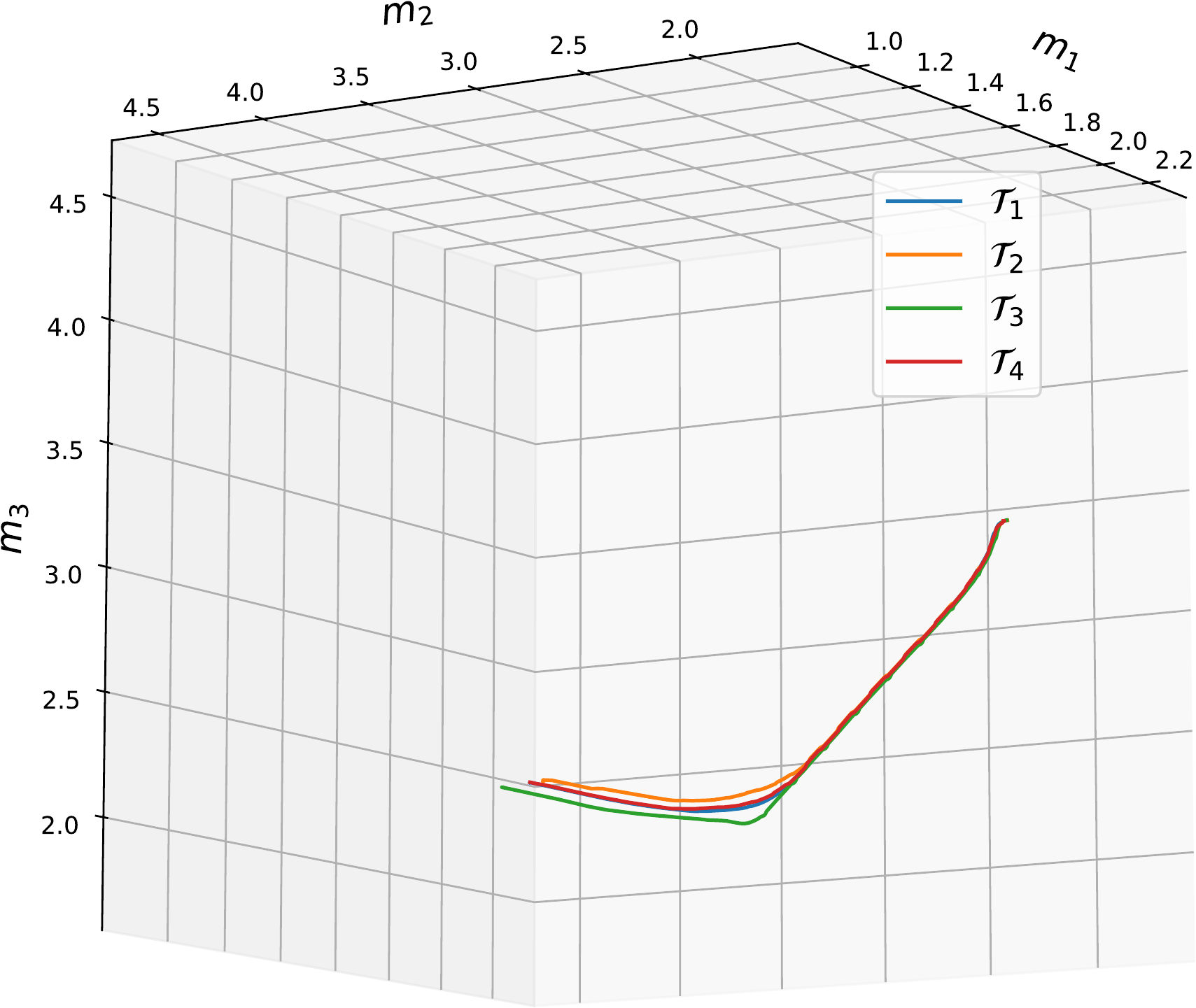}
        \caption{Medium noise}        
    \end{subfigure}
    \begin{subfigure}[b]{0.3\textwidth}
        \centering
        \includegraphics[width=0.9\textwidth]{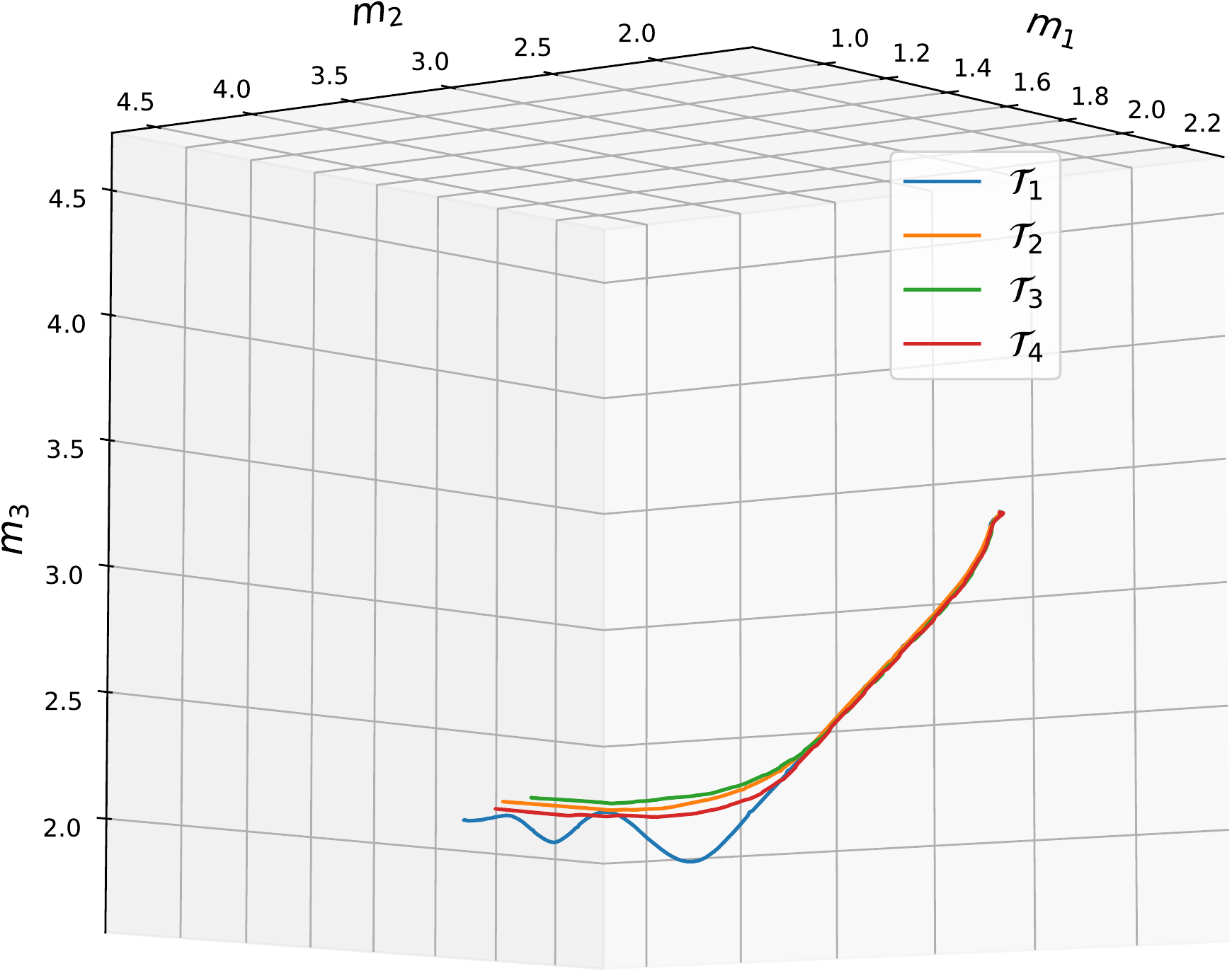}
        \caption{High noise}        
    \end{subfigure}
    \caption{PD controller with various noise intensities on $K_p$ parameter.}\label{fig:PD_perturbed_controller}
\end{figure}

\newpage

\section{Applications of physical derivatives}
\label{app:applications}

\paragraph{Robust control}
In control theory, robust control relates to the design of a controller whose performance is guaranteed for a range of systems and controllers belonging to a certain neighborhood around the nominal system~\citep{zhou1998essentials}. It is desired to have a controller that keeps the performance of the system at a certain good level even if the parameters of the controller are not fixed to the theoretical values. Assume the performance of the system is associated with some function of a trajectory $\Ecal(\Tcal)$. Changing the parameters of the controller $\thetab$ results in a change in the trajectories. This allows us to compute $\partial \Tcal / \partial \thetab$ that consequently gives us $\partial \Ecal(\Tcal) / \partial \thetab$ by the chain rule. Roughly speaking, between two sets of parameters $\thetab_1$ and $\thetab_2$, the set of parameters that gives the least $\partial \Ecal / \partial \thetab$ is preferred. This means that by perturbing the parameters of the controller and assessing the performance of the system, an estimate of the curvature of the landscape of $\Ecal(\Tcal(\thetab))$ is obtained. We prefer flatter regions of this space where a small change in $\thetab$ does not cause a drastic change in the performance metric $\Ecal$.

\paragraph{Safety}
Safety refers to the situations in which the agent may hurt itself or the environment and causes irreversible damages if it freely takes arbitrary actions~\citep{garcia2015comprehensive}. For a safety-critical system whose full physical models are hard to obtain, the physical gradients can assist in avoiding restricting the parameters of the robot to prevent unsafe behavior. The physical derivatives are learned in the Lab environment before the robot is deployed into the wild. For example, a rover whose mission is to safely explore an unknown environment often enjoys a learning loop that allows it to adapt to the new environment. Even though the learning in the new environment requires sufficient exploration, the physical derivatives can be used to give a rough simulation of the robot's next few states under a given update to its parameters. The potential harmful updates might be detected by such simulation and be avoided.

\paragraph{Adversarial system identification}
System identification concerns learning about the governing equations of the system from observed trajectories. Most systems require some input excitation to show their behavior. Normally the input is designed rich enough to elicit the important behaviors from the system~\citep{gevers2009identification}. Assuming the system would not be harmed by the input signal, this traditional approach has the downside: we never know the input signal is rich enough to excite every mode of the system. Besides, there exist scenarios when the identification must be carried out with minimal control effort. Hence, the design of the input signal must not be agnostic to the dynamics of the system. However, estimating the dynamics is the initial goal of system identification. This seeming chain problem can be solved by repetitively updating the control input and the estimated model. This results in an adversarial game between the controller and the current best estimate of the system which can be described as the following min-max problem

\begin{equation}
    \argmin_{\phib}\argmax_{\thetab, \xb_0}\Lcal= \frac{1}{T}\int_{t_0}^{t_0+T}
\lVert \int_{t_0}^t \hat{f}(\xb_m(t), \ub(\xb,t;\thetab);\phib) dt - \xb(t;\thetab)\rVert_2^2 \diff t.
\label{eq:adversarial_sysid}
\end{equation}
In this formula, $\xb_m$ denotes the states of the model parameterized by $\phib$ and the parameters of the estimated model while $\thetab$ shows the parameters of the controller. Notice that $\xb$ is the state of the physical system, which is a function of the controller parameters too. The procedure proceeds as follows: A controller with parameters $\thetab$ is applied to both models and the physical system. The produced trajectories by the model and the physical system are compared. The model parameters are updated in the direction that minimizes this distance to give a better estimate of the system. The controller parameters, on the other hand, are updated to maximize this distance. This maximization ensures that the controller drives the system towards the regions of the state space for which the current model is not yet accurate.

\section{Detailed literature review}
\label{sec:detailed_literature_review}
There has been a recent surge of interest in unsupervised methods in reinforcement learning when a task-specific reward function is not the only driving force to train the agent~\citep{baranes2013active,bellemare2016unifying, gregor2016variational,hausman2018learning, houthooft2016variational, badia2020never, sekar2020planning}. A truly intelligent agent must behave intelligently in a range of tasks, not only in a single task associated with its reward function. This requires the agent to develop some sort of \emph{general competence} that allows it to come up with solutions to new problems by combining some low-level primitive skills. This general competence is a key factor in animals to quickly and efficiently adapt to a new problem~\citep{weng2001autonomous}. By calling the traditional RL,~\emph{extrinsicially motivated RL}, the new framework is called~\emph{intrinsically motivated RL}. There have been many ideas in this line with various definitions for the terms~\emph{motivation} and~\emph{intrinsic}. Some researchers assume a developmental period in which the agent acquires some reusable modular skills that can be easily combined to tackle more sophisticated tasks~\citep{kaplan2003motivational, weng2001autonomous}. Curiosity and confidence are other unsupervised factors that can be used to drive the agent towards unexplored spaces to achieve new skills~\citep{schmidhuber1991possibility, kompella2017continual, bagaria2021skill}. Interestingly, there are observations in neuroscience that dopamine, a known substance that controls one's motivation for extrinsic rewards, is also associated with intrinsic properties of the agent, such as novelty and curiosity. A novel sensory stimulus activates the dopamine cells the same way they are activated by extrinsic reward. Children build a collection of skills accumulatively while they engage in activities without a specific goal, e.g., hitting a ball repeatedly without a long-term target such as scoring a goal. The achieved skills contribute to their stability while handling objects~\citep{touwen1992development}.

Another line of work concerns the fundamental constraints of the agent/environment and ensures those constraints are met while learning. For example, in many practical systems, learning episodes must halt if the system is likely to undergo an irreversible change. For example, the training episodes of a fragile robot must ensure the robot does not fall or will not be broken in any circumstance while acting under a certain policy. The general name~\emph{safe RL} embodies ideas to tackle such issues in current interactive learning algorithms~\citep{garcia2015comprehensive, srinivasan2020learning}. One major aspect of safety is~\emph{stability} that loosely means that states of the system converge to some invariant sets or remain within a certain bound~\citep{lyapunov1992general}. Control theory enjoys a physical model of the system to guarantee stability~\citep{khalil2002nonlinear}. When the physical model is not known in advance, the model is either learned along with the policy (model-based RL) or will be implicitly distilled in the value function (model-free RL)~\citep{sutton2018reinforcement}. Stability can be categorized as an intrinsic motivation for the agent. No matter what task the agent aims to solve, it must remain stable all the time. Learning the transition model, which is the major concern of model-based RL, can also be seen as intrinsic motivation. The agent learns to predict the future step given the current state. The advantage of learning a model\textemdash even inaccurately\textemdash is twofold: the agent would know where to go and where not to go. It knows which regions of the state space are unsafe to explore and must be avoided. It also knows which regions are unexplored and might be informative to improve the model. This brings us to another view to intrinsic reward that encourages~\emph{diversity}.

Our work is also relevant to sensitivity analysis and its use in training the parameters of dynamical models. After~\citeauthor{chen2018neural}'s NeuralODE on training neural networks by sensitivity analysis of the network parameters, the method was successfully applied to various tasks such as learning dynamics~\citep{rudy2019deep}, optimal control~\citep{han2018mean}, and generative models~\citep{grathwohl2018ffjord}. Our method can be seen as a mode-free sensitivity analysis in real-world systems. In NeuralODE, the gradient with respect to the parameters requires solving ODEs for both states and adjoint states that require a transition model. Since we work directly on the physical system, we don't need to calculate the integrals forward in time. The system itself acts as a physical ODE solver.

The importance of learning from unlabelled experiences is a known fact in animals. Many animals function efficiently soon after birth before being exposed to a massive labeled experience. Part of it might be due to unsupervised learning, but the major part of the story can be a genetic heritage after years of evolution that~\citeauthor{zador2019critique} called \emph{genomic bottleneck}. The same idea turned out to be valid in statistical learning, where an automatically discovered neural network architecture performs surprisingly well with a shared random weight~\citep{weightagnostic2019}. The embedded inductive bias in the neural network architectures could be analogous to the wiring of the brain of animal babies, which transfers from generation to generation by genes.

\section{Solution to inherent noise}

\subsection{Solution to temporal noise}
\label{sec:temporal_noise}
Fortunately, this type of noise is not state-dependent by definition. If we find out how much a trajectory is shifted in time with respect to another trajectory, we can simply shift the trajectory for those many time steps and compensate for the delay. Hence, the problem becomes detecting the lagged trajectories with respect to a reference trajectory and also estimating the amount of the required time shift to compensate for the delay. We can either use physical landmarks in the trajectories to align them or use the correlation between them as a measure of alignment. The latter gave better results. Hence, we postpone the description of the former to the~\Cref{sec:detecting_zero_crossing}.

\paragraph{Correlation-based delay estimation}

\sloppy
In this method, we use the correlation between zero-meaned trajectories $\Tcal^{(i)}$ and $\Tcal^{(j)}$ to check if one is the lagged version of the other one. The delay $\tau$ is found by
\begin{equation}\label{eq:synchronization}
  \tau^* = \argmax_\tau \sum_{t=0}^{T-\tau} \langle S_\tau \xb^{(i)}_t, \xb^{(j)}_t\rangle
\end{equation}
where $S_\tau$ is a shift-operator by $\tau\in \ZZ$ time steps. In practice, we take one trajectory of $\{\Tcal^{(1)}, \Tcal^{(2)}, \ldots, \Tcal^{(M)}\}$, e.g. $\Tcal^{(r)}$ as the reference and synchronize other trajectories with respect to it using~\Cref{eq:synchronization}. The trajectories must be initially normalized to avoid trivial solutions where every trajectory is pushed towards the larger parts of the reference trajectory. For illustrative purposes, the plots of~\Cref{fig:time_shift_effect} show a sample of the lagged trajectory from the finger platform and its correction by the above method.

\begin{figure}[h!]
    \centering
    \begin{subfigure}[t]{0.4\textwidth}
        \centering
        \includegraphics[width=\textwidth]{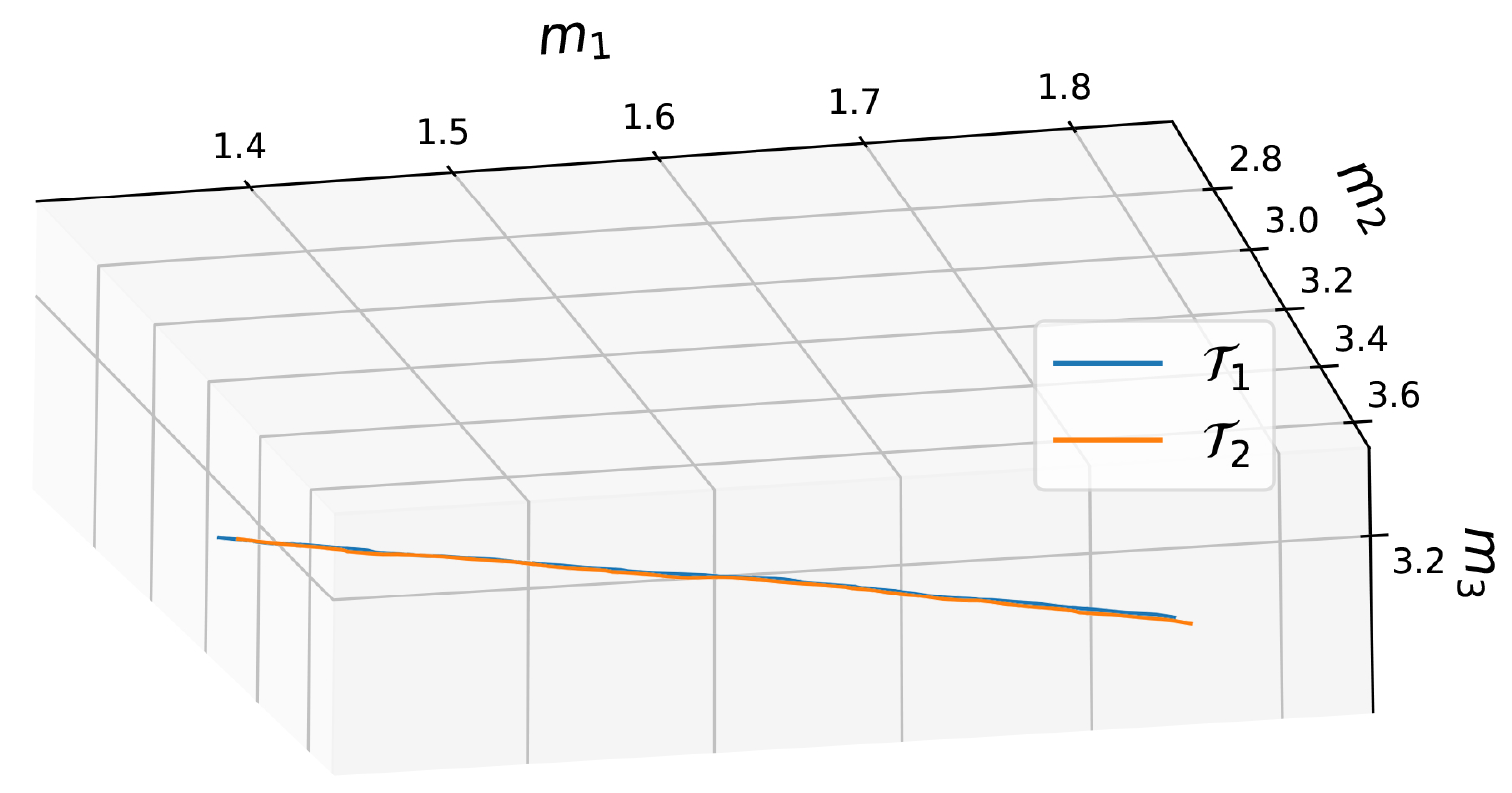}
        \caption{Trajectories affected by temporal noise}        
    \end{subfigure}
   \begin{subfigure}[t]{0.4\textwidth}
        \centering
        \includegraphics[width=\textwidth]{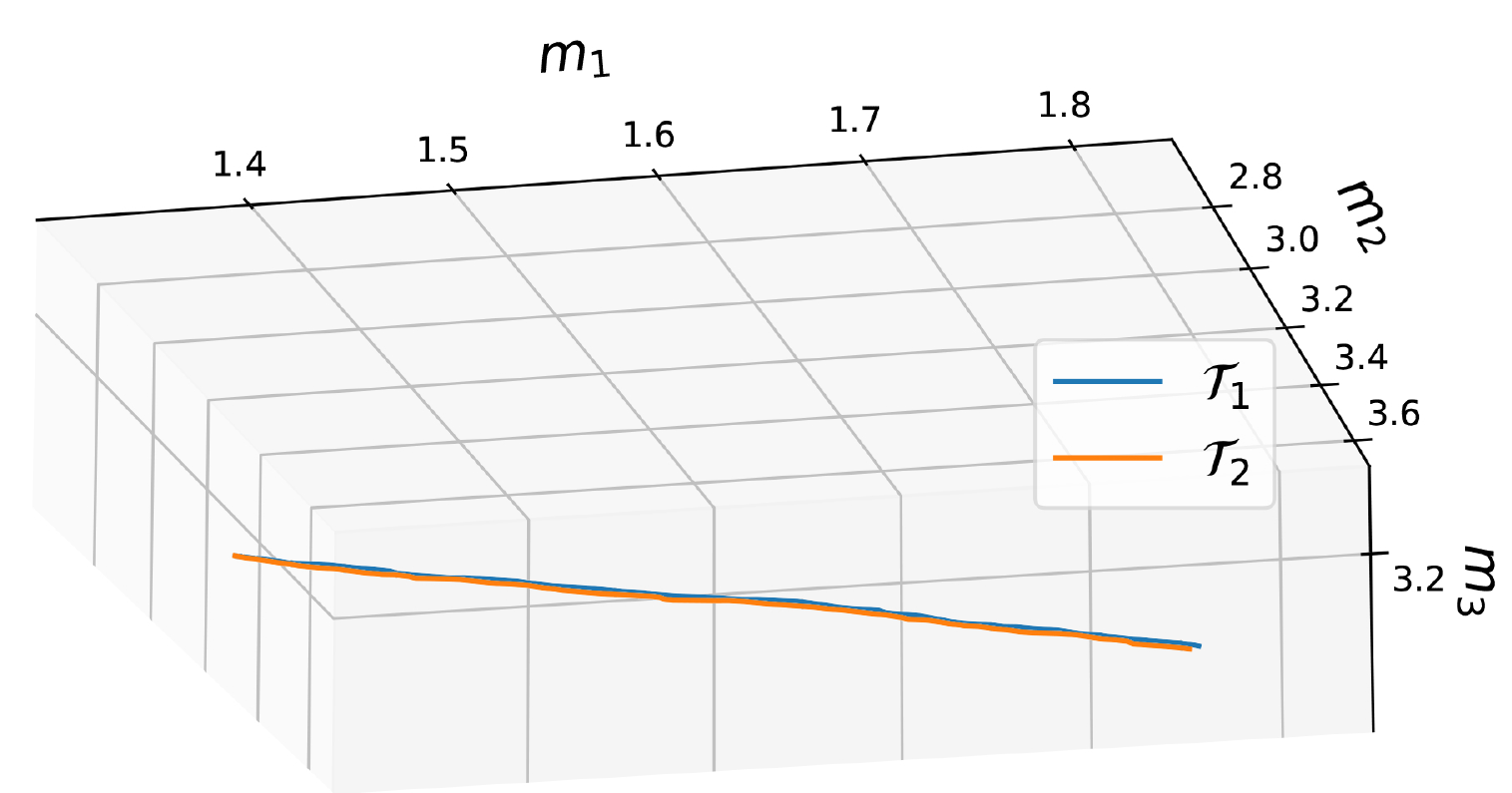}\caption{Temporally aligned trajectories}    
    \end{subfigure}
    \caption{The effect of temporal noise in delaying one trajectory versus the other one and its correction. The trajectories are produced by the linear open-loop controller similar to those used in~\Cref{sec:linear_openloop_controller} (Zooming is recommended).}\label{fig:time_shift_effect}
\end{figure}

\subsection{Solution to spatial noise}
\label{sec:spatial_noise}
The spatial noise can be a stochastic function of the actuator, environmental change, and electronic drivers. In a perfect model of the transition dynamics $\xb_{t+1}=f(\xb_t, \ub_t)$, applying the same control sequence $\{\ub_0, \ub_1, \ldots, \ub_{T-1}\}$ always results in the same sequence of states $\{\xb_1, \xb_2, \ldots, \xb_T\}$ when it starts from the same initial state $\xb_0$. This assumption is often violated in physical systems as different runs of the same policy may result in different trajectories, as can be seen in~\Cref{fig:same_controller_multiple_runs_setting1} in the Appendix. The noise in the dynamics can be any function of states, input, and time. Therefore, it is difficult to model this noise since it requires a prohibitively large number of random experiments. The good news is that if the physical system is built properly, the effect of this noise is expectedly low. Based on our observations from the finger platform, we can assume the following.

\begin{figure}[t!]
    \centering
    \begin{subfigure}[t]{0.24\textwidth}
        \centering
        \captionsetup{justification=centering}
        \includegraphics[width=\textwidth]{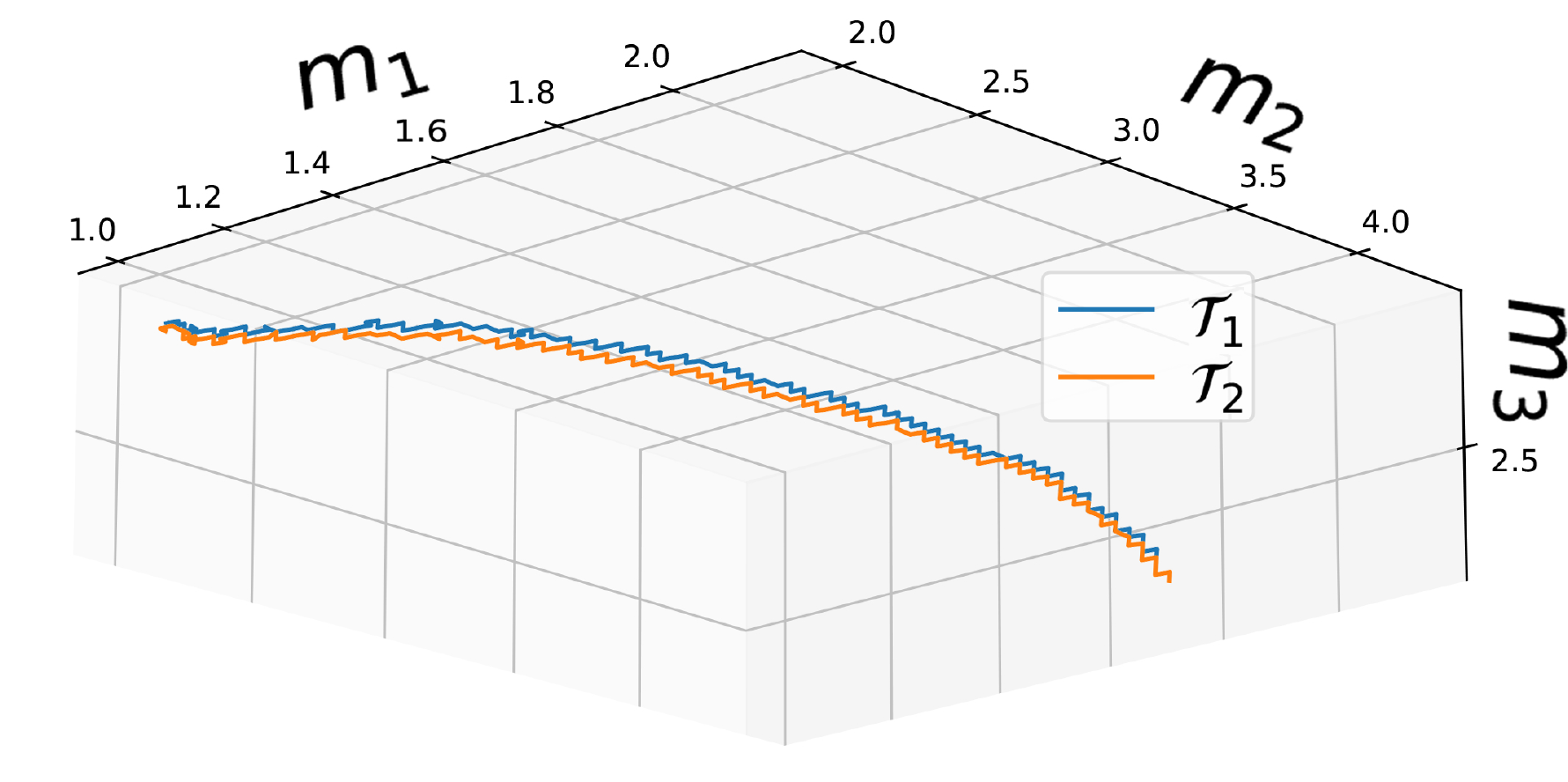}
        \caption{Small voxels \\ ($\gamma=0.01$)}        
    \end{subfigure}
   \begin{subfigure}[t]{0.24\textwidth}
        \centering
        \captionsetup{justification=centering}
        \includegraphics[width=\textwidth]{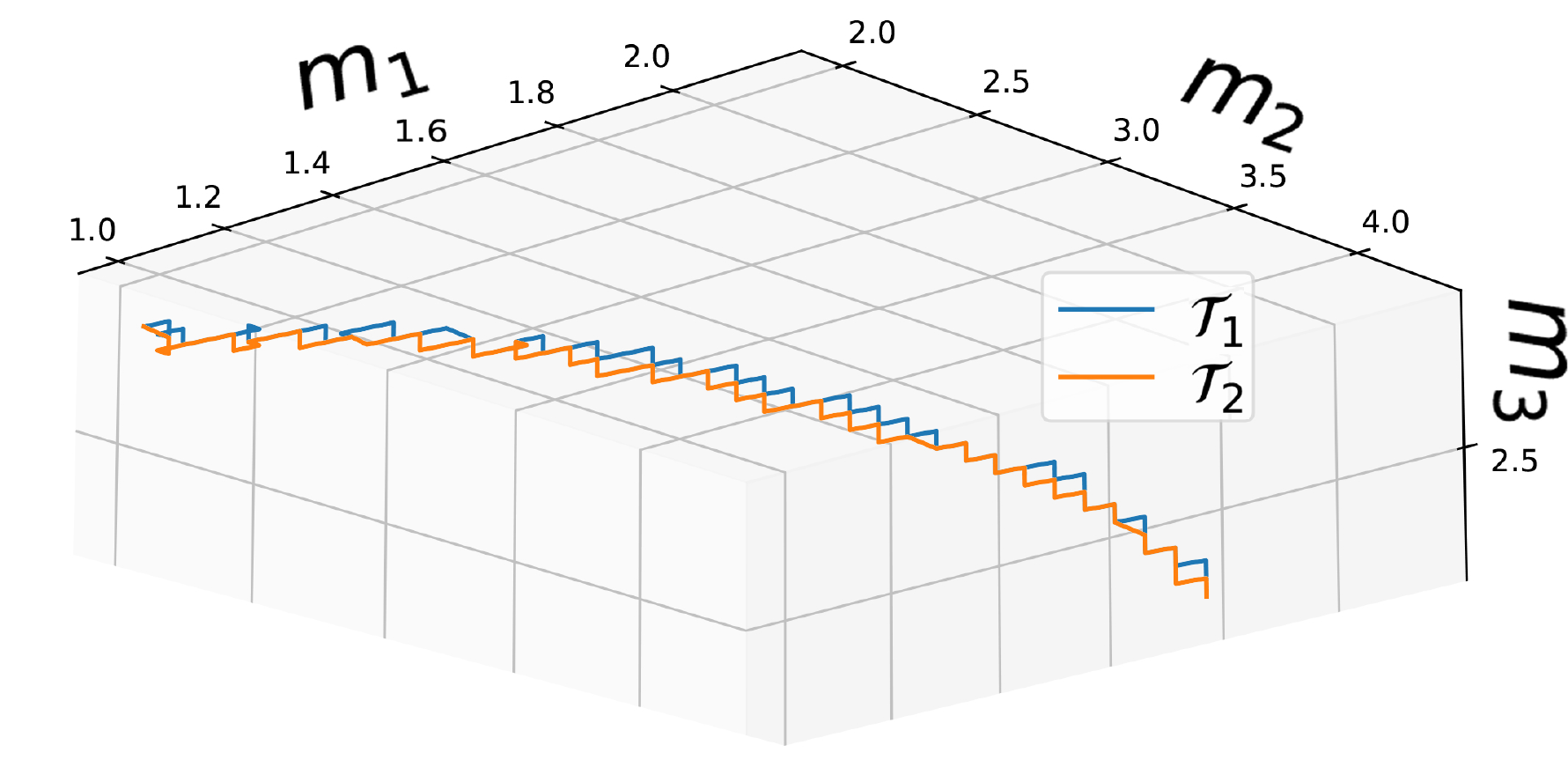}\caption{Medium voxels \\ ($\gamma=0.04$)}    
    \end{subfigure}
    \begin{subfigure}[t]{0.24\textwidth}
        \centering
        \captionsetup{justification=centering}
        \includegraphics[width=\textwidth]{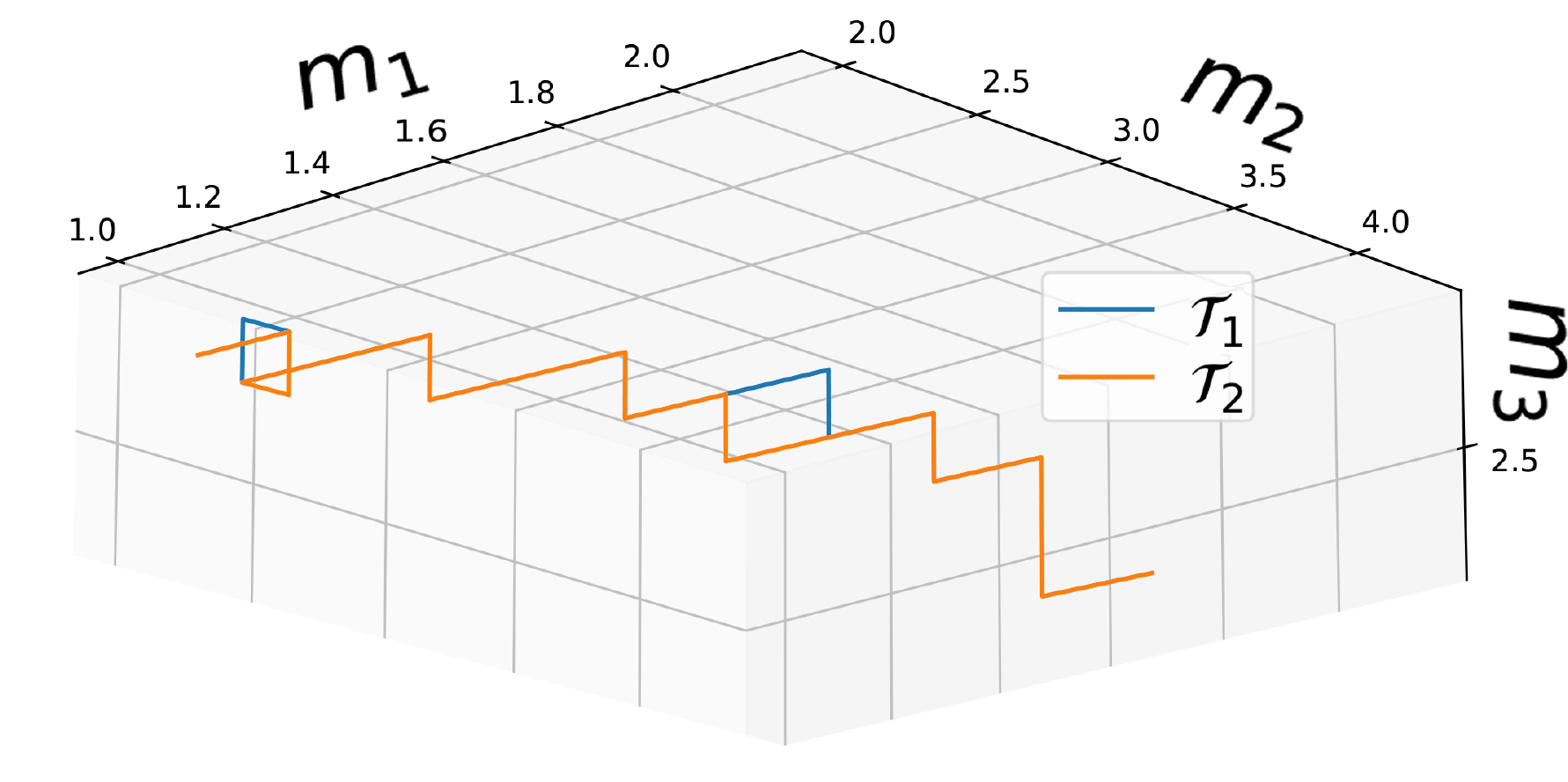}\caption{Large voxels \\($\gamma=0.16$)}    
    \end{subfigure}
    \begin{subfigure}[t]{0.24\textwidth}
        \centering
        \captionsetup{justification=centering}
        \includegraphics[width=\textwidth]{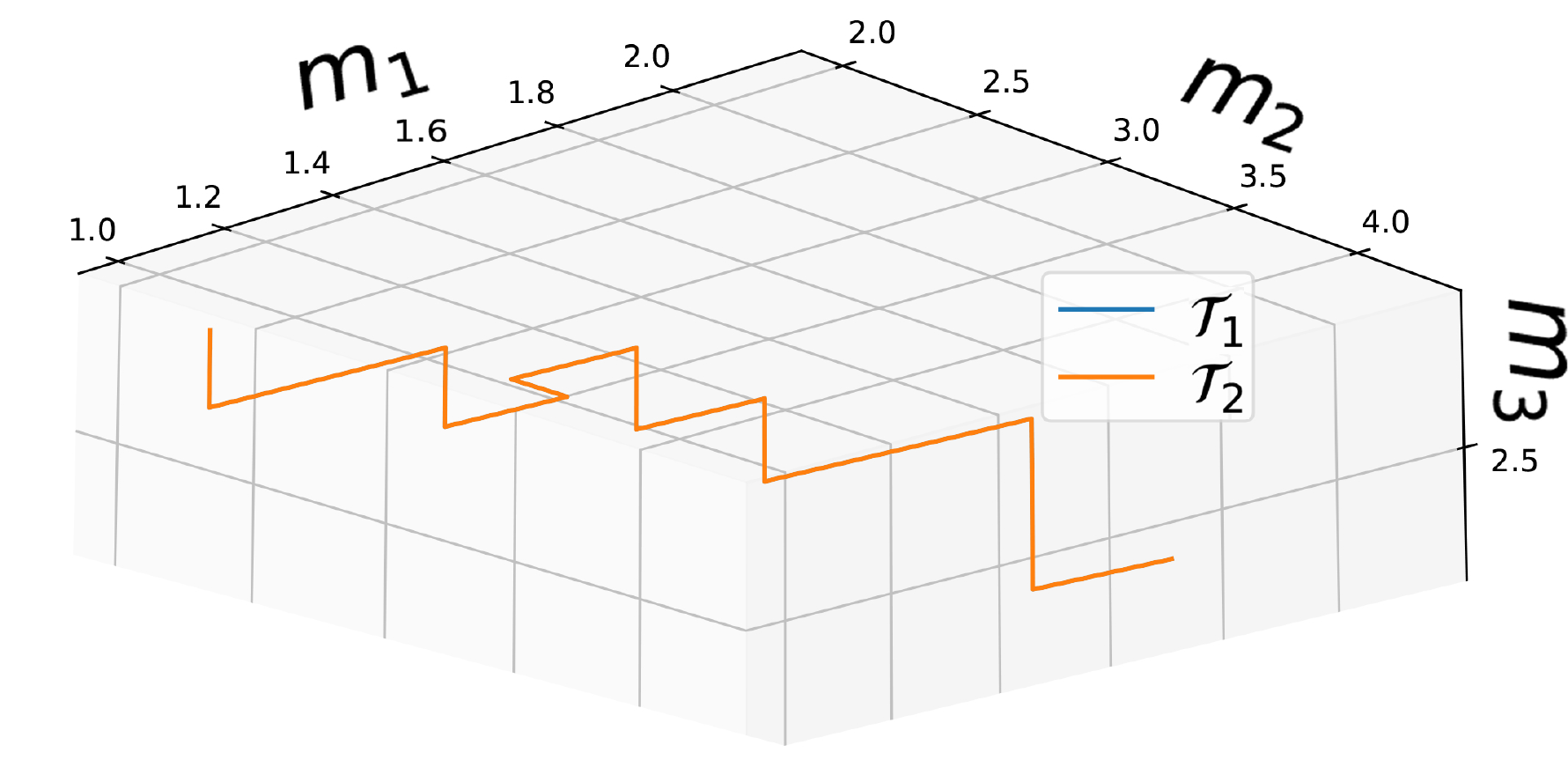}\caption{Large voxels \\ ($\gamma=0.2$)}    
    \end{subfigure}
    \caption{The effect of voxels on supressing spatial noise of the physical system. The trajectories are produced by linear open-loop controllers as those in~\Cref{sec:linear_openloop_controller} for the purpose of illustrating the effect of voxelization.}\label{fig:voxel_effect}
\end{figure}

\begin{assumption}\label{assum:limited_physical_noise}
    Limit on the physical noise: Let's the control sequence $\Ub=\{\ub_0, \ub_1, \ldots, \ub_{T-1}\}$ be applied to the system $M$ times resulting in multiple sequence of states $\Tcal^{(1)}, \Tcal^{(2)}, \ldots, \Tcal^{(M)}$. There exists a relatively small $\zeta$ such that
  \begin{equation}
    \lVert \Tcal^{(i)} - \Tcal^{(j)}\rVert_\infty \leq \zeta\;\; \mathrm{for\; every}\;\; i, j \in \{1, 2, \ldots, m\}.
  \end{equation}
\end{assumption}
The word \emph{relatively} here means that the change of the trajectory due to the inherent physical noise of the system must be small compared to the change of the trajectories when the parameters of the policy are perturbed.

To reduce the sensitivity of the estimated gradient to this unwanted spatial noise, we divide the state space of the physical system into regularly located adjacent cells called \emph{voxels}. Each voxel $vox(\cbb)$ is represented by its center $\cbb$ and is defined as
\begin{equation}
  vox(\cbb) = \{\xb\in\Xcal\; |\; \lVert \xb -\cbb \rVert_\infty\leq \gamma\}
\end{equation}
where $\gamma$ is the parameter of the voxelization.
The concept of the voxel is roughly used as a \emph{superstate}. Every state that ends up within $vox(\cbb)$ gives rise to the same superstate. After recording the trajectories from the robot, every state is mapped to the center of the voxel it belongs to as
\begin{equation}
  \cbb \leftarrow \xb\; \mathrm{for}\; \xb\in vox(\cbb)
\end{equation}

For simplicity, we denote the center $\cbb$ of the voxel that $\xb$ belongs to with $\cbb_{\gamma}(\xb)$. After voxelization, we work with $\cbb_{\gamma}(\xb)$ instead of $\xb$. For example, all the gradients of~\eqref{eq:gradients_in_all_directions} are computed as $\nabla_\thetab \cbb$ rather than $\nabla_\thetab \xb$. To illustrate the positive effect of voxelization of the state space, it can be seen in~\Cref{fig:voxel_effect} that increasing the voxel size improves the overlapping between two trajectories that deviate from each other due to the inherent spatial noise of the system not because of perturbing the parameters of the policy, but because of the inherent imperfection of the mechanical and electrical components of the system. This benefit comes with a cost which is the error introduced by voxelization. Fortunately, this error is bounded due to the following lemma.

\begin{lemma}
\label{lem:voxelization_boudnedn_error}
  The additional error caused by voxelization $\cbb_{\gamma}(.)$ is bounded proportional to the size of each voxel $\gamma$.
\end{lemma}
\begin{proof}
    Assume that we have two noisy (spatial noise) trajectories $\yb^{(1)}_t=\xb^{(1)}_t + \epsilon^{(1)}_t, \yb^{(2)}_t = \xb^{(2)}_t + \epsilon^{(2)}_t, \textrm{ for } t \in \{1, 2, \dots, T\}$ where $\xb^{(1)}_t, \xb^{(2)}_t$ are primary noiseless trajectories having significant difference that we want to calculate derivatives for and spatial errors $\epsilon^{(1)}_t, \epsilon^{(2)}_t$ comes from a distribution. The error induced in the standard case is $\mathbb{E}_{\epsilon^{(1)}_t, \epsilon^{(2)}_t} [|| (\yb^{(2)}_t -  \yb^{(1)}_t) - (\xb^{(2)}_t - \xb^{(1)}_t) ||_\infty] = \mathbb{E}_{\epsilon^{(1)}_t, \epsilon^{(2)}_t} [|| (\yb^{(2)}_t - \xb^{(2)}_t) - (\yb^{(1)}_t - \xb^{(1)}_t) ||_\infty] = \mathbb{E}_{\epsilon^{(1)}_t, \epsilon^{(2)}_t} [|| \epsilon^{(2)}_t - \epsilon^{(1)}_t ||_\infty]$, while the error in the voxelized world is $\mathbb{E}_{\epsilon^{(1)}_t, \epsilon^{(2)}_t} [|| (\cbb_{\gamma}(\yb^{(2)}_t) -  \cbb_{\gamma}(\yb^{(1)}_t)) - (\xb^{(2)}_t - \xb^{(1)}_t) ||_\infty] = \mathbb{E}_{\epsilon^{(1)}_t, \epsilon^{(2)}_t} [|| (\cbb_{\gamma}(\yb^{(2)}_t) - \yb^{(2)}_t) - (\cbb_{\gamma}(\yb^{(1)}_t) - \yb^{(1)}_t) + (\yb^{(2)}_t -  \yb^{(1)}_t) - (\xb^{(2)}_t - \xb^{(1)}_t) ||_\infty]$. Using triangle inequality and knowing that $||\cbb_{\gamma}(\yb^{(i)}_t) - \yb^{(i)}_t||_\infty \leq \gamma$, we have
    \begin{equation*}
        \mathbb{E}_{\epsilon^{(1)}_t, \epsilon^{(2)}_t} [|| (\cbb_{\gamma}(\yb^{(2)}_t) -  \cbb_{\gamma}(\yb^{(1)}_t)) - (\xb^{(2)}_t - \xb^{(1)}_t) ||_\infty] \leq 2 \gamma + \mathbb{E}_{\epsilon^{(1)}_t, \epsilon^{(2)}_t} [|| \epsilon^{(2)}_t - \epsilon^{(1)}_t ||_\infty]
    \end{equation*}
    Hence, the lemma is proved.
\end{proof}

The voxels become boxes in 3D as in~\Cref{fig:voxel_boxes}. The gradient is estimated as the distance between two points in 3D coordinates. Hence the source of voxelization error is approximating the distance between two points in 3D with the distance between the centers of the corresponding boxes to which those points belong. This error is written next to the boxes in~\Cref{fig:voxel_boxes}. The maximum additional error due to voxelization is proportional to the voxel sizes $\gamma$ and its proportion is upper-bounded by $2$.

\begin{figure}[!ht]
    \centering
    \includegraphics[width=\linewidth/2]{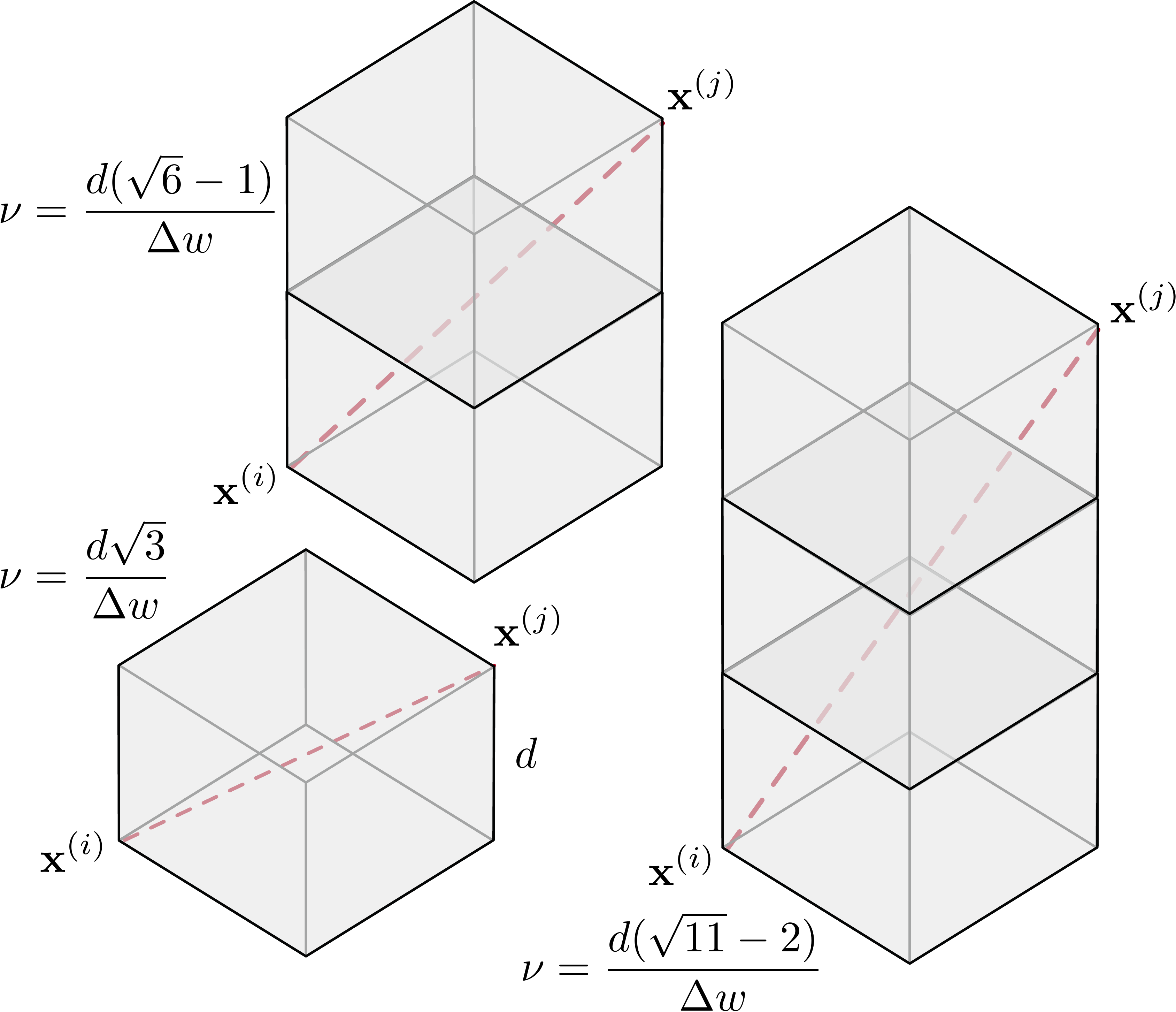}
    \caption{The maximum potential error in the estimated gradients when the space is voxelized. As can be seen, the error vanishes when the corresponding voxels to $\xb^{(i)}$ and $\xb^{(j)}$ are far from each other.}
    \label{fig:voxel_boxes}
\end{figure}

\section{Extended set of solutions to the real world challenges}
\subsection{Detecting zero crossing}
\label{sec:detecting_zero_crossing}
In this method, we take advantage of special landmarks in the trajectories. The landmarks are typically caused by the physical constraints of the system. For example, when a robot's leg touches the ground, the velocity of the leg becomes zero. Likewise, when a joint reaches its physical limit, the velocity of the connected arm to the joint becomes zero or changes signs. In both cases, a zero-crossing occurs that can be used as a landmark to synchronize lagged trajectories with a reference trajectory. Even though this method will eliminate the temporal noise, it requires the presence of such landmarks along the trajectories. Notice that from a mathematical point of view, there is nothing special about \emph{zero}. We can pick any value of states along a reference trajectory and synchronize all other trajectories with respect to it. However, in practice, physical landmarks are easier to detect and have less ambiguity that consequently giving a more accurate synchronization.

\section{Perturbation methods}
\label{sec:shaking_detailed}
{\it Gaussian Perturbation---} Likely values of $\thetab$ create nominal policies encoded by $\{\thetab^{(1)}, \thetab^{(2)}, \ldots, \thetab^{(m)}\}$. We put Gaussian distributions centered at each of the nominal values resulting in a mixture of Gaussians. To reduce the hyper-parameters, we assume the variances of the Gaussians are themselves sampled from an exponential distribution making sure they all take positive values (See~\Cref{fig:shaking} left). Here, we manually choose a reasonable value for the rate parameter of the exponential distribution. Making inference on the hyper-parameters of the sampling distributions can be a topic for future research, especially in active learning for a more clever less costly sampling strategy. \\
{\it Uniform Perturbation---} In this setting, the state space of the changeable parameters of the policy is discretized and a uniform distribution is assumed around each value of this grid with some overlapping with the neighboring cells (See~\Cref{fig:shaking} right).

\section{Experimental details}
Starting position in all the experiments is $(\frac{\pi}{2}, \frac{\pi}{2}, \pi ) $. Task's overall details are as following:
\begin{table}[htb!]
	\centering
	\begin{tabular}{lccc}
	\toprule
	    Task & number of trajectories & timesteps\\
	    \hline
        Linear (N) & 640 & 1500 \\
		PD controller(N) &   640  &  1500    \\
		PD controller(U) &  1000  &   1500   \\
		Sine 1 joint(N) & 640 & 5000 \\
		Sine 1 joint(U) & 1000 & 5000 \\
		Sine 2 joints(U) &  640   &   5000    \\
		Sine 2 joints(N) &  1000   &   5000    \\
	\bottomrule
	\end{tabular}
\vspace{0.3cm}
\label{tab:overall}
\end{table}
\\
In normal sampling cases, we ran 10 simulations for each set of $\lambda$ parameters which indicates noise level.
\subsection{Linear}
\begin{equation}
	u_{it} = w_i t + b_i\;\;\;\mathrm{for}\;\;\; i=1,2,3
\end{equation}
\subsubsection{Gaussian Sampling}

\begin{center}
    $ w_i = W_i + \epsilon_{w, i}\;\;\;\mathrm{for}\;\;\;
	i=1,2,3 $
	$$ \epsilon_{w,i} \sim N(0, e_{w} \times \lVert W_i \rVert_2) $$
	$$ e_{w} \sim exp(\lambda_{w}) \;\;\;\mathrm{for}\;\;\; \lambda_{w}=1,5,10,50,100,500,1000, 5000 $$
\end{center}
\vspace{0.1pt}
\begin{center}
    $ b_i = B_i + \epsilon_{b, i}\;\;\;\mathrm{for}\;\;\;
	i=1,2,3 $
	$$ \epsilon_{b,i} \sim N(0, e_{b} \times \lVert B_i \rVert_2) $$
	$$ e_{b} \sim exp(\lambda_{b}) \;\;\;\mathrm{for}\;\;\; \lambda_{b}=1,5,10,50,100,500,1000, 5000 $$
\end{center}
\vspace{0.2pt}
\begin{center}
    $ W = [0.00001, 0.0001, -0.00001], B = [-0.28, -0.15, -0.08]$
\end{center}

\subsection{PD Controller}
Final destination is $(\frac{\pi}{10}, 3\frac{\pi}{4}, 7\frac{\pi}{12})$
\subsubsection{Gaussian Sampling}
\begin{center}
    $ kp = KP + \epsilon $
	$$ \epsilon_{kp} \sim N(0, e_{kp} \times \lVert KP \rVert) $$
	$$ e_{kp} \sim exp(\lambda_{kp}) \;\;\;\mathrm{for}\;\;\; \lambda_{kp}=1,5,10,50,100,500,1000, 5000 $$
\end{center}
\begin{center}
    $ kd = KD + \epsilon $
	$$ \epsilon_{kd} \sim N(0, e_{kd} \times \lVert KD \rVert) $$
	$$ e_{kd} \sim exp(\lambda_{kd}) \;\;\;\mathrm{for}\;\;\; \lambda_{kd}=1,5,10,50,100,500,1000, 5000 $$
\end{center}

\subsubsection{Uniform Sampling}
\begin{center}
    $ kp \sim U(-0.5, 1.5), KP = 1$
\end{center}
\begin{center}
    $ kd = KD = 0.01$
\end{center}
\subsection{Sine 1 joint}
\subsubsection{Gaussian Sampling}
\begin{center}
    $ w = W + \epsilon $
	$$ \epsilon_{w} \sim N(0, e_{w} \times \lVert W \rVert) $$
	$$ e_{w} \sim exp(\lambda_{w}) \;\;\;\mathrm{for}\;\;\; \lambda_{w}=1,5,10,50,100,500,1000, 5000 $$
\end{center}
\begin{center}
    $ a = A + \epsilon $
	$$ \epsilon_{a} \sim N(0, e_{a} \times \lVert A \rVert) $$
	$$ e_{a} \sim exp(\lambda_{a}) \;\;\;\mathrm{for}\;\;\; \lambda_{a}=1,5,10,50,100,500,1000, 5000 $$
\end{center}
\begin{center}
    $ W = 0.01, B = 0.5$
\end{center}

\subsubsection{Uniform Sampling}
\begin{center}
    $ w \sim U(0.005, 0.015), a = A = 0.5$
\end{center}

\subsection{Sine 2 joints}
\subsubsection{Gaussian Sampling}
\begin{center}
    $ w_i = W_i + \epsilon \;\;\;\mathrm{for}\;\;\; i=1,2$
	$$ \epsilon_{w,i} \sim N(0, e_{w} \times \lVert W \rVert_2) $$
	$$ e_{w} \sim exp(\lambda_{w}) \;\;\;\mathrm{for}\;\;\; \lambda_{w}=1,5,10,50,100,500,1000, 5000 $$
\end{center}
\begin{center}
    $ a_i = A_i + \epsilon \;\;\;\mathrm{for}\;\;\; i=1,2$
	$$ \epsilon_{a, i} \sim N(0, e_{a} \times \lVert A \rVert_2) $$
	$$ e_{a} \sim exp(\lambda_{a}) \;\;\;\mathrm{for}\;\;\; \lambda_{a}=1,5,10,50,100,500,1000, 5000 $$
\end{center}
\begin{center}
    $ W = [0.01, 0.01], A = [-0.4, 0.5]$
\end{center}

\subsubsection{Uniform Sampling}
\begin{center}
    $ w_i \sim U(0.005, 0.015)\;\;\;\mathrm{for}\;\;\; i=1,2, a = A = 0.5$
\end{center}

\subsection{GP Score:}
\label{sec:gp_score}

Definition of the GP score:
The score is defined as $(1 - u/v)$, where u is the residual sum of squares $\Sigma (y_\mathrm{true} - y_\mathrm{pred})^2$ and $v$ is the total sum of squares $\Sigma (y_\mathrm{true} - \mathrm{mean}(y_\mathrm{true}))^2$. The best possible score is 1.0.

\subsection{Zero-shot planning task:}
\label{sec:zero-shot-planning_exp_details}

For the task of~\Cref{sec:zeroshot_planning}:
Number of training trajectories: $100$ each with $1500$ time steps

Kd = $0.01$

Kp = Uniformly sampled from $[0.2, 0.6]$

Initial point: $X_\circ = [\pi/2, \pi/2, \pi])$

desired position = $[\pi / 10, 3 * \pi / 4, 7 * \pi / 12]$

\newpage

\section{More results}
\label{sec:more_results}
In this section, the results of the extra experiments that were eliminated from the main text due to the space limit are presented. 

The following figures show GP models trained by a set of directional derivatives collected during the perturbation phase. The results are provided for the experiments of~\Cref{sec:nonlinear_openloop_controller,sec:feedback_controller}.

\begin{figure}[h!]
    \centering
    \includegraphics[width=0.81\textwidth]{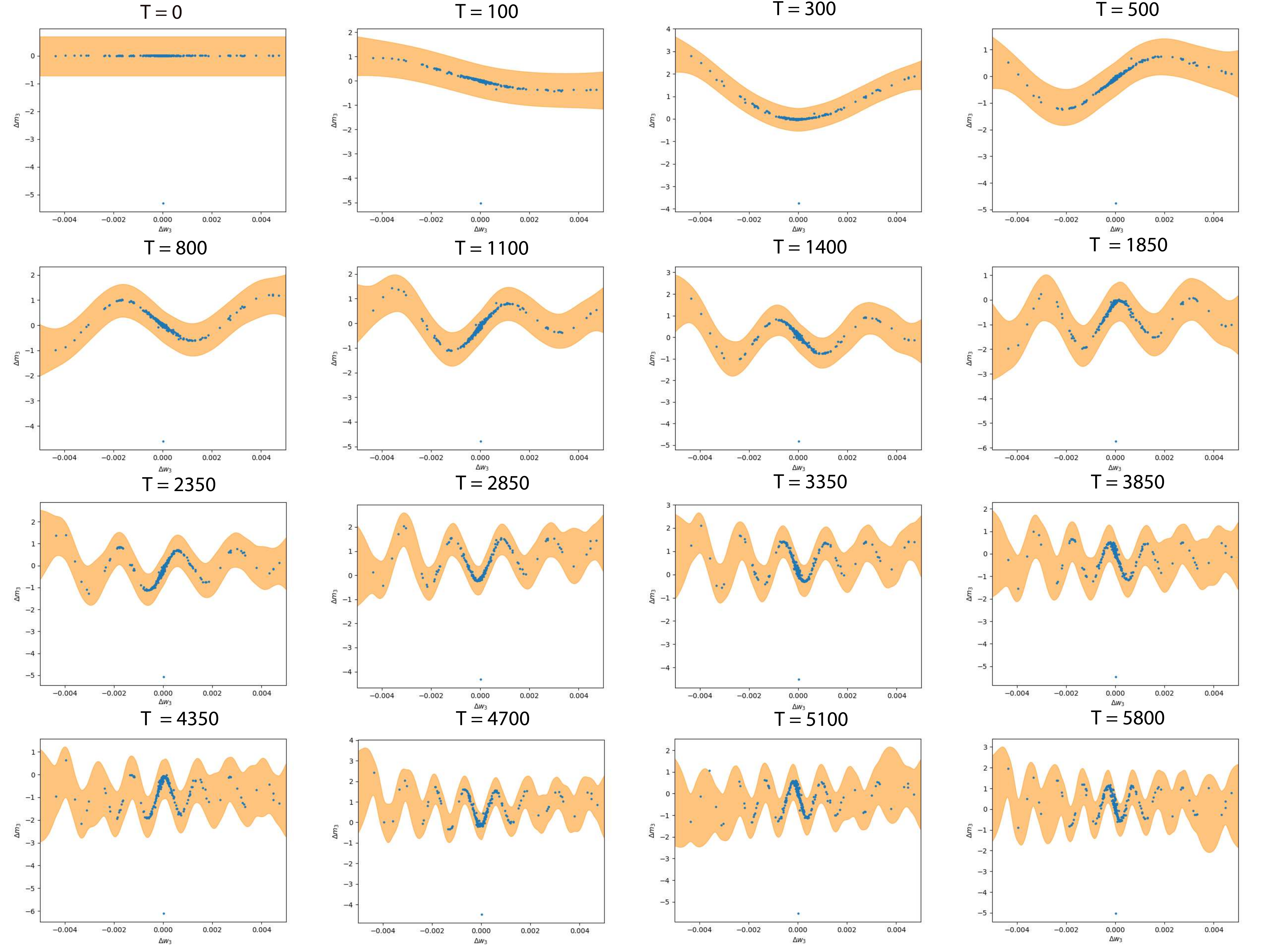}
    \caption{The time evolution of the learned GP models from directional derivatives for $\partial x_3 / \partial k_p$ by Gaussian sampling (Sine 1 joint).}
    \label{fig:sine_joint3_normal_dim3_voxel=0}
\end{figure}

\begin{figure}[t!]
    \centering
    \includegraphics[width=0.81\textwidth]{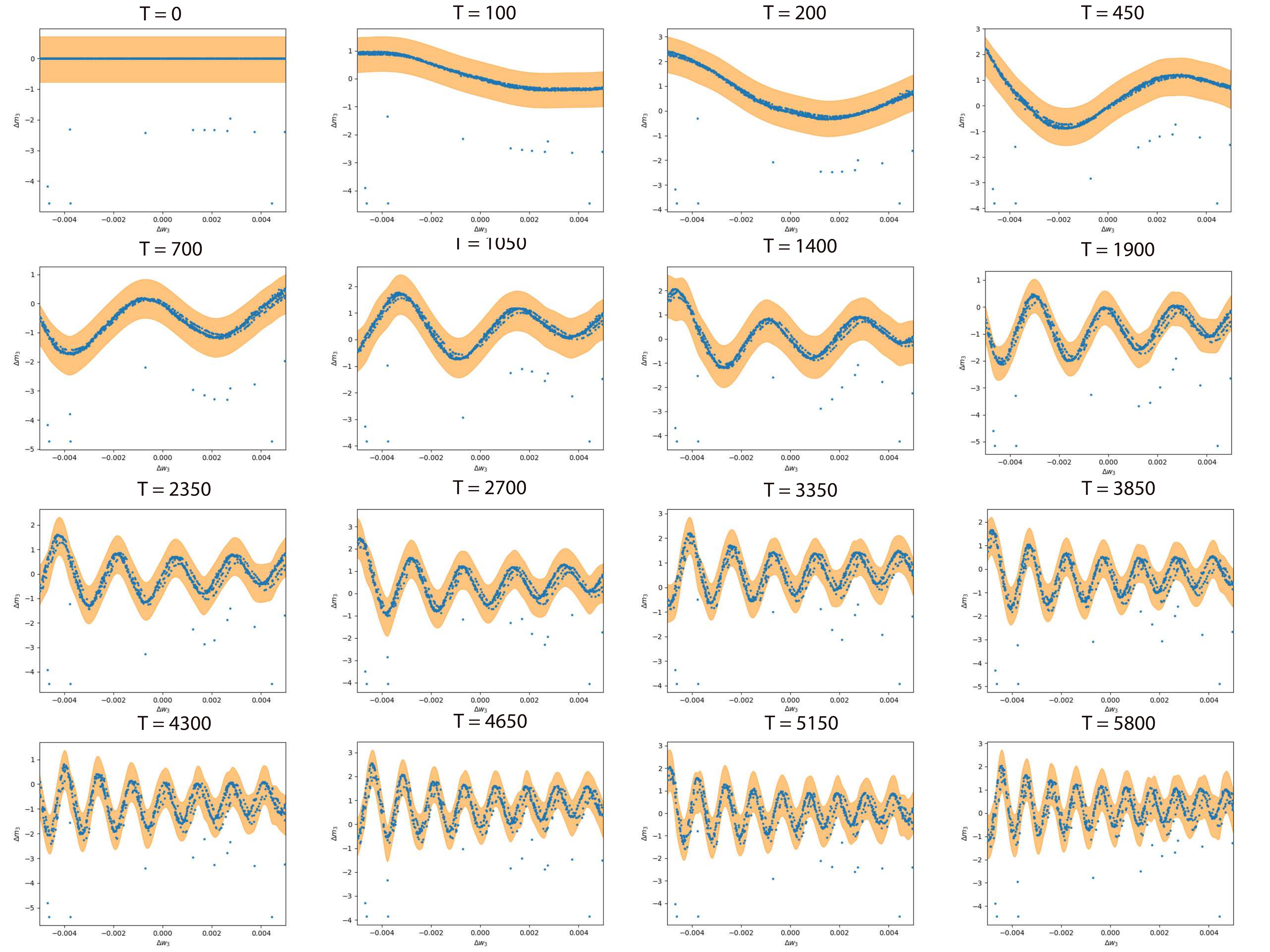}
    \caption{The time evolution of the learned GP models from directional derivatives for $\partial x_3 / \partial k_p$ by uniform sampling (Sine 1 joint).}
    \label{fig:sine_joint3_uniform_dim3_voxel=0}
\end{figure}

\begin{figure}[t!]
    \centering
    \includegraphics[width=0.81\textwidth]{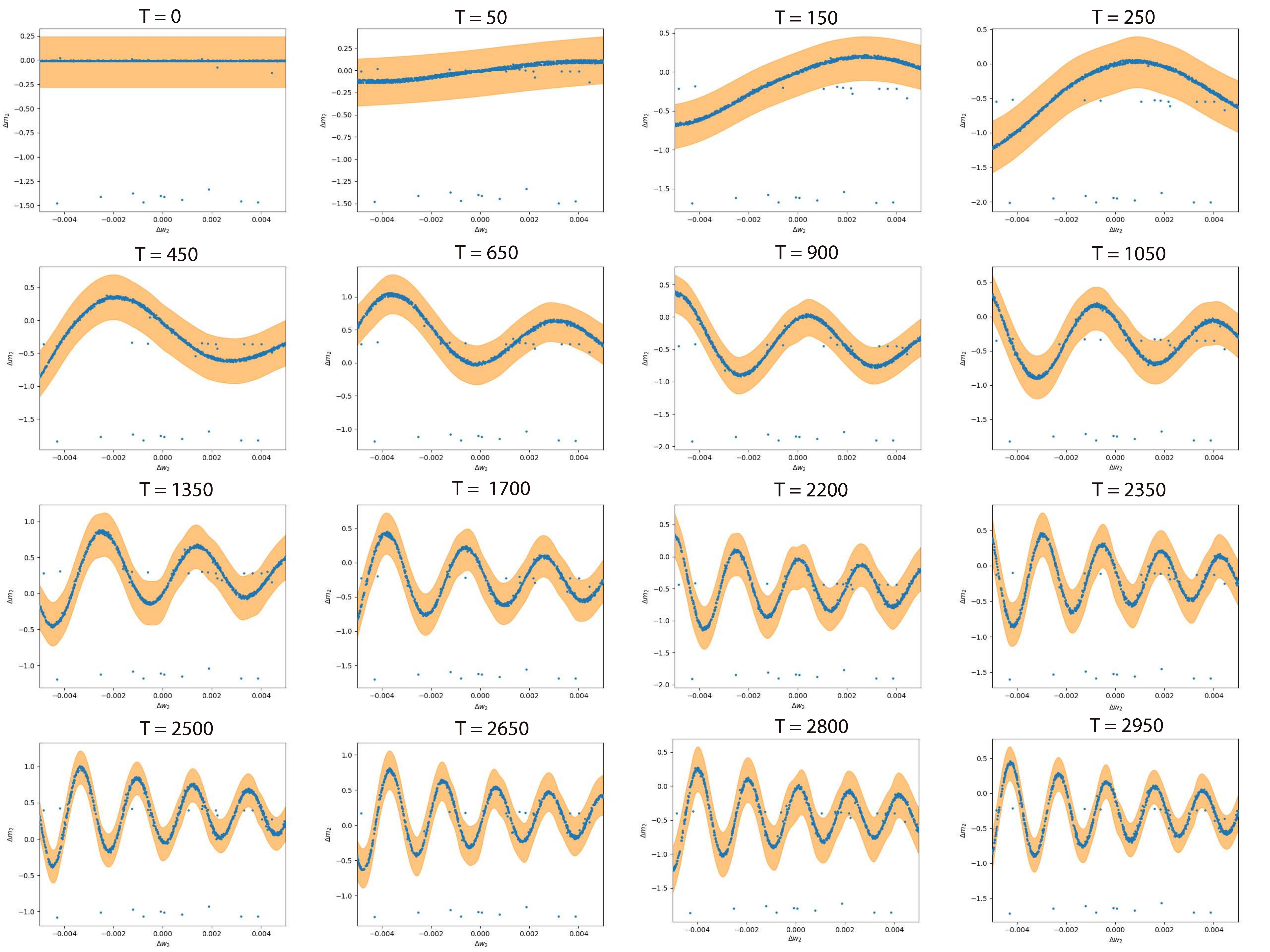}
    \caption{The time evolution of the learned GP models from directional derivatives for $\partial x_2 / \partial k_p$ by uniform sampling (Sine 2 joints).}
    \label{fig:sine_joint23_uniform_dim2_voxel=0}
\end{figure}

\begin{figure}[t!]
    \centering
    \includegraphics[width=0.81\textwidth]{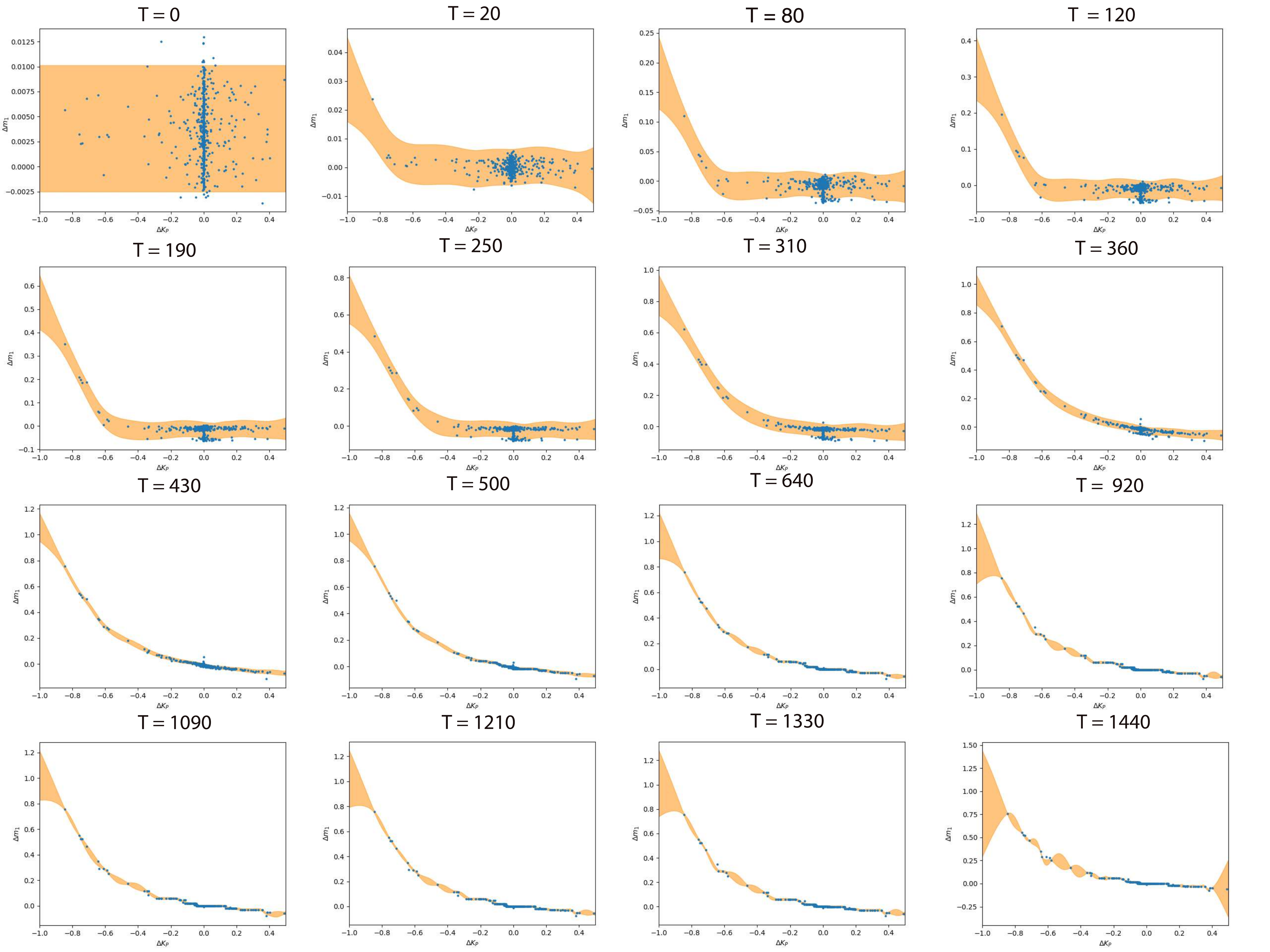}
    \caption{The time evolution of the learned GP models from directional derivatives for $\partial x_1 / \partial k_p$ by Gaussian sampling (PD controller).}
    \label{fig:pd_normal_dim1_voxel=0}
\end{figure}

\begin{figure}[t!]
    \centering
    \includegraphics[width=0.81\textwidth]{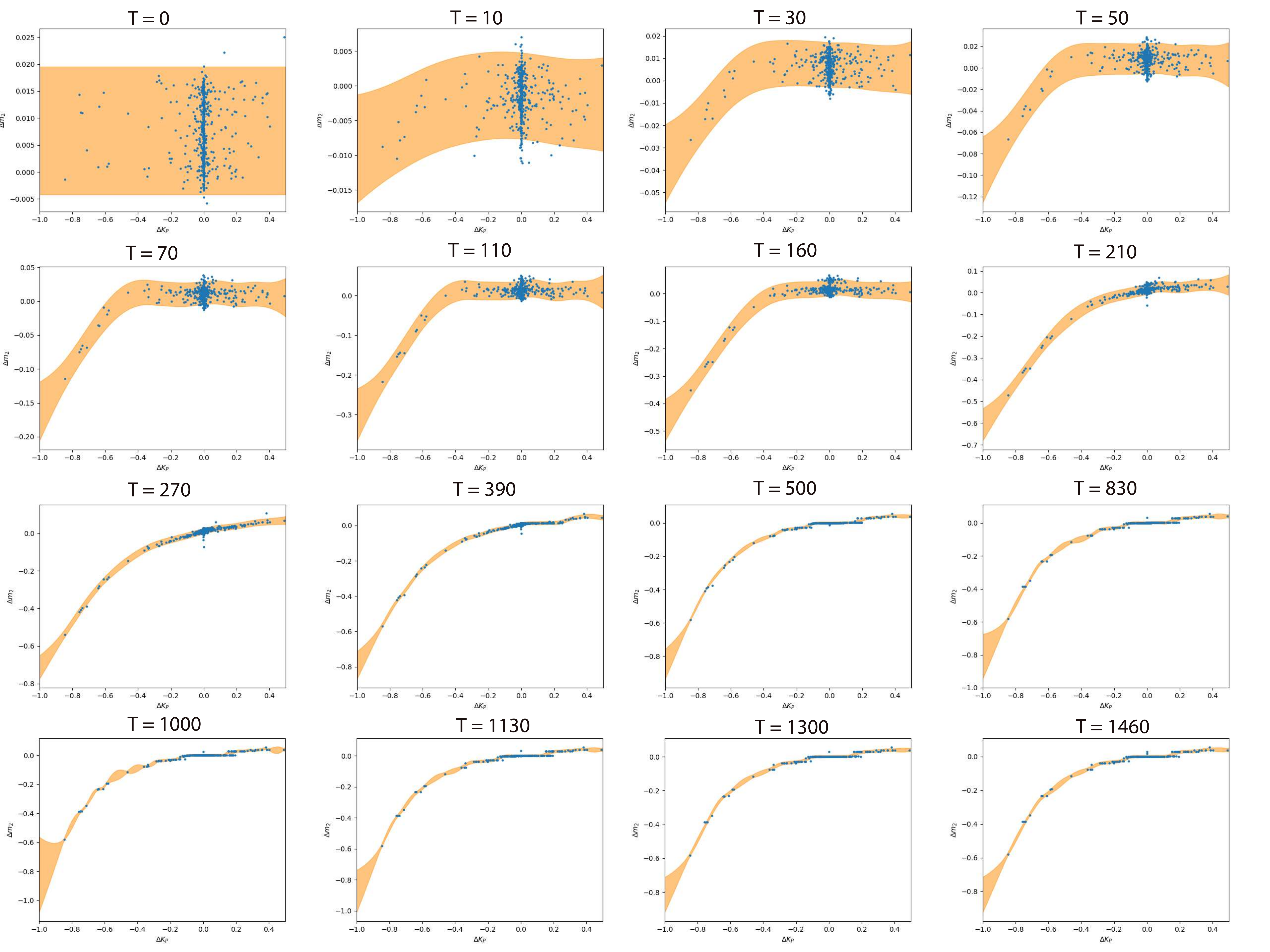}
    \caption{The time evoltuion of the learned GP models from directional derivatives for $\partial x_2 / \partial k_p$ by Gaussian sampling (PD controller).}
    \label{fig:pd_normal_dim2_voxel=0}
\end{figure}

\begin{figure}[t!]
    \centering
    \includegraphics[width=0.81\textwidth]{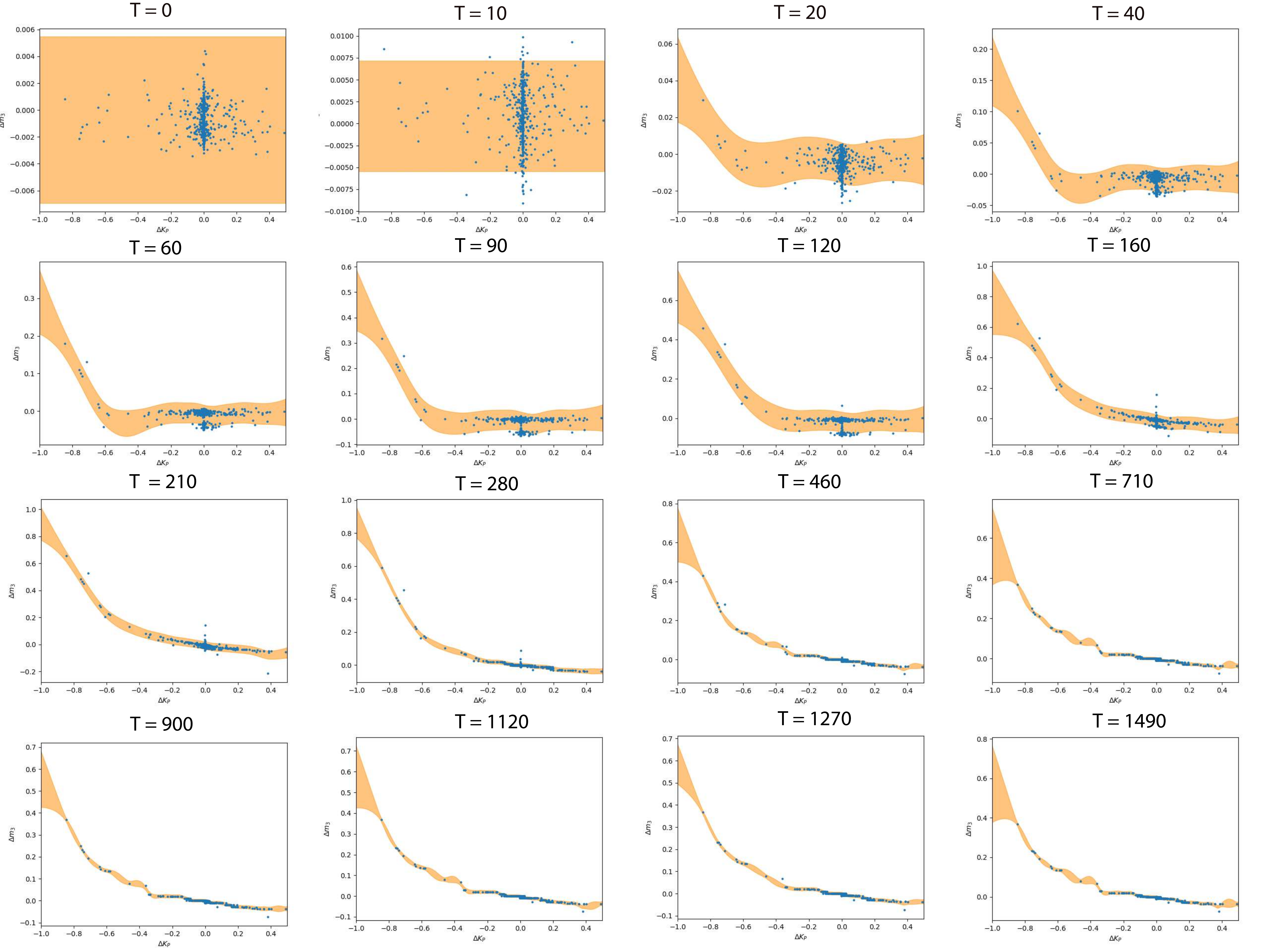} 
    \caption{The time evolution of the learned GP models from directional derivatives for $\partial x_3 / \partial k_p$ by Gaussian sampling (PD controller).}
    \label{fig:pd_normal_dim3_voxel=0}
\end{figure}

\begin{figure}[t!]
    \centering
    \includegraphics[width=0.81\textwidth]{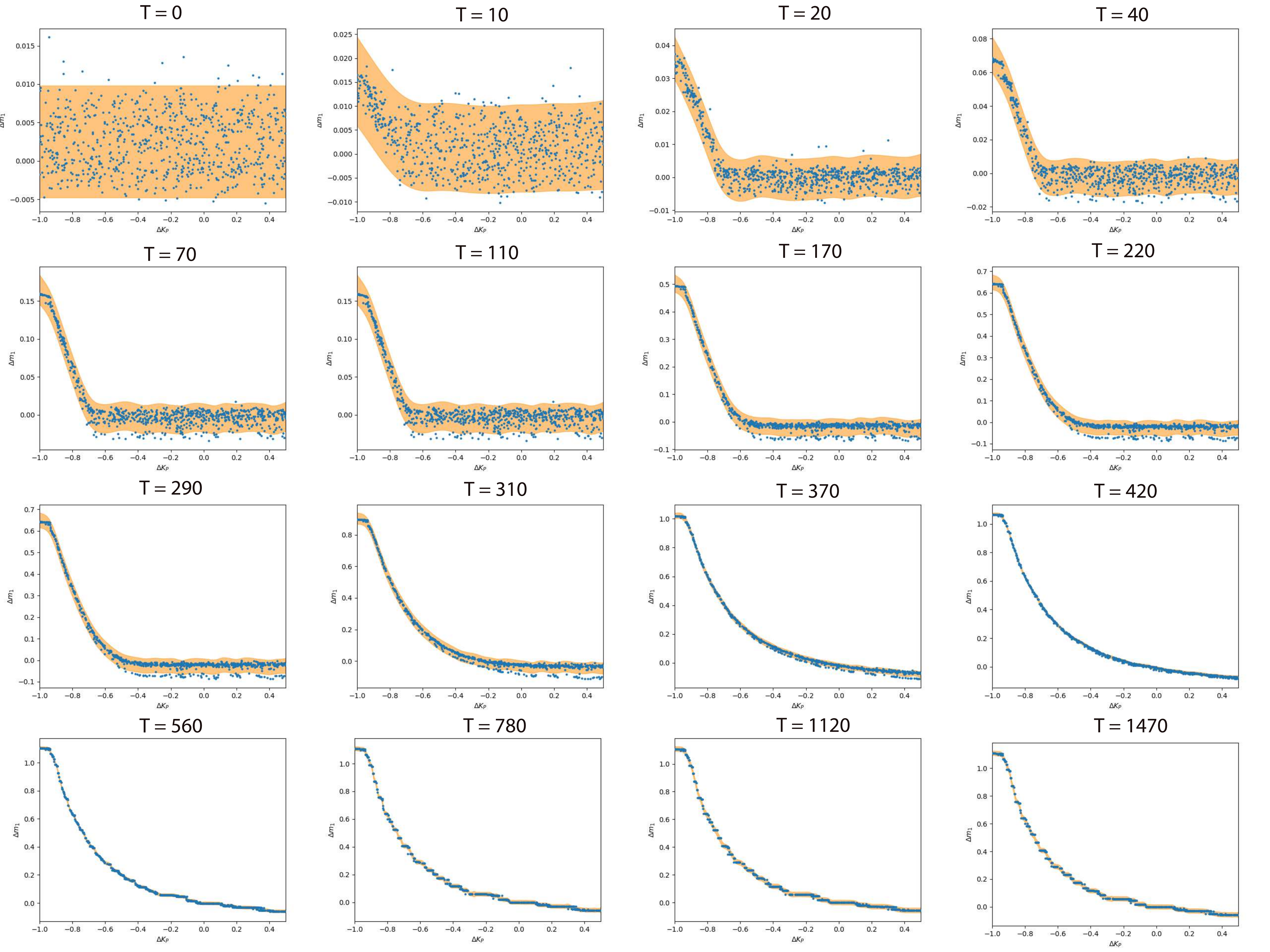}
    \caption{The time evolution of the learned GP models from directional derivatives for $\partial x_1 / \partial k_p$ by uniform sampling (PD controller).}
    \label{fig:pd_uniform_dim1_voxel=0}
\end{figure}

\begin{figure}[t!]
    \centering
    \includegraphics[width=0.81\textwidth]{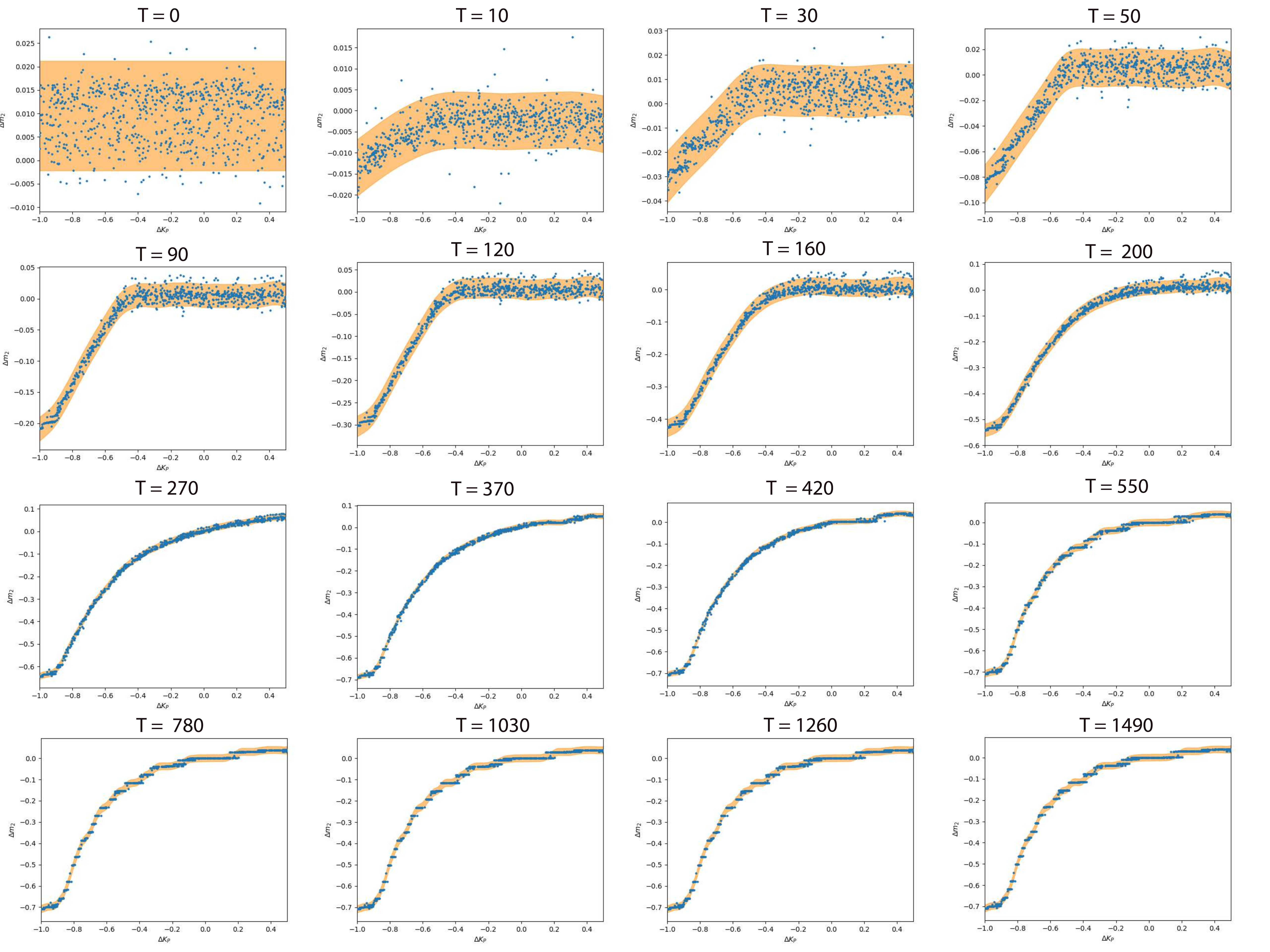}
    \caption{The time evolution of the learned GP models from directional derivatives for $\partial x_2 / \partial k_p$ by uniform sampling (PD controller).}
    \label{fig:pd_uniform_dim2_voxel=0}
\end{figure}

\begin{figure}[t!]
    \centering
    \includegraphics[width=0.81\textwidth]{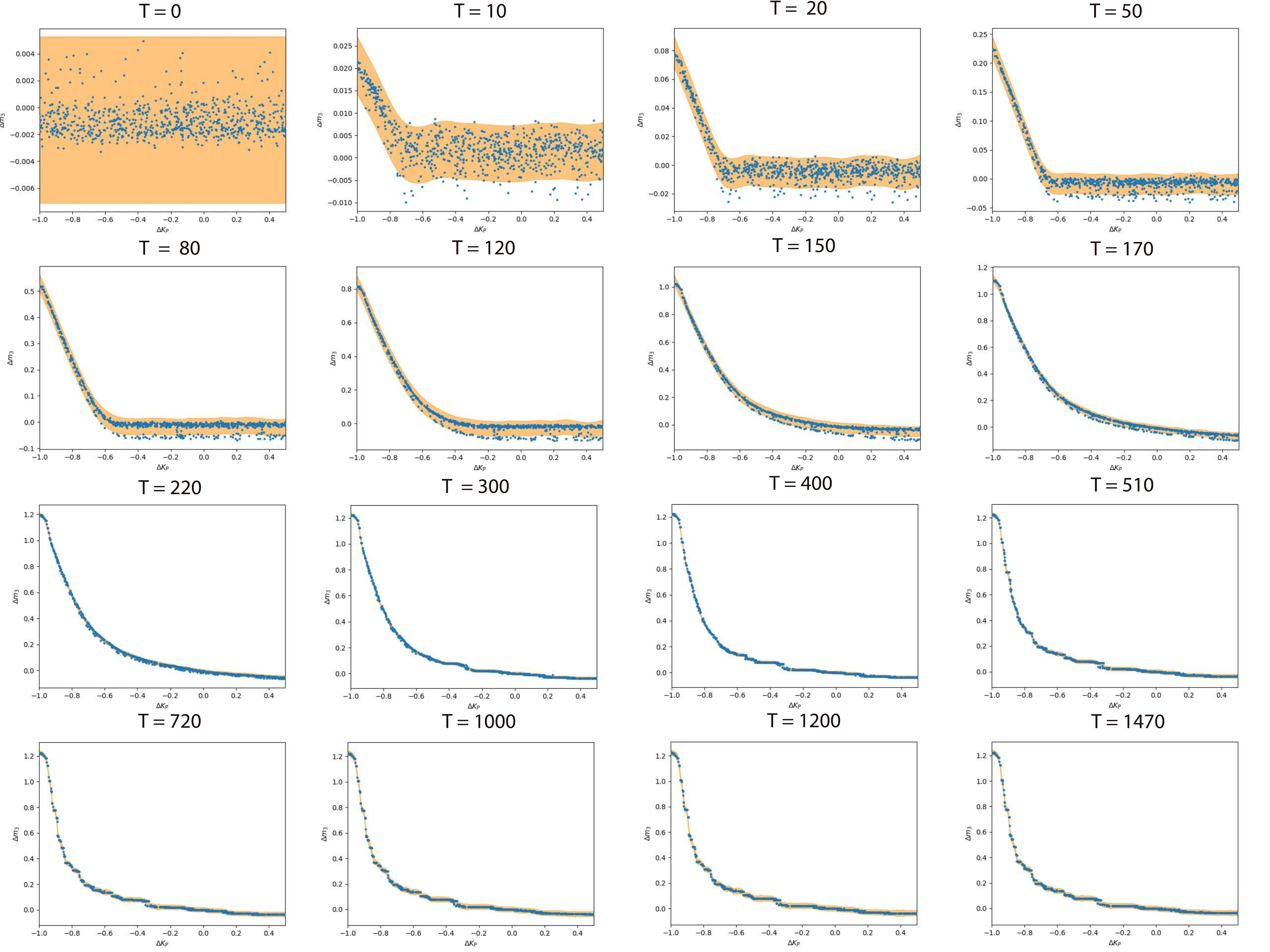}
    \caption{The time evolution of the learned GP models from directional derivatives for $\partial x_3 / \partial k_p$ by uniform sampling (PD controller).}
    \label{fig:pd_uniform_dim3_voxel=0}
\end{figure}

\begin{figure}[t!]
    \centering
    \includegraphics[width=0.77\textwidth]{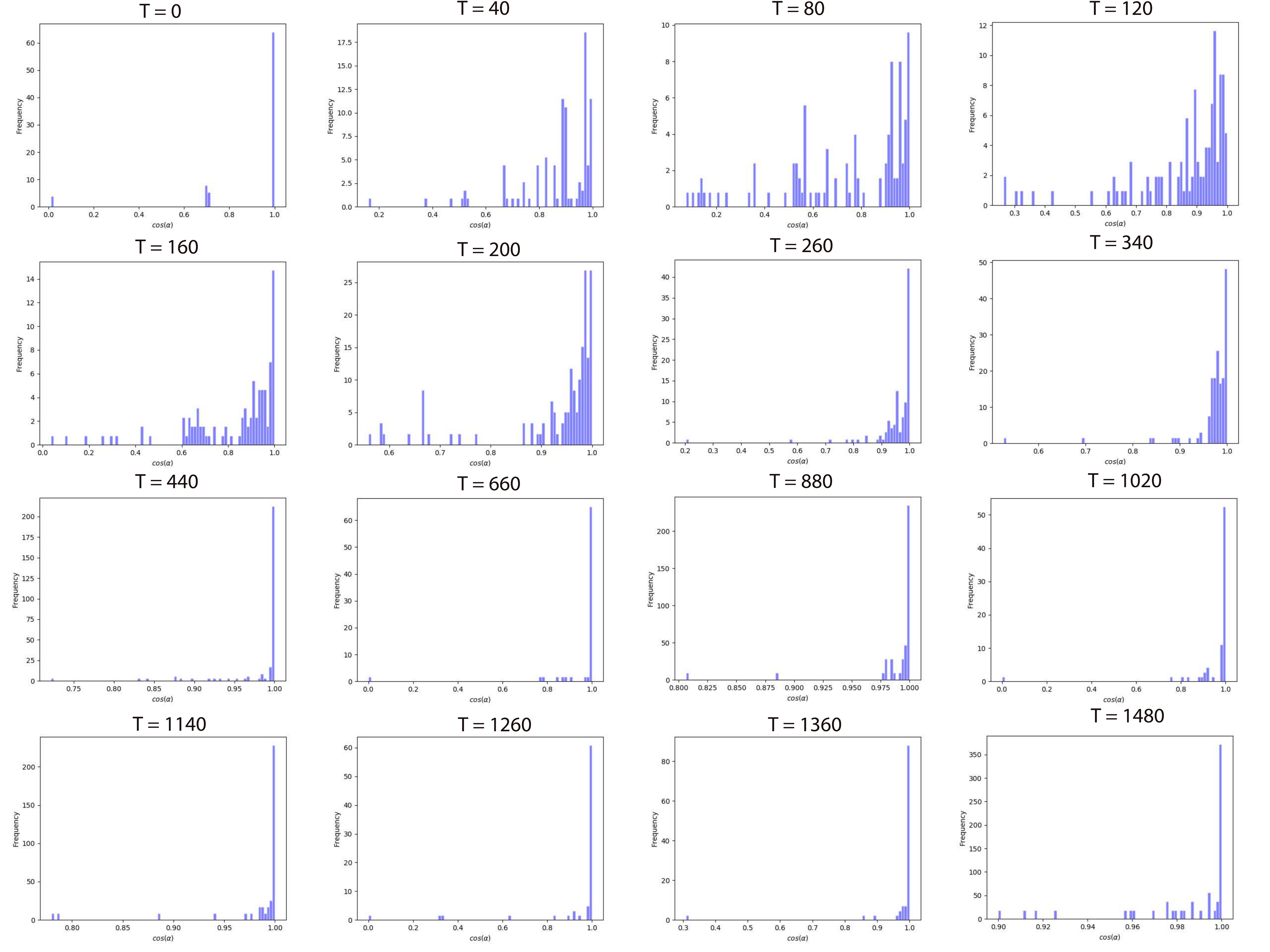}
    \caption{The time evolution of the histogram of $\cos(\alpha)$ where $\alpha$ is the angle between the true and predicted directional derivative. The perturbations in the training phase are generated by Gaussian sampling. As shown, angles of the derivatives are predicted mostly correctly (PD controller).}
    \label{fig:hist_pd_normal_voxel=0.01}
\end{figure}

\begin{figure}[t!]
    \centering
    \includegraphics[width=0.77\textwidth]{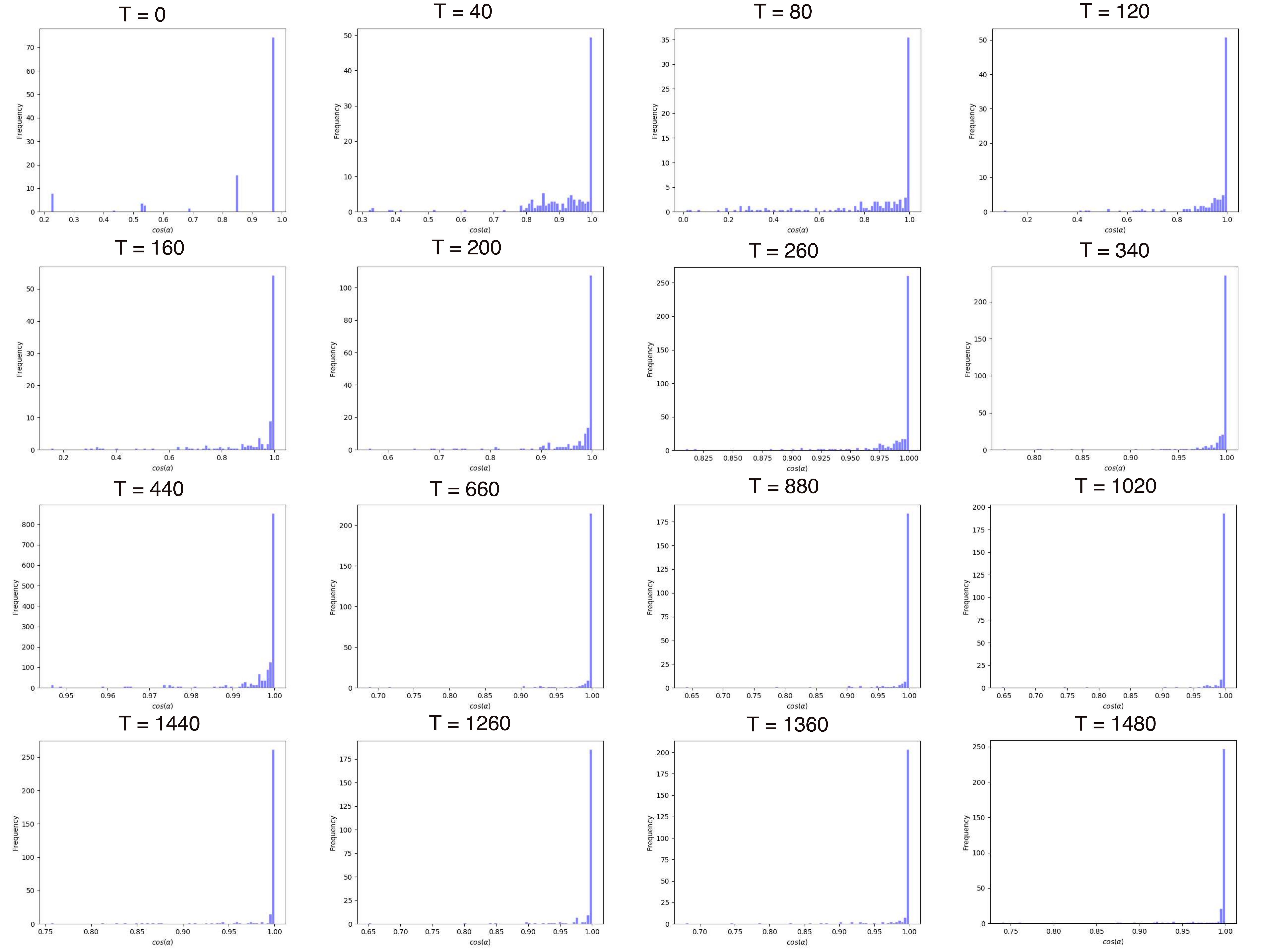}
    \caption{The time evolution of the histogram of $\cos(\alpha)$ where $\alpha$ is the angle between the true and predicted directional derivative. The perturbations in the training phase are generated by uniform sampling. As shown, angles of the derivatives are predicted mostly correctly (PD controller).}
    \label{fig:hist_pd_uniform_voxel=0.01}
\end{figure}

\begin{figure}[t!]
    \centering
    \includegraphics[width=0.77\textwidth]{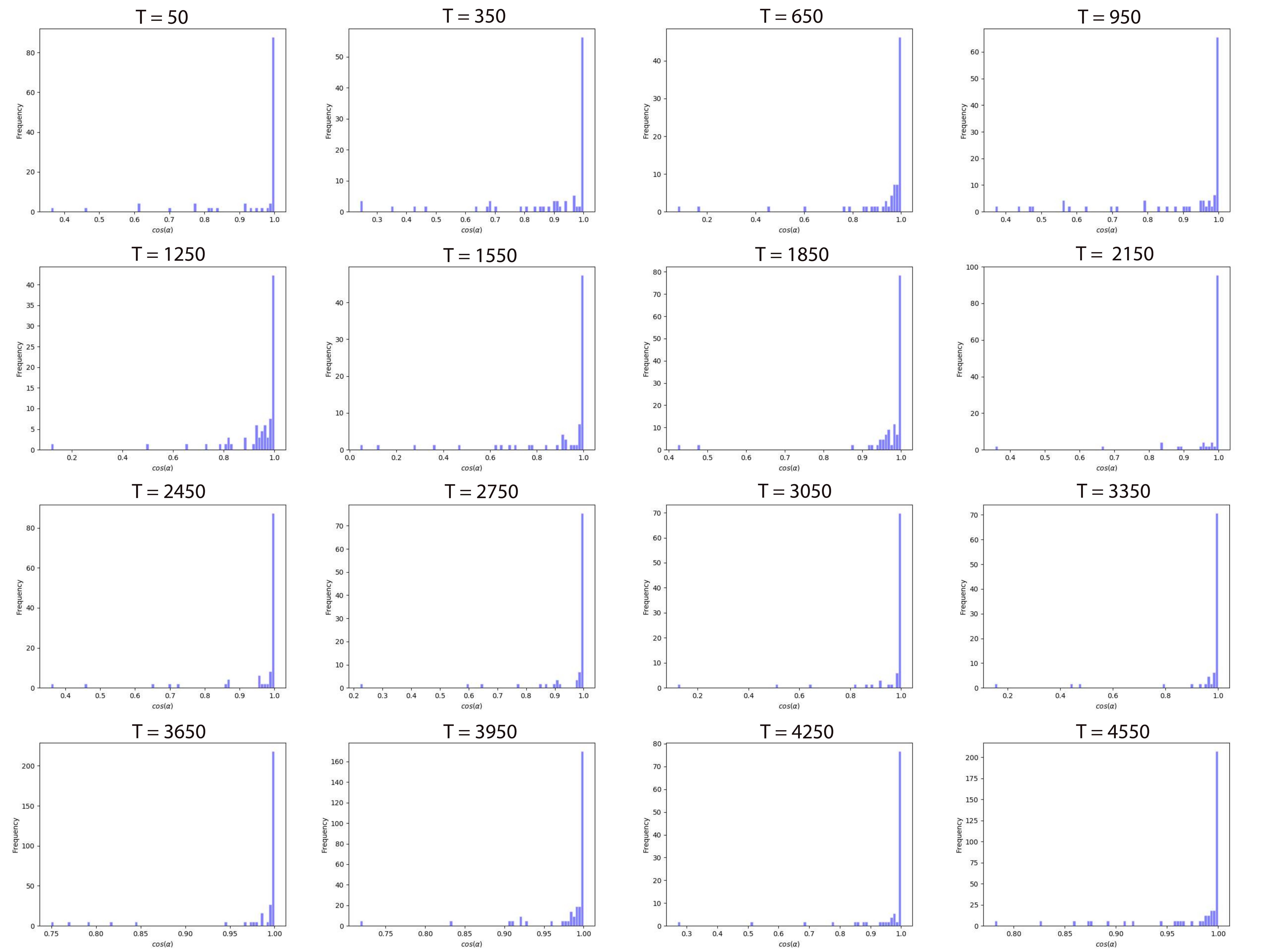}
    \caption{The time evolution of the histogram of $\cos(\alpha)$ where $\alpha$ is the angle between the true and predicted directional derivative. The perturbations in the training phase are generated by Gaussian sampling. As shown, angles of the derivatives are predicted mostly correctly (Sine 2 joints).}
    \label{fig:hist_sine_joint23_normal_voxel=0.01}
\end{figure}

\begin{figure}[t!]
    \centering
    \includegraphics[width=0.77\textwidth]{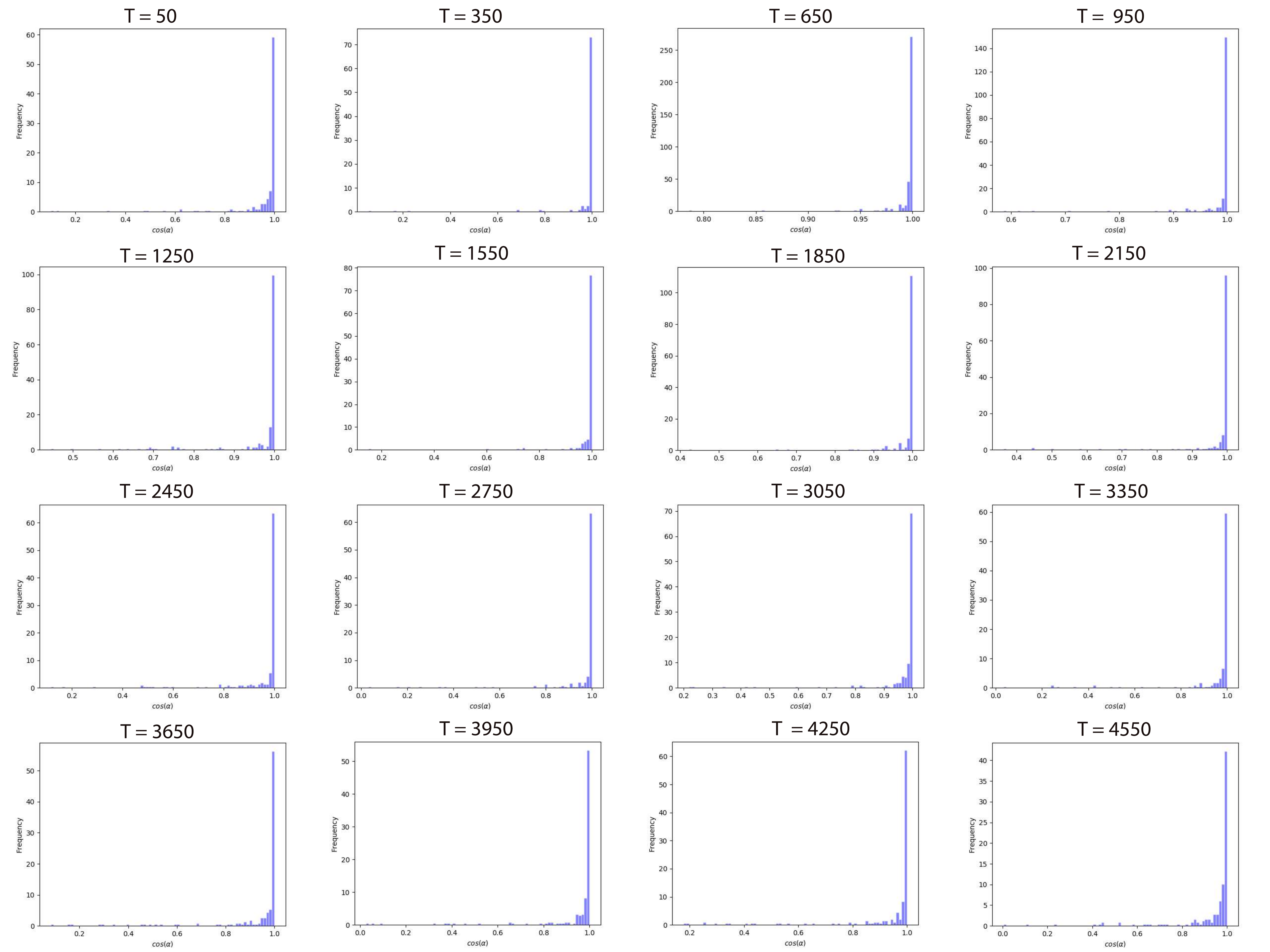}
    \caption{The time evolution of the histogram of $\cos(\alpha)$ where $\alpha$ is the angle between the true and predicted directional derivative. The perturbations in the training phase are generated by uniform sampling. As shown, angles of the derivatives are predicted mostly correctly (Sine 2 joints).}
    \label{fig:hist_sine_joint23_uniform_voxel=0.01}
\end{figure}

\begin{figure}[h!]
    \centering
    \begin{subfigure}[t]{1.35in}
        \centering
        \includegraphics[width=1.445in]{figs/experiments/quiver/w1w1_0,9.png}
        \caption{}        
    \end{subfigure}
    \begin{subfigure}[t]{1.35in}
        \centering
        \includegraphics[width=1.445in]{figs/experiments/quiver/w1w1_1,2.png}
        \caption{}    
   \end{subfigure}
   \begin{subfigure}[t]{1.35in}
        \centering
        \includegraphics[width=1.445in]{figs/experiments/quiver/w1w1_5,6.png}
        \caption{}    
   \end{subfigure}
   \begin{subfigure}[t]{1.35in}
        \centering
        \includegraphics[width=1.445in]{figs/experiments/quiver/w10w10_5,6.png}
        \caption{}    
    \end{subfigure}

    \begin{subfigure}[t]{1.35in}
        \centering
        \includegraphics[width=1.445in]{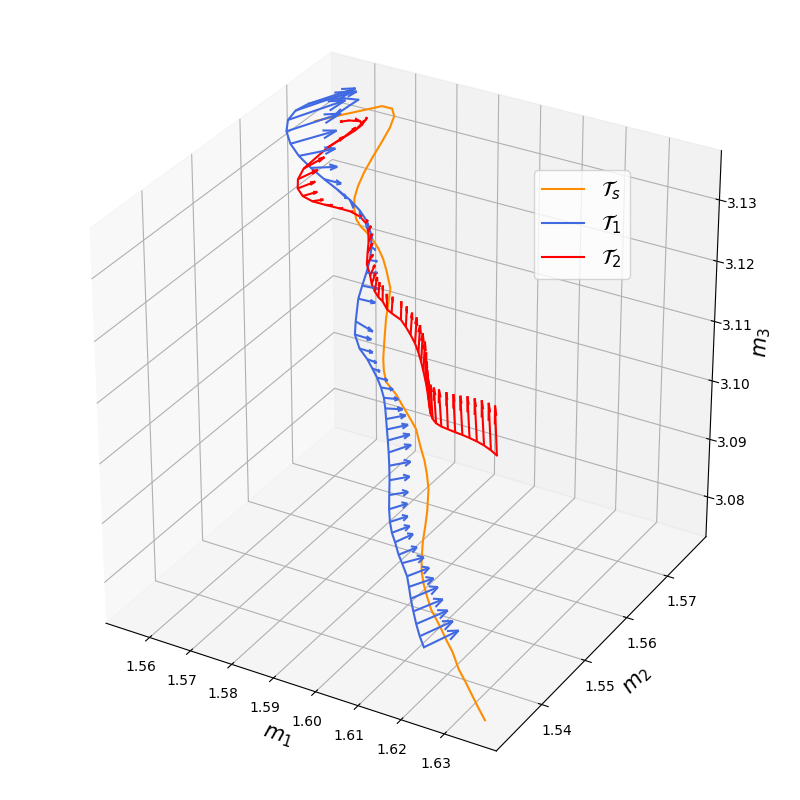}
        \caption{}        
    \end{subfigure}
    \begin{subfigure}[t]{1.35in}
        \centering
        \includegraphics[width=1.445in]{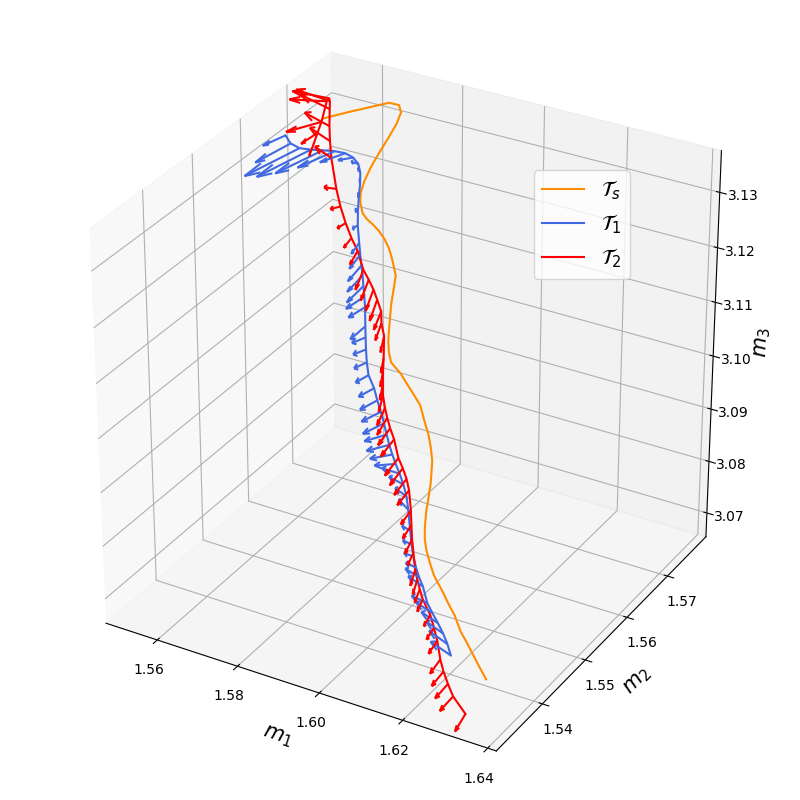}
        \caption{}    
   \end{subfigure}
   \begin{subfigure}[t]{1.35in}
        \centering
        \includegraphics[width=1.445in]{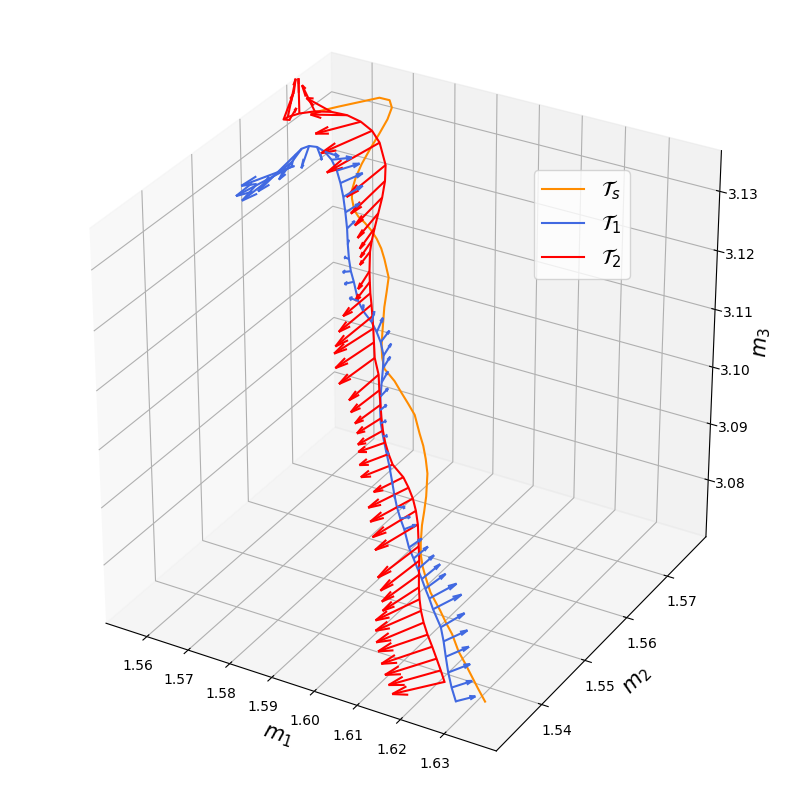}
        \caption{}    
   \end{subfigure}
   \begin{subfigure}[t]{1.35in}
        \centering
        \includegraphics[width=1.445in]{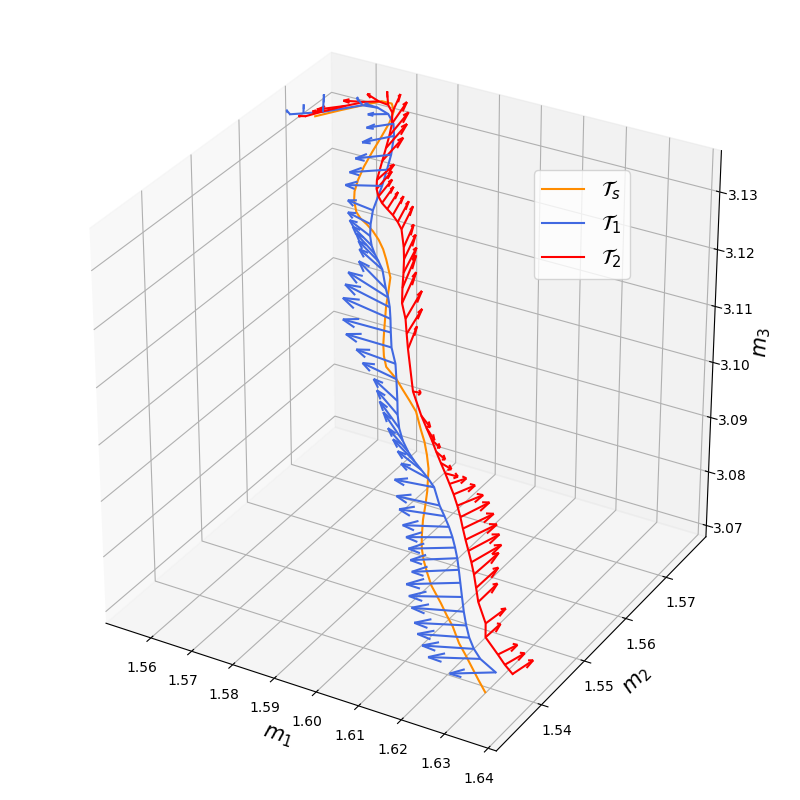}
        \caption{}    
    \end{subfigure}

    \begin{subfigure}[t]{1.35in}
        \centering
        \includegraphics[width=1.445in]{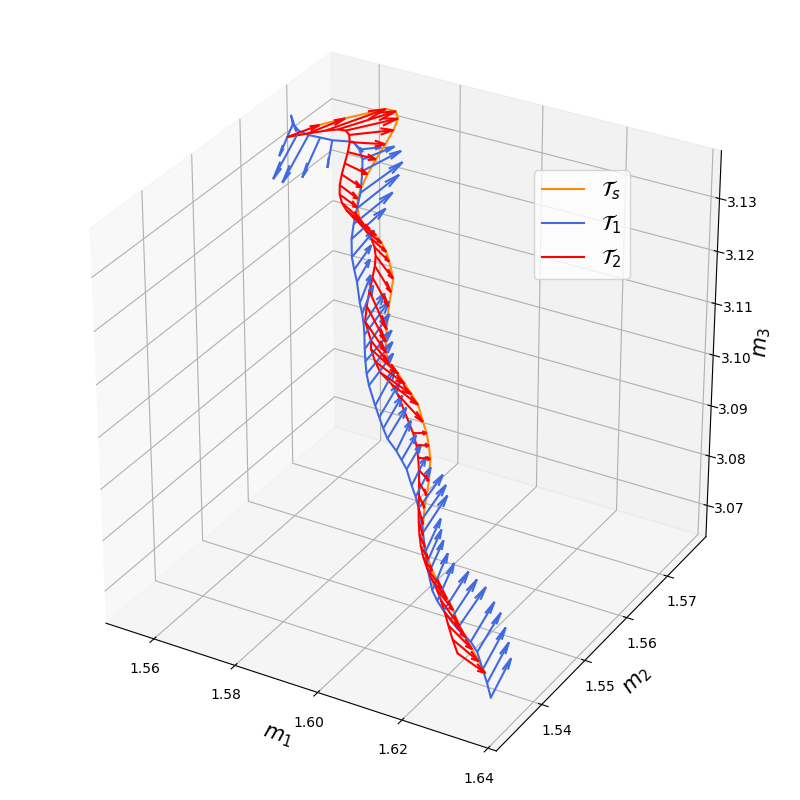}
        \caption{}        
    \end{subfigure}
    \begin{subfigure}[t]{1.35in}
        \centering
        \includegraphics[width=1.445in]{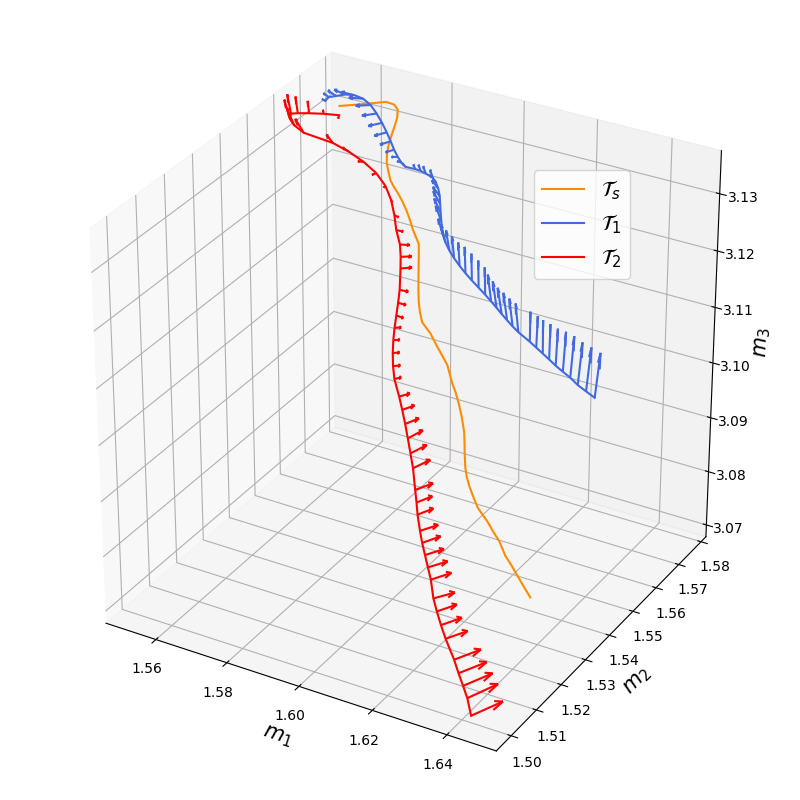}
        \caption{}    
   \end{subfigure}
   \begin{subfigure}[t]{1.35in}
        \centering
        \includegraphics[width=1.445in]{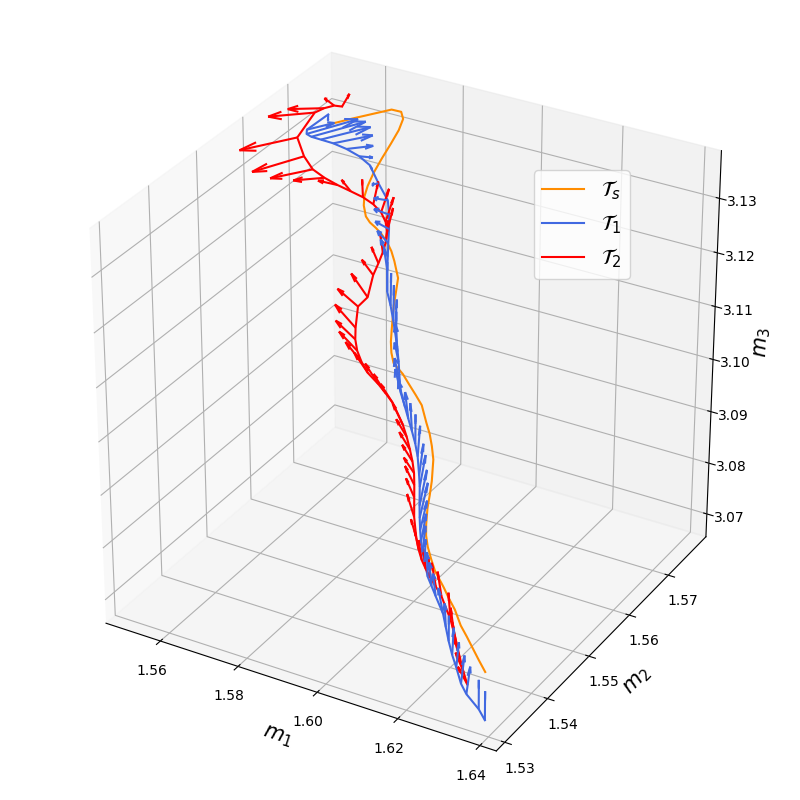}
        \caption{}    
   \end{subfigure}
   \begin{subfigure}[t]{1.35in}
        \centering
        \includegraphics[width=1.445in]{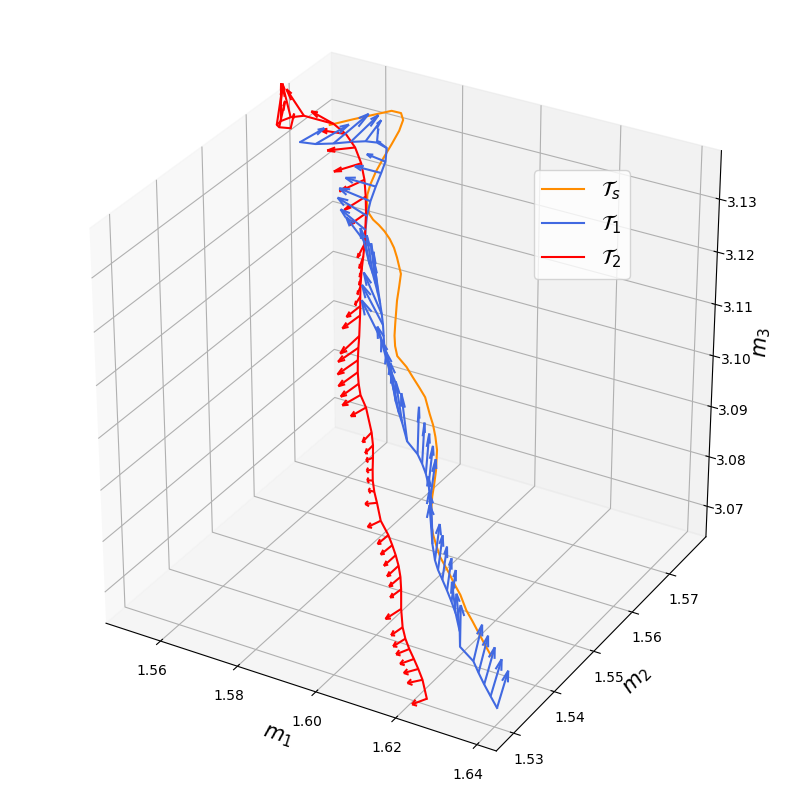}
        \caption{}    
    \end{subfigure}
    
    \caption{Examples of trajectories produced by the perturbed controller and the computed derivatives along the trajectory. The arrows are plotted as they originate from the perturbed trajectories only for easier distinction. Each arrow corresponds to the change of the states at a certain time step on the source trajectory as a result of perturbing the policy. Each figure corresponds to a pair of nominal values of $\{w, b\}$ for the linear open-loop controller of~\Cref{sec:linear_openloop_controller} and the perturbed trajectories are produced by Gaussian sampling (The orange trajectory is the source trajectory $\Tcal_s$ and others are the perturbed ones $\Tcal_1, \Tcal_2$. The quivers on the perturbed trajectories represent calculated physical derivatives).}\label{fig:quivers_more}
\end{figure}

\end{document}